\documentclass{sanket}
\usepackage{psfrag,epsf}
\usepackage{enumerate}
\usepackage{natbib}
\usepackage{url} 
\usepackage{ulem}

\usepackage{subcaption}
\usepackage{amsmath,bm,bbm,amsthm}
\usepackage{amssymb}
\usepackage{float}
\usepackage{graphicx}
\usepackage{xcolor}
\usepackage{algorithm,algorithmic}
\usepackage{setspace}
\usepackage{wrapfig}
\usepackage[export]{adjustbox}

\usepackage{tikz}
\usetikzlibrary{positioning}

\allowdisplaybreaks

\usepackage[utf8]{inputenc} 
\usepackage[T1]{fontenc}    
\usepackage{hyperref}       
\usepackage{url}            
\usepackage{booktabs}       
\usepackage{amsfonts}       
\usepackage{nicefrac}       
\usepackage{microtype}      

\newtheorem{theorem}{Theorem}[section]
\newtheorem{lemma}[theorem]{Lemma}
\newtheorem{corollary}[theorem]{Corollary}

\newtheorem{definition}[theorem]{Definition}

\newcommand{\E}{\mathbb{E}}
\newcommand{\R}{\mathbb{R}}

\newcommand{\boldB}{\boldsymbol{B}}
\newcommand{\boldD}{\boldsymbol{D}}

\newcommand{\boldI}{\boldsymbol{I}}

\newcommand{\boldW}{\boldsymbol{W}}
\newcommand{\boldX}{\boldsymbol{X}}

\newcommand{\bolda}{\boldsymbol{a}}

\newcommand{\bolds}{\boldsymbol{s}}

\newcommand{\boldk}{\boldsymbol{k}}

\newcommand{\boldv}{\boldsymbol{v}}
\newcommand{\boldw}{\boldsymbol{w}}
\newcommand{\boldx}{\boldsymbol{x}}
\newcommand{\boldy}{\boldsymbol{y}}
\newcommand{\boldz}{\boldsymbol{z}}

\newcommand{\calF}{\mathcal{F}}

\newcommand{\calM}{\mathcal{M}}
\newcommand{\calN}{\mathcal{N}}

\newcommand{\calQ}{\mathcal{Q}}

\newcommand{\indicator}{\mathbbm{1}}

\newcommand{\bdelta}{\boldsymbol{\delta}}

\newcommand{\bzeta}{\boldsymbol{\zeta}}
\newcommand{\bsigma}{\boldsymbol{\sigma}}

\newcommand{\bTheta}{\boldsymbol{\Theta}}
\newcommand{\btheta}{\boldsymbol{\theta}}
\newcommand{\bmu}{\boldsymbol{\mu}}

\newcommand{\prob}{\mathbb{P}}

\newcommand{\SJ}[1]{\textcolor{black}{#1}}

\newcommand{\blind}{1}

\begin{document}

\def\spacingset#1{\renewcommand{\baselinestretch}%
{#1}\small\normalsize} \spacingset{1}


\if1\blind
{
  \title{Layer Adaptive Node Selection in Bayesian Neural Networks: Statistical Guarantees and Implementation Details}
  \author{Sanket Jantre\thanks{corresponding author}}
  \author{Shrijita Bhattacharya}
  \author{Tapabrata Maiti}
  \affil{Department of Statistics and Probability, Michigan State University}
  \date{}
  \maketitle
} \fi

\if0\blind
{
  \title{Layer Adaptive Node Selection in Bayesian Neural Networks: Statistical Guarantees and Implementation Details}
  \author{}
  \date{}
  \maketitle
  \vspace{-1cm}
} \fi

\begin{abstract} \spacingset{1.1} 
\noindent Sparse deep neural networks have proven to be efficient for predictive model building in large-scale studies. Although several works have studied theoretical and numerical properties of sparse neural architectures, they have primarily focused on the edge selection. Sparsity through edge selection might be intuitively appealing; however, it does not necessarily reduce the structural complexity of a network. Instead pruning excessive nodes leads to a structurally sparse network \SJ{with significant computational speedup during inference.} To this end, we propose a Bayesian sparse solution using spike-and-slab Gaussian priors to allow for automatic node selection during training. The use of spike-and-slab prior alleviates the need of an ad-hoc thresholding rule for pruning. In addition, we adopt a variational Bayes approach to circumvent the computational challenges of traditional Markov Chain Monte Carlo (MCMC) implementation. In the context of node selection, we establish the fundamental result of variational posterior consistency together with the characterization of prior parameters. In contrast to the previous works, our theoretical development relaxes the assumptions of the equal number of nodes and uniform bounds on all network weights, thereby accommodating sparse networks with layer-dependent node structures or coefficient bounds. With a layer-wise characterization of prior inclusion probabilities, we discuss the optimal contraction rates of the variational posterior. \SJ{We empirically demonstrate that our proposed approach outperforms the edge selection method in computational complexity with similar or better predictive performance. Our experimental evidence further substantiates that our theoretical work facilitates layer-wise optimal node recovery.}
\end{abstract}

\noindent%

\spacingset{1.1} 



\section{\SJ{Introduction}}

\label{intro} 
Deep learning profoundly impacts science and society due to its impressive empirical success driven primarily by copious amounts of datasets, ever increasing computational resources, and deep neural network's (DNN) ability to learn task-specific representations. The key characteristic of deep learning is that accuracy empirically scales with the size of the model and the amount of training data. As such, large neural network models such as OpenAI GPT-3 (175 Billion) now typify the state-of-the-art across multiple domains such as natural language processing, computer vision, speech recognition etc. Nevertheless deep neural networks do have some drawbacks despite their wide ranging applications. First, this form of model scaling is exorbitantly prohibitive in terms of computational requirements, financial commitment, energy requirements etc. Second, DNNs tend to overfit leading to poor generalization in practice \citep{Zhang2017generalize}. Finally, there are numerous scenarios where training and deploying such huge models is practically infeasible. Examples of such scenarios include federated learning, autonomous vehicles, robotics, recommendation systems where models have to be refreshed daily/hourly or in an online manner for optimal performance.

A promising direction for addressing these issues while improving the efficiency of DNNs is exploiting sparsity. From a practical perspective, it has been well-known that neural networks can be sparsified without significant loss in performance, \cite{Mozer-Smolensky-1988}, and there is growing evidence that it is more so in the case of modern DNNs. Sparsity can arise naturally or be induced in multiple forms in DNNs, including input data, weights, and nodes. Weight pruning approaches perform high model compression leading to significant storage cost reduction at test-time \citep{Han-et-al-2016,Molchanov-et-al-2017,Zhu2018_to_prune,frankle2018LOT}. However, they result in unstructured sparsity in deep neural architectures which leads to inefficient computational gains in practical setups \citep{Wen-et-al-2016}. Instead, inducing group sparsity on collection of incoming weights into a given node (or node selection) reduces the dimensions of weight matrices per layer allowing for significant computational savings. To that effect, edge selection and node selection approaches are complementary with the former leading to storage reduction and the later leading to computational speedup during inference stage. Although one may argue node selection arises as a byproduct of edge selection, we clearly demonstrate that an approach which targets node selection directly leads to lower latency models (smaller number of nodes per layer) compared to an approach which achieves node selection through edge selection.

Node selection in deep neural networks has been explored under frequentist setting in \cite{Alvarez2016}, \cite{Wen-et-al-2016}, and \cite{Scardaane2017} using group sparsity regularizers. On the other hand, \cite{Louizos-et-al-2017}, \cite{Neklyudov-et-al-2017}, and \cite{Ghosh-JMLR-2018} incorporate group sparsity via shrinkage priors in Bayesian paradigm. These group sparsity approaches specifically applied for node selection have shown significant computational speedup and lower memory footprint at inference stage. However, all of the proposed methods of neuron selection perform ad-hoc pruning requiring fine-tuned thresholding rules. Moreover, the posterior inference of network weights in Bayesian neural networks (BNN) through standard MCMC method, ex. Hamiltonian Monte Carlo \citep{Neal1992-HMC}, does not scale well to modern neural network architectures and large datasets used in practice. Instead computationally efficient variational inference as an alternative to MCMC \citep{Jordan_Graph-2000,Blei2017}, has been explored in the context of edge selection both theoretically and numerically by \cite{blundell2015weight}, \cite{Cherief-Abdellatif-2020}, \cite{Bai-Guang-2020}. On the other hand, \cite{Louizos-et-al-2017} and \cite{Ghosh-JMLR-2018} have explored variational inference for node selection problem. In this work, we propose a Gaussian spike-and-slab prior for automatic node selection in Bayesian neural networks thereby alleviating the need of an ad-hoc thresholding rule for pruning. Further for scalability, we develop a variational Bayes algorithm for posterior inference of BNN model parameters in our proposed model and demonstrate its numerical performance through simulation and real regression and classification datasets. Finally, we provide the theoretical guarantees to our node selection method under mild restrictions on the network topology.
\vspace{2mm}

\begin{figure}[t]
\centering
\includegraphics[width=0.8\textwidth]{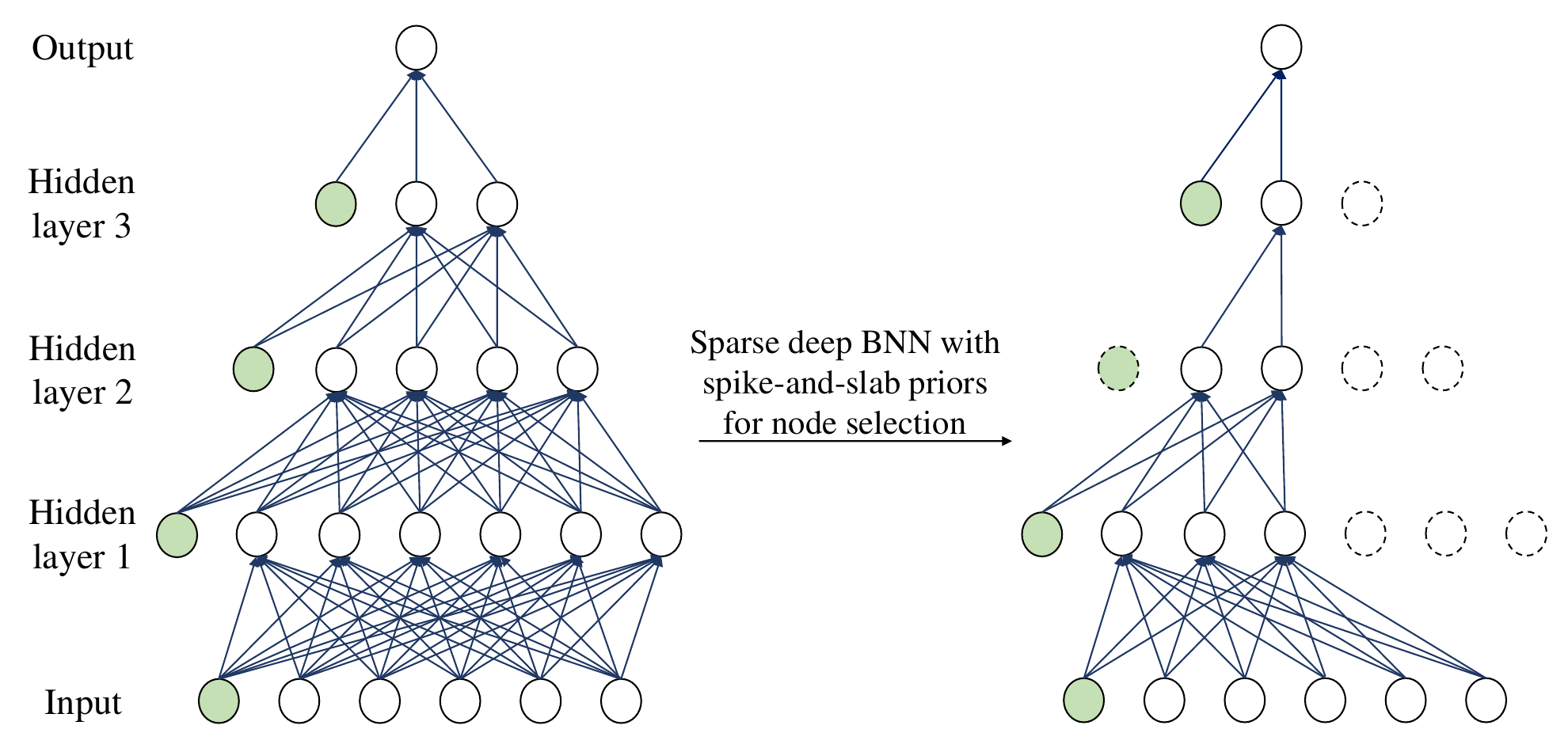}
\caption{Sparse deep BNN using spike-and-slab priors achieves node selection in the given dense network on left leading to a sparse network on right.}
\label{fig:Aim_1_sparse_BDNN_Node_Edge_Sel_fig}
\end{figure}

\noindent {\bf Related Work.} A closely related work to our paper is \cite{Bai-Guang-2020}'s automated edge selection model using spike-and-slab prior. There the slab distribution controls the magnitude of weights and spike allows for the exact setting of weights to 0. We introduce spike-and-slab framework for node selection in BNNs and show the key resource efficiency trade-off between node and edge selection at test-time. There are two main advantages to node selection over edge selection (1) fewer parameters to train during optimization, (2) results in structurally compact network leading to computational speedup at test-time.

On the theoretical front, sparse BNNs have been studied in the works of \cite{Polson-Rockova-2018} and \cite{Sun-Liang-2021}. In the context of variational inference, sparse BNNs have been studied in the recent works of \cite{Cherief-Abdellatif-2020} and \cite{Bai-Guang-2020}. All these works concentrate on the problem of edge selection facilitated through the use of Gaussian spike-and-slab priors. In the context of node selection, \cite{Ghosh-JMLR-2018} makes use of regularized horseshoe prior. The main limitations of their approach include (1) need for fine tuning of the thresholding rule for node selection, and (2) lack of a theoretical justification.

The only two works which have provided theoretical guarantees of their proposed sparse DNN methods under variational inference include those of \cite{Cherief-Abdellatif-2020} and \cite{Bai-Guang-2020}. Since they focus on the problem of edge selection, their theoretical developments are related to the results of \cite{Schmidt-Hieber-2017} (see the sieve construction in relation (4) in \cite{Schmidt-Hieber-2017}) and not directly extendable to our setup. Additionally, they assume certain restrictions on the network topology like (i) equal number of nodes in each layer, (ii) a known uniform bound $B$ on all network weights, and (iii) a global sparsity parameter which may not lead to a structurally compact network. Although from a numerical standpoint, one may implicitly extend the problem of edge selection to node selection, the theoretical guarantees of node selection consistency in sparse DNNs is not immediate.
\vspace{2mm}

\noindent {\bf Detailed Contributions.}
\begin{enumerate}
    \item We propose a Gaussian spike-and-slab node selection model and develop a variational Bayes approach for posterior inference of the model parameters. We call our approach {\bf SS-IG} ({\bf S}pike-and-{\bf S}lab {\bf I}ndependent {\bf G}aussian) model.
    \item We derive the variational consistency using a functional space of neural networks which takes two layer dependent bounds, one which upper bounds the number of neurons in each layer and the other which upper bounds the $L_1$ norm of the weights incident onto each node of a layer. These layer dependent bounds allow the generalization of the theoretical results presented to guarantee the consistency of any generic shaped network structure. Further, it also  guides the calculation of layer-wise prior inclusion probabilities which allow for optimal node recovery per layer in the computational experiments.
    \item We measure the computational gains achieved by our approach using layer-wise node sparsities for shallow models and floating point operations in larger models. Our numerical results validate the proposed theoretical framework for the node selection in DNN models. These empirical experiments further justify the use of layer-wise node inclusion probabilities to facilitate the optimal node recovery.
\end{enumerate}

\section{Nonparametric regression: deep learning approach}

Consider the nonparametric regression model with $p$ dimensional covariate $\bm{X}$. 
\begin{equation}
\label{e:class-tmp}
    Y_i = \eta_0(\bm{X}_i) + e_i, \enskip i = 1,\dots,n,
\end{equation}
where $e_i \stackrel{\text{i.i.d.}}{\sim} N(0,\sigma_e^2)$ (here i.i.d. denotes independent and identically distributed) and $\eta_0(\cdot):\mathbb{R}^p \to \mathbb{R}$. 

Thus, the conditional distribution of $Y|\boldX=\boldx$ under the true model is
\begin{equation}
\label{e:class}
f_{0}(y|\boldx)=(\sqrt{2\pi \sigma_e^2})^{-1}\exp\left(-(y-\eta_{0}(\boldx))^2/(2\sigma_e^2)\right)
\end{equation}
where $\boldx$ is a feature vector from a marginal distribution $P_{\boldX}$ and  $y$ is the corresponding output from the conditional distribution $Y|\boldX=\boldx$.

\noindent Let $g:\mathbb{R}^p \to \mathbb{R}$ be a measurable function, then for some loss function $\mathcal{L}$, the risk of  $g$ is 
$$R(g)=\int_{\mathcal{Y}\times \mathcal{X}}\mathcal{L}(Y,g(\boldX))dP_{\boldX,Y}$$
where  $P_{\boldX,Y}$, the joint distribution of $(\boldX,Y)$ is product of $P_{\boldX}$ and the conditional distribution $Y|X=\boldx$. (see \cite{Cannings-Samworth-2017} for more details). For the squared error loss, the above risk is minimized by $g^*(\boldx)=\eta_0(\boldx)$ \citep{friedman2001elements}. In practice, this estimator is not useful since $\eta_0(\boldx)$ is unknown. Thus, an estimator of $\eta_0(\boldx)$ is obtained based on the training observations, $\mathcal{D}=\{(\boldx_1,y_1),\cdots,(\boldx_n,y_n)\}$, drawn from $P_{\boldX,Y}$. To find the class of optimal estimators, we use DNNs as an approximation to $\eta_0(\boldx)$. 

For a $p\times 1$ input vector $\boldx$, consider a DNN with $L$ hidden layers with $k_{1}, \cdots, k_{L}$ as the number of nodes in the hidden layers denoted by $\eta_{\btheta}(\boldx)$. Also, 
\begin{equation}
\label{e:eta-x}
\eta_{\btheta}(\boldx)=\boldv_L+\boldW_L \psi(\boldv_{L-1}+\boldW_{L-1} \psi (\cdots \psi(\boldv_1+\boldW_1\psi(\boldv_0+\boldW_0 \boldx)))
\end{equation}
where $\boldv_{l}$ and $\boldW_{l}$,  $l=0, \cdots, L$ are $k_{l+1} \times 1$ vectors and  $k_{l+1} \times k_{l}$ matrices, respectively and $\psi$ is the activation function. Let $\btheta = \{\overline{\boldW}_0,\dots,\overline{\boldW}_L\}$ denote all the parameters in the DNN model under consideration. 
Using the DNN in \eqref{e:eta-x} to approximate the true function $\eta_0(\boldx)$,  the conditional distribution of $Y|\boldX=\boldx$ is 
\begin{equation*}
f_{\btheta}(y|\boldx)=(\sqrt{2\pi \sigma_e^2})^{-1}\exp\left(-(y-\eta_{\btheta}(\boldx))^2/(2\sigma_e^2)\right)
\end{equation*}
Thus, the likelihood function for the data $\mathcal{D}$ under the model and the truth is 
\begin{equation}
\label{e:lik}
P_{\btheta}^n=\prod_{i=1}^n f_{\btheta}(y_i|\boldx_i), \hspace{10mm}
P_0^n=\prod_{i=1}^n f_{0}(y_i|\boldx_i).
\end{equation}
For theoretical development in the subsequent sections  we shall assume $P_{\boldX}= U[0,1]^p$ and  $\sigma_e^2=1$  and $\psi$ is any $1-$Lipschitz continuous activation function.  

\section{Node selection with spike-and-slab prior}  

\label{sec:sparsity-priors}
To allow for automatic node selection, we consider a spike-and-slab prior consisting of a Dirac spike ($\delta_0$) at 0 and a slab distribution \citep{Mitchell-Beauchamp-1988}. The spike part is represented by an indicator variable which is set to 0 if a node is not present in the network. The slab  part comes from a Gaussian distributed random variable. To allow for the layer-wise node selection, we assume that the prior inclusion probability $\lambda_l$ varies as a function of the layer index $l$. The symbol i.d. is used to denote independently distributed random variables.

\noindent {\bfseries Prior:} We assume a spike-and-slab prior of the following form with $z_{lj}$ as the indicator for the presence of $j^{\rm th}$ node in the $l^{\rm th}$ layer
\begin{equation*}
\overline{\boldw}_{lj}|z_{lj} \stackrel{\text{i.d.}}{\sim}   \left[(1-z_{lj})\bdelta_0 + z_{lj} N(0,\sigma_0^2 \boldI) \right], \hspace{3mm} z_{lj} \stackrel{\text{i.d.}}{\sim}   {\rm Ber}(\lambda_l)
\end{equation*} 
where $l=0, \dots, L$, $j=1, \dots, k_{l+1}$. Also, $\overline{\boldw}_{lj} = (\overline{w}_{lj1},\dots,\overline{w}_{lj{k_l+1}})$ is a vector of edges incident on the $j^{\rm th}$ node in the $l^{\rm th}$ layer. In the above formula, note $\bdelta_0$ is a Dirac spike vector of dimension $k_l+1$ with all entries zero and $\boldI$ is the identity matrix of dimension $k_l+1 \times k_l+1$. Furthermore, $z_{lj}$ with $j=(1,\dots, k_{l+1})$ all follow Bernoulli($\lambda_l$) to allow for common prior inclusion probability, $\lambda_l$, for each node from a given layer $l$. We set $\lambda_L=1$ to ensure no node selection occurs in the output layer. 

\noindent {\bfseries Posterior:} With $\boldz_l=(z_{l1}, \cdots, z_{lk_{l+1}})$, let  $\boldz=(\boldz_{1}, \cdots, \boldz_{L})$  denote the vector of all indicator variables. The posterior 
distribution of $(\btheta,\boldz)$ given $\mathcal{D}$ is given by
\begin{equation}
\label{e:posterior}
\pi(\btheta,\boldz|\mathcal{D}) = \frac{ P_{\btheta}^n \pi(\btheta|\boldz)\pi(\boldz)}{\sum_{\boldz} \int  P_{\btheta}^n\pi(\btheta|\boldz)\pi(\boldz)d\btheta}=\frac{ P_{\btheta}^n \pi(\btheta|\boldz)\pi(\boldz)}{m(\mathcal{D})}
\end{equation}
where $P_{\btheta}^n=\prod_{i=1}^nf_{\btheta}(y_i|\boldx_i)$ is the likelihood function as in \eqref{e:lik}, $\pi(\boldz)$ is the probability mass function of $\boldz$ with respect to the counting measure and $\pi(\btheta|\boldz)$ is the conditional probability density function with respect to the Lebesgue measure of $\btheta$ given $\boldz$ . Further, $m(\mathcal{D})$ is the marginal density of the data and is free of  $(\btheta,\boldz)$. 

Let
$\widetilde{\pi}(\btheta)=\sum_{\boldz}  \pi(\btheta,\boldz)$ be the marginal prior of $\btheta$. We shall use the notation 
\begin{equation}
    \label{e:prior-theta}
    \widetilde{\Pi}(\mathcal{A})=\int_{\mathcal{A}}\widetilde{\pi}(\btheta)d\btheta
\end{equation}
  to denote the probability distribution function corresponding to the density function $\widetilde{\pi}$. The marginal posterior of $\btheta$ expressed as a function of the marginal prior for $\btheta$ is
    \begin{equation*}
        \widetilde{\pi}(\btheta|\mathcal{D})=\sum_{\boldz}  \pi(\btheta,\boldz|\mathcal{D})=\frac{P_{\btheta}^n \widetilde{\pi}(\btheta)}{\int P_{\btheta}^n \widetilde{\pi}(\btheta)d\btheta}=\frac{P_{\btheta}^n \widetilde{\pi}(\btheta)}{m(\mathcal{D})}
        \end{equation*}
Thus, the probability distribution function corresponding to the density function $\widetilde{\pi}(|\mathcal{D})$ is then given by
       \begin{equation}
    \label{e:true-posterior-theta}
    \widetilde{\Pi}(\mathcal{A}|\mathcal{D})=\int_{\mathcal{A}}\widetilde{\pi}(\btheta|\mathcal{D})d\btheta
\end{equation} 

\noindent {\bfseries Variational family:} We posit the following mean field variational family ($\mathcal{Q}^{\bf MF}$) on network weights as
\begin{equation*}
 \mathcal{Q}^{\bf MF}=\Big\{\overline{\boldw}_{lj}|z_{lj} 
\stackrel{\text{i.d.}}{\sim} \left[(1-z_{lj})\bdelta_0+ z_{lj} N(\bmu_{lj},\text{diag}(\bsigma^2_{lj}))\right], \hspace{3mm} z_{lj} \stackrel{\text{i.d.}}{\sim} {\rm Ber}(\gamma_{lj}) \Big\}
\end{equation*}
for $l=0, \dots, L$, $j=1, \dots, k_{l+1}$. This ensures that weight distributions follow spike-and-slab structure which allows for node sparsity through variational approximation. Further,  the weight distributions conditioned on the node indicator variables are all independent of each other (hence use of the term mean field family).  The variational distribution of parameters obtained post optimization will then inherently prune away redundant nodes from each layer. Also, Gaussian distribution for slab component is widely popular for approximating neural network weight distributions \citep{blundell2015weight, Louizos-et-al-2017, Bai-Guang-2020}.

Additionally, $\bmu_{lj}=(\mu_{lj1},\dots,\mu_{lj{k_l+1}})$ and $\bsigma^2_{lj}=(\sigma^2_{lj1},\dots,\sigma^2_{lj{k_l+1}})$ denote the vectors of variational mean and standard deviation parameters of the edges incident on the $j^{\rm th}$ node in the $l^{\rm th}$ layer. Similarly, $\gamma_{lj}$ denotes the variational inclusion probability of the $j^{\rm th}$ node in the $l^{\rm th}$ layer. We set $\gamma_{Lj}=1$ to ensure no node selection occurs in the output layer. 

\noindent {\bfseries Variational posterior:} Variational posterior aims to reduce the Kullback-Leibler (KL) distance between a variational family and the true posterior (\cite{Blei_2007,Hinton93}) as
\begin{equation}
\label{e:var-posterior}
    \pi^*=\underset{q \in \mathcal{Q}^{\bf MF}}{\text{argmin}}\:\: d_{\rm KL}(q,\pi(|\mathcal{D}))
\end{equation}
where $d_{\rm KL}(q,\pi(|\mathcal{D}))$ denotes the KL-distance between $q$ and $\pi(|\mathcal{D})$.

Note, the variational member $q$ can be written as $q(\btheta,\boldz)=q(\btheta|\boldz)q(\boldz)$ where $q(\boldz)$ is the probability mass function of $\boldz$ with respect to the counting measure  and $q(\btheta|\boldz)$ is the conditional density function  given with respect to the Lebesgue measure of  $\btheta$ given $\boldz$. Further,

\begin{align}
\label{e:elbo}
\nonumber \pi^*&= \underset{q \in \mathcal{Q}^{\bf MF}}{\text{argmin}} \sum_{\boldz}\int [\log q(\btheta,\boldz)-\log \hspace{0.5mm} \pi(\btheta,\boldz|\mathcal{D})]q(\btheta,\boldz)d\btheta\\
\nonumber &=\underset{q \in \mathcal{Q}^{\bf MF}}{\text{argmin}}\left(\sum_{\boldz}\int [\log q(\btheta,\boldz)-\log \pi(\btheta,\boldz,\mathcal{D})] q(\btheta,\boldz) d\btheta+\log m(\mathcal{D})\right)\\
 &=\underset{q \in \mathcal{Q}^{\bf MF}}{\text{argmin}}\:\: [-\text{ELBO}(q,\pi(|\mathcal{D}))]+\log m(\mathcal{D})=\underset{q \in \mathcal{Q}^{\bf MF}}{\text{argmax}}\:\: \text{ELBO}(q,\pi(|\mathcal{D}))
 \end{align}
Since $\log m(\mathcal{D})$ is free from $q$, it suffices to maximize the evidence lower bound (ELBO) above. 

Let $\widetilde{\pi}^*(\btheta)=\sum_{\boldz} \pi^*(\btheta|\boldz)\pi^*(\boldz)$ then $\widetilde{\pi}^*$ denotes the marginal variational posterior for $\btheta$. We shall use the notation \begin{equation}
\label{e:var-posterior-theta}
    \widetilde{\Pi}^*(\mathcal{A})=\int_{\mathcal{A}}\widetilde{\pi}^*(\btheta)d\btheta
\end{equation} to denote the probability distribution function corresponding to the density function $\widetilde{\pi}^*$.

\section{Posterior contraction rates}  

\label{sec:theory} 
In this section, we develop the theoretical consistency of the variational posterior in \eqref{e:var-posterior-theta} in context of node selection.
Previous works which establish the statistical consistency of sparse deep neural networks do so only in the context of edge selection. Thereby, the works of  \cite{Polson-Rockova-2018}, \cite{Cherief-Abdellatif-2020} and \cite{Bai-Guang-2020} use several results from the pioneer work of \cite{Schmidt-Hieber-2017}.
\SJ{In addition to node selection consistency, we also relax certain network restrictions considered in the previous works. These restrictions include (1) equal number of nodes in each layer which restricts one from using any previous information on the number of nodes in the deep neural architecture (2)  a known bound $B$ on all the neural network weights as they essentially rely on the sieve construction in equation 3 of \cite{Schmidt-Hieber-2017} which assumes that $L_{\infty}$ norm of all $\btheta$ entries is smaller than 1 (3) a global sparsity parameter $s$ which does not always consider structurally sparse networks.}


\SJ{Towards the proof, firstly our sieve construction allows the number of nodes of the neural network to vary as a function of the layer. Secondly, instead of global sparsity parameter $s$ (see the sieve construction in relation (4) of \cite{Schmidt-Hieber-2017}) we allow for layer wise sparsity vector $\bolds$ to account for the number of nodes in each layer. Finally, we relax the  assumption of a known bound $B$ by considering a sieve with a layer wise constraint (denoted by the vector $\boldB$) on the $L_1$ norm of the incoming edges of a node. Thus, our work extends on current literature along three directions (1) theoretically quantifies predictive performance of Bayesian neural networks with node based pruning (2) establishes that even without a fixed bound on network weights, one can recover true solution by appropriate choice of the prior (3) provides layer wise node inclusion probabilities to allow for structurally sparse solutions. The relaxation of these network structure assumptions requires us to provide the framework for node selection including appropriate sieve construction together with the derivation of the results in \cite{Schmidt-Hieber-2017} customized to our problem.}

To establish the posterior contraction rates, we show that the variational posterior in \eqref{e:var-posterior} concentrates in shrinking Hellinger neighborhoods of the true density function $P_0$ with overwhelming probability. Since $\boldX \sim U[0,1]^p$, thus $f_0(\boldx)=f_{\btheta}(\boldx)=1$. This further implies $P_0=f_0(y|\boldx)f_0(\boldx)=f_0(y|\boldx)$ and similarly $P_{\btheta}=f_{\btheta}(y|\boldx)$. We next define the Hellinger neighborhood of the true density $P_0$ as $$\mathcal{H}_{\varepsilon}=\{\btheta: d_{\rm H}(P_0,P_{\btheta})<\varepsilon \}$$ where the Hellinger distance between the true density function $P_0$ and the model density $P_{\btheta}$ is 
\begin{equation*}
d_{\rm H}^2(P_0,P_{\btheta})=\frac{1}{2}\int \left(\sqrt{f_{\btheta}(y|\boldx)}-\sqrt{f_0(y|\boldx)}\right)^2 dy d\boldx
\end{equation*}
We also define the KL neighborhood of the true density $P_0$ as $$\mathcal{N}_\varepsilon =\{\btheta: d_{\rm KL}(P_0,P_{\btheta})<\varepsilon \}$$ where the KL distance $d_{\rm KL}$ between the true density function $P_0$ and the model density $P_{\btheta}$ is
\begin{equation*}
d_{\rm KL}(P_0,P_{\btheta})=\int \log \frac{f_0(y|\boldx)}{f_{\btheta}(y|\boldx)}f_0(y|\boldx)dyd\boldx
\end{equation*}


Let $\boldk=(k_0, \cdots, k_{L+1})$ be the node vector, ${\overline{\boldW}}_l=(\boldw_{l1}^\top, \cdots,\boldw_{lk_{l+1}}^\top)^\top $
be the row representation of  ${\overline{\boldW}}_l$ and $\widetilde{\boldw}_l=(||\boldw_{l1}||_1, \cdots, 
||\boldw_{lk_{l+1}}||_1)$ be the vector of $L_1$ norms of the rows of ${\overline{\boldW}}_l$. Next we consider layer-wise sparsity, $\bolds=(s_1, \cdots, s_L)$ for node selection. Similarly, we consider layer-wise norm constraints, $\boldB=(B_1, \cdots, B_L)$ on $L_1$ norms of weights including bias incident onto any given node in each layer. Based on  $\bolds$ and $\boldB$, we define the following sieve of neural networks (check definition \ref{def:sieve}). 
\begin{equation}
    \label{e:sieve}
  \mathcal{F}(L,\boldk,\bolds,\boldB)=\left\{\eta_{\btheta} \in \eqref{e:eta-x}: ||\widetilde{\boldw}_l||_0 \leq s_l, ||\widetilde{\boldw}_l||_\infty \leq B_l  \right\}.
\end{equation}
The construction of a sieve is one of the most important tools towards the proof of consistency in infinite-dimensional spaces. In the works of  \cite{Schmidt-Hieber-2017}, \cite{Polson-Rockova-2018}, \cite{Cherief-Abdellatif-2020} and \cite{Bai-Guang-2020}, the sieve in the context of edge selection is given by
\begin{equation*}
  \mathcal{F}(L,\boldk,s)=\left\{\eta_{\btheta} \in \eqref{e:eta-x}: ||\btheta||_0 \leq s, ||\btheta||_{\infty} \leq 1  \right\}.
\end{equation*}
which works with an overall sparsity parameter $s$. In addition, note the $L_{\infty}$ norm of all the entries in $\btheta$ is assumed to be known constant equal to 1 (see relation (4) in \cite{Schmidt-Hieber-2017} and section 4 in \cite{Polson-Rockova-2018}). Section 3 in \cite{Bai-Guang-2020} does not explicitly mention the dependence of their sieve on some fixed bound $B$ on the edges in a network, however, their derivations on covering numbers (see proof of Lemma 1.2 in the supplement of \cite{Bai-Guang-2020}) borrow results from \cite{Schmidt-Hieber-2017}  which is based on sieve with $B=1$.

 Consider any sequence $\epsilon_n$. For Lemmas \ref{lem:test} and \ref{lem:prior}, we work with the sieve $\mathcal{F}(L,\boldk,\bolds,\boldB)$ in \eqref{e:sieve} with $\bolds=\bolds^\circ$ and $\boldB=\boldB^\circ$  where
  $s_l^\circ+1=n\epsilon_n^2/(\sum_{j=0}^L u_j)$ and $\log B_l^\circ=(n\epsilon_n^2)/((L+1)\sum_{j=0}^L (s_j^\circ+1))$ with $u_l=(L+1)^2(\log n+\log (L+1)+\log k_{l+1}+\log (k_l+1))$. Note, $s_l^\circ$ and $B_l^\circ$ do not depend on $l$.

  Lemma \ref{lem:test} below holds when the covering number (check definition~\ref{def:covering-no}) of the functions which belong to the sieve   $\mathcal{F}(L,\boldk,\bolds^\circ,\boldB^\circ)$ is well under control.  \noindent Lemma \ref{lem:prior} below states that for the same choice of the sieve, the prior gives sufficiently small probabilities on the complement space $\mathcal{F}(L,\boldk,\bolds^\circ,\boldB^\circ)^c$ (see the discussion under Theorem \ref{thm:var-post} for more details).

For the subsequent results, the symbol $\mathcal{A}^c$ will be used to denote complement of a set  $\mathcal{A}$.

\begin{lemma}[Existence of Test Functions]
\label{lem:test}
Let $\epsilon_n \to 0$ and $n \epsilon_n^2 \to \infty$.  There exists a testing function $\phi \in [0,1]$ and constants $C_1,C_2>0$,
\begin{align*}
\E_{P_0}(\phi) &\leq \exp \{-C_1n\epsilon_n^2 \} \\  \sup_{\substack{\btheta \in \mathcal{H}_{\epsilon_n}^c,\eta_{\btheta} \in \mathcal{F}(L,\bm{k},\bm{s}^\circ,\bm{B}^\circ)}}
\E_{P_{\btheta}}(1-\phi)  &\leq \exp \{-C_2nd^2_{\rm H}(P_0,P_{\btheta}) \}
\end{align*}
where $\mathcal{H}_{\epsilon_n}=\{\btheta:d_{\rm H}(P_0,P_{\btheta})\leq \epsilon_n\}$ is the Hellinger neighborhood of radius $\epsilon_n$.
\end{lemma}

\begin{lemma}[Prior mass condition.] \label{lem:prior} Let $\epsilon_n \to 0$, $n \epsilon_n^2 \to \infty$ and  $n\epsilon_n^2/\sum_{l=0}^L u_{l} \to \infty$, then for $\widetilde{\Pi}$ as in \eqref{e:prior-theta} and some constant $C_3>0$, $$ \widetilde{\Pi}(\mathcal{F}(L,\bm{k},\bm{s}^\circ,\bm{B}^\circ)^c)\leq \exp (-C_3n\epsilon_n^2/\sum_{l=0}^L u_l)$$

\end{lemma}

Whereas Lemmas \ref{lem:test} and \ref{lem:prior} work with a specific choice of the sieve, the following Lemma \ref{lem:kl} is developed for any generic choice of  sieve indexed by $\bolds$ and $\boldB$. The final piece of the theory developed next tries to addresses two main questions (1) Can we get a sparse network solution whose layer-wise sparsity levels and $L_1$ norms of incident edges (including the bias) of the nodes are controlled at levels $\bolds$ and $\boldB$ respectively? (2) Does this sparse network retain the same predictive performance as the original network?  

In this direction, let
\begin{equation*}
    \xi=\text{min}_{\eta_{\btheta} \in \mathcal{F}(L,\boldk,\bolds, \boldB)}||\eta_{\btheta}-\eta_0||^2_\infty
 \end{equation*}
Based on the values $\bolds$ and $\boldB$, we also define
\begin{align}
\label{e:var-r}
\nonumber    \vartheta_l&={B_l}^2/(k_l+1) + \sum_{m=0,m\neq l}^{L} \log B_m + L + \log k_{l+1} + \log(k_l + 1) + \log n + \log (\sum_{m=0}^L u_m)\\
    r_l&=s_l (k_l+1)  \vartheta_l/n
\end{align} 

Lemma \ref{lem:kl} has two sub conditions. Condition 1. requires that shrinking KL neighborhood of the true density function $P_0$ gets sufficiently large probability. This along with Lemma \ref{lem:test} and \ref{lem:prior} is an essential condition to guarantee the convergence of the true posterior in \eqref{e:posterior}. Condition 2. is the assumption needed to control the KL distance between true posterior and variational posterior and thereby guarantees the convergence of the variational posterior in \eqref{e:var-posterior} (see the discussion under Theorem \ref{thm:var-post} for more details). 

\begin{lemma}[Kullback-Leibler conditions]
\label{lem:kl}
Suppose $\sum_{l=0}^L r_l+\xi \to 0$ and $n (\sum_{l=0}^L r_l+\xi) \to \infty$ and the following two conditions hold for the prior $\widetilde{\Pi}$ in \eqref{e:prior-theta} and some $q \in \mathcal{Q}^{\bf MF}$
\begin{align*}
1.&\hspace{5mm} \widetilde{\Pi}\left(\mathcal{N}_{\sum_{l=0}^L r_l+\xi}\right)\geq \exp(-C_4 n(\sum_{l=0}^L r_l+\xi))\\
2. &\hspace{5mm}
d_{\rm KL}(q,\pi)+n\sum_{\boldz} \int d_{\rm KL}(P_0,P_{\btheta})q(\btheta,\boldz)d\btheta\leq C_5 n(\sum_{l=0}^L r_l+\xi) 
\end{align*}
where $\pi$ is the joint prior of $(\btheta,\boldz)$, $q$ is the joint variational distribution of $(\btheta,\boldz)$ and $\mathcal{N}_{\sum_{l=0}^L r_l+\xi} $ is the KL neighborhood of radius $\sum_{l=0}^L r_l +\xi$.
\end{lemma}

The following result shows that the variational posterior is consistent as long as Lemma \ref{lem:test}, Lemma \ref{lem:prior} and Lemma \ref{lem:kl} hold. The proof of Theorem \ref{thm:var-post}  demonstrates how the validity of these three lemmas imply variational posterior consistency.
\begin{theorem}
\label{thm:var-post}
Suppose Lemma \ref{lem:kl} holds and Lemmas \ref{lem:test} and \ref{lem:prior} hold for $\epsilon_n=\sqrt{(\sum_{l=0}^L r_l+\xi)\sum_{l=0}^L u_l}$.  Then for some slowly increasing sequence $M_n \to \infty$,  $M_n \epsilon_n \to 0$ and $\widetilde{\Pi}^*$ as in \eqref{e:var-posterior-theta}, 
$$ \widetilde{\Pi}^*(\mathcal{H}_{M_n\epsilon_n}^c)\to 0, \quad n \to \infty$$
in $P_0^n$ probability where  $\mathcal{H}_{M_n\epsilon_n}^c=\{\btheta:d_{\rm H}(P_0,P_{\btheta})\leq M_n\epsilon_n\}$ is the Hellinger neighborhood of radius $M_n\epsilon_n$.
\end{theorem}

Note, the above contraction rate depends mainly on two quantities $r_l$ and $\xi$. Note $r_l$ controls the number of nodes in the neural network. If the network is not sparse, then  $r_l$ is $k_{l+1}(k_l+1)\vartheta_l/n$ instead of $s_l (k_l+1)\vartheta_l/n $ which can in turn make the convergence of $\epsilon_n\to 0$ difficult. On the other hand, if $s_l$ and $B_l$ are too small, it will cause $\xi$ to explode since a good approximation to the true function may not exist in a very sparse space. 

\vspace{2mm}
\noindent {\bfseries Remark (Rates as a function of $n$).} {\it Let $L\sim O(\log n)$, $B_l^2 \sim O(k_l+1)$ and  $s_l(k_l+1)=O(n^{1-2\varrho})$, for some $\varrho>0$, then one can work with $\epsilon_n=n^{-\varrho}\log^3 (n)$ as long as $\xi=O(n^{-2\varrho} \log^2(n))$. The exact expression of $\varrho$ is determined by the degree of smoothness of the function $\eta_0$.
}

\vspace{2mm}

\noindent{\bf Proof of Theorem \ref{thm:var-post}}
\vspace{1mm}

\noindent {\bf Discussion.} To further enunciate Lemmas \ref{lem:test} and \ref{lem:prior} consider the quantity $\mathcal{E}_{1n}=\int_{\mathcal{H}_{M_n\epsilon_n}^c}(P_{\btheta}^n/P_0^n)\widetilde{\pi}(\btheta)d\btheta$ as used in the following proof. Here, $\mathcal{E}_{1n}$ can be split into two parts
\begin{align*}
    \mathcal{E}_{1n}=\int_{\mathcal{H}_{M_n\epsilon_n}^c \cap \mathcal{F}(L,\bm{k},\bm{s}^\circ,\bm{B}^\circ)}(P_{\btheta}^n/P_0^n)\widetilde{\pi}(\btheta)d\btheta+\int_{\mathcal{H}_{M_n\epsilon_n}^c\cap \mathcal{F}(L,\bm{k},\bm{s}^\circ,\bm{B}^\circ)^c}(P_{\btheta}^n/P_0^n)\widetilde{\pi}(\btheta)d\btheta
\end{align*}
Whereas Lemma \ref{lem:test} provides a handle on the first term by controlling the covering number of the sieve $\mathcal{F}(L,\bm{k},\bm{s}^\circ,\bm{B}^\circ)$, Lemma \ref{lem:prior} gives a handle on the second term by controlling $\widetilde{\Pi}(\mathcal{F}(L,\bm{k},\bm{s}^\circ,\bm{B}^\circ)^c)$ (for more details we refer to Lemma \ref{lem:covering} in the Appendix \ref{AppendixA}).
 
Next, consider the quantity $\mathcal{E}_{2n}=\log \int (P_{\btheta}^n/P_0^n)\widetilde{\pi}(\btheta)d\btheta$ in the following proof. Lemma \ref{lem:kl} part 1. provides a control on this term (see Lemma \ref{lem:kl-denominator} in the the Appendix \ref{AppendixA} for more details). Finally,  consider the quantity $\mathcal{E}_{3n} =d_{\rm KL}(q,\pi)+\sum_{\boldz}\int \log  (P_0^n/P_{\btheta}^n)q(\btheta,\boldz)d\btheta$ in the following proof. Indeed Lemma \ref{lem:kl} part 2. provides a control on this term (see Lemma \ref{lem:q-determination} in the Appendix \ref{AppendixA} for further details).
\vspace{2mm}
  
\noindent {\bf Proof.} \noindent Let $\widetilde{\Pi}$ and $\widetilde{\Pi}^*$ be as in \eqref{e:true-posterior-theta} and \eqref{e:var-posterior-theta} respectively. Now,
\begin{align*}
d_{\rm KL}(\widetilde{\pi}^*,\widetilde{\pi}(|\mathcal{D}))&=\int_{\mathcal{A}} \widetilde{\pi}^*(\btheta)\log \frac{\widetilde{\pi}^*(\btheta)}{\widetilde{\pi}(\btheta|\mathcal{D})}d\btheta+\int_{\mathcal{A}^c} \widetilde{\pi}^*(\btheta)\log \frac{\widetilde{\pi}^*(\btheta)}{\widetilde{\pi}(\btheta|\mathcal{D})}d\btheta \\
&=-\widetilde{\Pi}^*(\mathcal{A})\int_{\mathcal{A}} \frac{\widetilde{\pi}^*(\btheta)}{\widetilde{\Pi}^*(\mathcal{A})}\log \frac{\widetilde{\pi}(\btheta|\mathcal{D})}{\widetilde{\pi}^*(\btheta)}d\btheta-\widetilde{\Pi}^*(\mathcal{A}^c)\int_{\mathcal{A}^c} \frac{\widetilde{\pi}^*(\btheta)}{\widetilde{\Pi}^*(\mathcal{A}^c)}\log \frac{\widetilde{\pi}(\btheta|\mathcal{D})}{\widetilde{\pi}^*(\btheta)}d\btheta\\
&\geq \widetilde{\Pi}^*(\mathcal{A}) \log 
	\frac{\widetilde{\Pi}^*(\mathcal{A})}{\widetilde{\Pi}(\mathcal{A}|\mathcal{D})}+\widetilde{\Pi}^*(\mathcal{A}^c) \log 
	\frac{\widetilde{\Pi}^*(\mathcal{A}^c)}{\widetilde{\Pi}(\mathcal{A}^c|\mathcal{D})}, \hspace{18mm}\text{Jensen's inequality}
	\end{align*}
where the above lines hold for any set $\mathcal{A}$. Since $\widetilde{\Pi}(\mathcal{A}|\mathcal{D})\leq 1$,
\begin{align*}
&\hspace{20mm}\geq \widetilde{\Pi}^*(\mathcal{A}) \log 
	\widetilde{\Pi}^*(\mathcal{A})+\widetilde{\Pi}^*(\mathcal{A}^c) \log 
	\widetilde{\Pi}^*(\mathcal{A}^c)-\widetilde{\Pi}^*(\mathcal{A}^c) \log \widetilde{\Pi}(\mathcal{A}^c|\mathcal{D})\\
	&\hspace{20mm} \geq  -\widetilde{\Pi}^*(\mathcal{A}^c) \log \widetilde{\Pi}(\mathcal{A}^c|\mathcal{D})-\log 2, 
	\hspace{4mm} (\because \hspace{1mm} x\log x+(1-x)\log (1-x) \geq -\log 2)
	\\
	&\hspace{20mm}=-\widetilde{\Pi}^*(\mathcal{A}^c) \Bigg(\underbrace{\log \int_{\mathcal{A}^c} (P_{\btheta}^n/P_0^n)\widetilde{\pi}(\btheta)d\btheta}_{\mathcal{E}_{1n}}-\underbrace{\log \int (P_{\btheta}^n/P_0^n)\widetilde{\pi}(\btheta)d\btheta}_{\mathcal{E}_{2n}}\Bigg)-\log 2
\end{align*}
The above representation is similar to the proof of Theorems 3.1 and 3.2 in \cite{BHAT2021}. For any $q \in \mathcal{Q}^{\bf MF}$,
\begin{align}
\label{e:kl-relation}
\nonumber -\widetilde{\Pi}^*(\mathcal{A}^c) \mathcal{E}_{1n}
&\leq d_{\rm KL}(\widetilde{\pi}^*,\widetilde{\pi}(|\mathcal{D}))-\widetilde{\Pi}^*(\mathcal{A}^c) \mathcal{E}_{2n}+\log 2\\
\nonumber &\leq  d_{\rm KL}(\pi^*,\pi(|\mathcal{D}))-\widetilde{\Pi}^*(\mathcal{A}^c) \mathcal{E}_{2n}+\log 2 \hspace{10mm} \text{by Lemma \ref{lem:kl-upp}}\\
\nonumber &\leq  d_{\rm KL}(q,\pi(|\mathcal{D}))-\widetilde{\Pi}^*(\mathcal{A}^c) \mathcal{E}_{2n}+\log 2\hspace{10mm} \text{$\pi^*$ is the KL minimizer}\\
\nonumber &\leq   \underbrace{d_{\rm KL}(q,\pi)+\sum_{\boldz}\int \log  \frac{P_0^n}{P_{\btheta}^n}q(\btheta,\boldz)d\btheta}_{\mathcal{E}_{3n}}+(1-\widetilde{\Pi}^*(\mathcal{A}^c)) \mathcal{E}_{2n}+\log 2 \\
&= \mathcal{E}_{3n}+(1-\widetilde{\Pi}^*(\mathcal{A}^c))\mathcal{E}_{2n}+\log 2
\end{align}
where the fourth inequality in the above equation follows since 
\begin{align*}
    d_{\rm KL}(q,\pi(|\mathcal{D}))&=\sum_{\boldz}\int (\log q(\btheta,\boldz)-\log P_{\btheta}^n-\log \pi(\btheta,\boldz)+\log m(\mathcal{D})) q(\btheta,\boldz)d\btheta\\
    &=\underbrace{\sum_{\boldz}\int (\log q(\btheta,\boldz)-\log \pi(\btheta,\boldz))q(\btheta,\boldz)d\btheta}_{d_{\rm KL}(q,\pi)}+\sum_{\boldz}\int (\log P_0^n-\log P_{\btheta}^n)q(\btheta,\boldz)d\btheta\\
    &+\underbrace{  \log  m(\mathcal{D}) -\log P_0^n}_{\mathcal{E}_{2n}} 
\end{align*}
where $m(\mathcal{D})$ is the marginal distribution of data as in \eqref{e:posterior}.

Take $\mathcal{A}=\mathcal{H}_{M_n \epsilon_n}^c=\{\btheta: d_{\rm H}(P_0,P_{\btheta})>M_n\epsilon_n\}$

\noindent If Lemma \ref{lem:test} and \ref{lem:prior} hold, then by Lemma \ref{lem:covering}, it can be shown that $\mathcal{E}_{1n}\leq - n C M_n^2 \epsilon_n^2/\sum u_l$ for any $M_n \to \infty$ with high probability.

\noindent If Lemma \ref{lem:kl} condition 1. holds, then by Lemma \ref{lem:kl-denominator} , $\mathcal{E}_{2n} \leq n M_n  (\sum_{l=0}^L r_l+\xi)$ for any $M_n \to \infty$.

\noindent If Lemma \ref{lem:kl} condition 2. hold, then by Lemma \ref{lem:q-determination}, $\mathcal{E}_{3n} \leq n M_n (\sum_{l=0}^L r_l+\xi)$ for any $M_n \to \infty$.

\noindent Therefore, by \eqref{e:kl-relation}, we get
\begin{align*}
      \frac{n C M_n^2 \epsilon_n^2}{\sum u_l} \widetilde{\Pi}^*\left(\mathcal{H}_{M_n \epsilon_n}^c\right)&\leq   n M_n  (\sum_{l=0}^L r_l+\xi)+n M_n  (\sum_{l=0}^L r_l+\xi)+\log 2\\
      &\leq   n M_n  (\sum_{l=0}^L r_l+\xi)+n M_n  (\sum_{l=0}^L r_l+\xi)+M_n (\sum_{l=0}^L r_l+\xi) \\
      \implies \widetilde{\Pi}^*\left(\mathcal{H}_{M_n \epsilon_n}^c\right)&\leq \frac{3 M_n(\sum_{l=0}^L r_l+\xi) \sum u_l}{C_1 M_n^2 \epsilon_n^2} 
\end{align*}
Taking $\epsilon_n=\sqrt{\sum_{l=0}^L (r_l+\xi) \sum u_l}$ and noting $M_n \to \infty$, the proof follows. \hfill \qedsymbol
\\

We next give conditions on the prior probabilities $\lambda_l$ and $\sigma_0$ to guarantee that Lemmas \ref{lem:test}, \ref{lem:prior} and \ref{lem:kl} hold. This in turn implies the conditions of Theorem \ref{thm:var-post} hold and variational posterior is consistent.

\begin{corollary}\label{post_contraction_main_theorem}
Let $\sigma_0^2=1$,  $-\log \lambda_l=\log (k_{l+1})+C_l(k_l+1) \vartheta_l$, then conditions of Theorem \ref{thm:var-post} hold and $\widetilde{\Pi}^*$ as in \eqref{e:var-posterior-theta} satisfies 
$$ \widetilde{\Pi}^*(\mathcal{H}_{M_n\epsilon_n}^c)\to 0, \quad n \to \infty$$
in $P_0^n$ probability where and  $\mathcal{H}_{M_n\epsilon_n}=\{\btheta:d_{\rm H}(P_0,P_{\btheta})\leq M_n\epsilon_n\}$ is the Hellinger neighborhood of radius $M_n\epsilon_n$.
\end{corollary}

The proof of the corollary has been provided in Appendix \ref{AppendixA}.

In the preceding corollary, note that our expression of prior inclusion probability varies as a function of $l$ thereby providing a handle on layer-wise sparsity. Indeed, using these expressions in numerical studies further substantiates the theoretical framework developed in this section.

\vspace{2mm}

\noindent{\bf Remark (Optimal Contraction).} {\it For a fixed choice of $\boldk$, the optimal contraction rate is achieved at 
$\bolds^\star, \boldB^\star=\underset{\bolds,\boldB}{\rm argmin} (\sum r_l+\xi)$. Thus, $\bolds^\star$ and $\boldB^\star$ are the optimal values of $\bolds$ and $\boldB$  which give the best sparse network with minimal loss in the true accuracy. The corresponding probability expressions in Corollary \ref{post_contraction_main_theorem} can be accordingly modified by setting $\bolds=\bolds^\star$ and $\boldB=\boldB^\star$ in the expressions of $\vartheta_l$ and $r_l$ in \eqref{e:var-r}.}

\section{Implementation Details}

\label{Sec:implement}
\noindent {\bfseries Evidence Lower Bound}. The ELBO presented in \eqref{e:elbo} is given by
$\mathcal{L}= -E_{q} [\log P_{\btheta}^n]+d_{\rm KL}(q,\pi)$
which is further simplified as
\begin{align*}
& -E_{q} [\log P_{\btheta}^n]+d_{\rm KL}(q,\pi) \\
& = -\E_{q(\btheta|\boldz)q(\boldz)}[\log P_{\btheta}^n] + d_{\rm KL}\left(q(\btheta|\boldz)q(\boldz),\pi(\btheta|\boldz)\pi(\boldz)\right)\\
& = -\E_{q(\btheta|\boldz)q(\boldz)}[\log P_{\btheta}^n]+ \sum_{l,j} d_{\rm KL}(q(z_{lj})||\pi(z_{lj})) \\
& \qquad + \sum_{l,j} \left[ q(\boldz_{lj} = 1) d_{\rm KL}( q(\overline{\boldw}_{lj}|z_{lj}=1) || \pi(\overline{\boldw}_{lj}|z_{lj}=1)) + q(\boldz_{lj} = 0) d_{\rm KL}( q(\overline{\boldw}_{lj}|z_{lj}=0) || \pi(\overline{\boldw}_{lj}|z_{lj}=0)) \right] \\
& = -\E_{q(\btheta|\boldz)q(\boldz)}[\log P_{\btheta}^n]+ \sum_{l,j} d_{\rm KL}(q(z_{lj})||\pi(z_{lj})) + \sum_{l,j} q(\boldz_{lj} = 1) d_{\rm KL}( q(\overline{\boldw}_{lj}|z_{lj}=1) || \pi(\overline{\boldw}_{lj}|z_{lj}=1)) \\
& = -\E_{q(\btheta|\boldz)q(\boldz)}[\log P_{\btheta}^n]+ \sum_{l,j} d_{\rm KL}(q(z_{lj})||\pi(z_{lj})) + \sum_{l,j} q(\boldz_{lj} = 1) d_{\rm KL}( N(\bmu_{lj},\text{diag}(\bsigma^2_{lj})) || N(0,\sigma_0^2 \boldI)) 
\end{align*}
The KL of discrete variables appearing in the above expression creates a challenge in practical implementation. \cite{Jang2017Gumbel}, \cite{Maddison2017Contrelax} proposed to replace discrete random variable with its continuous relaxation. Specifically, the continuous relaxation approximation is achieved through Gumbel-softmax (GS) distribution, that is  $q(z_{lj}) \sim {\rm Ber}(\gamma_{lj})$ is approximated by $q(\tilde{z}_{lj}) \sim {\rm GS}(\gamma_{lj},\tau)$, where   
$$ \tilde{z}_{lj} = (1+\exp(-\eta_{lj}/\tau))^{-1}, \quad \eta_{lj} = \log(\gamma_{lj}/(1-\gamma_{lj})) + \log (u_{lj}/(1-u_{lj})), \quad u_{lj} \sim U(0,1)$$
where $\tau$ is the temperature. We set $\tau=0.5$ for this paper (also see section 5 in \cite{Bai-Guang-2020}). $\tilde{z}_{lj}$ is used in the backward pass for easier gradient calculation, while $z_{lj}$ will be used for selecting nodes in the forward pass. We use non-centered parameterization for the Gaussian slab variational approximation where $N(\bmu_{lj},\text{diag}(\bsigma^2_{lj}))$ is reparameterized as $\bmu_{lj}+\bsigma_{lj}\odot \bzeta_{lj}$ for $\bzeta_{lj}\sim N(0,\boldI)$, where $\odot$ denotes the entry-wise (Hadamard) product.  
    
\begin{algorithm}[t]
\caption{Variational inference in SS-IG Bayesian neural networks}
\label{alg_SSIG}
\begin{algorithmic}
    \STATE {\bfseries Inputs:} training dataset, network architecture, and optimizer tuning parameters.
    \STATE \textit{Model inputs:} prior parameters for $\btheta$, $\boldz$. 
    \STATE \textit{Variational inputs:} number of Monte Carlo samples $S$.
    \STATE {\bfseries Output:} Variational parameter estimates of network weights and sparsity.
    \STATE {\bfseries Method:} Set initial values of variational parameters.
   \REPEAT
   \STATE Generate $S$ samples from $\bzeta_{lj} \sim N(0,\boldI)$ and $u_{lj}\sim U(0,1)$
   \STATE Generate $S$ samples for $( z_{lj}, \tilde{z}_{lj})$ using $u_{lj}$
   \STATE Use $\bmu_{lj}, \bsigma_{lj}, \bzeta_{lj}$ and $z_{lj}$ to compute loss (ELBO) in forward pass
   \STATE Use $\bmu_{lj}, \bsigma_{lj}, \bzeta_{lj}$ and $\tilde{z}_{lj}$ to compute gradient of loss in backward pass
   \STATE Update the variational parameters with gradient of loss using stochastic gradient descent algorithm (e.g. Adam \citep{Kingma2015Adam})
   \UNTIL{change in ELBO $< \epsilon$}
\end{algorithmic}
\end{algorithm}

\section{Numerical Experiments}

\label{Sec:Numerical}
In this section, we present several numerical experiments to demonstrate the performance of our spike-and-slab independent Gaussian (SS-IG) Bayesian neural networks which we implement in PyTorch \citep{PyTorch2019NeurIPS}. Further, to evaluate the efficacy of the variational inference we benchmark our model on synthetic as well as real datasets. Our numerical investigation justifies the use of proposed choices of prior hyperparameters specifically layer-wise prior inclusion probabilities, which in turn substantiates the significance of our theoretical developments. \SJ{With fully Bayesian treatment, we are also able to quantify the uncertainties for the parameter estimates and variational inference helps to scale our model to large network architectures as well as complex datasets.}

We compare our sparse model with a node selection technique: horseshoe BNN (HS-BNN) \citep{Ghosh-JMLR-2018} and an edge selection technique: spike-and-slab BNN (SV-BNN) \citep{Bai-Guang-2020} in the second simulation study and UCI regression dataset examples.  We use optimal choices of prior parameters and fine tuning parameters provided by the authors of HS-BNN and SV-BNN in their respective models. Further we compare our model against dense variational BNN model (VBNN) \citep{blundell2015weight} in all of the experiments. Since it has no sparse structure, it serves as a baseline allowing to check whether sparsity compromises accuracy. In all the experiments, we fix $\sigma^2_0=1$ and  $\sigma^2_e=1$. For our model, the choices of  layer-wise $\lambda_l$ follow from Corollary \ref{post_contraction_main_theorem}: $\lambda_l=(1/k_{l+1})\exp(- C_l(k_l+1)\vartheta_l)$. We take $C_l$ values in the negative order of 10 such that prior inclusion probabilities do not fall below $10^{-50}$ otherwise $\lambda_l$ values close to 0 might prune away all the nodes from a layer (check appendix~\ref{AppendixB} for more discussion).  The remaining tuning parameter details such as learning rate, minibatch size, and initial parameter choice are provided in the appendix~\ref{AppendixB}. The prediction accuracy is calculated using variational Bayes posterior mean estimator with 30 Monte Carlo samples in testing phase.
\vspace{2mm}

\noindent {\bf Node sparsity estimates.} \SJ{In our experiments, we provide node sparsity estimates for each hidden layer separately. For all models, the node sparsity in a given hidden layer is the ratio of number of neurons with atleast one nonzero incoming edge over the original number of neurons present in that layer before training. The layer-wise node sparsity estimates give clear picture of the structural compactness of the trained model during test time. The structurally compact trained model has lower latency during inference stage.
}

\subsection{Simulation Study - I} 
We consider a two dimensional regression problem where the true response $y_0$ is generated by sampling $X$ from $U([-1,1]^2)$ and feeding it to a deep neural network with known parameters. We add a random Gaussian noise with $\sigma=5\%\sqrt{Var(y_0)}$ to $y_0$ to get noisy outputs $y$. We create the dataset using a shallow neural network consisting of 2 inputs, one hidden layer with 2 nodes and 1 output (2-2-1 network). We train our SS-IG model and VBNN model using a single hidden layer network with 20 neurons in the hidden layer and administer sigmoid activation. Each model is trained till convergence. We found that both models give competitive predictive performance while fitting the given data. In Figure~\ref{fig:simulation-1_20_nodes} we plot the magnitudes of the incoming weights into the hidden layer nodes using boxplots. Our model with the help of spike and slab prior is able to prune away redundant nodes not required for fitting the model. Since VBNN is densely connected, it shows all the nodes being active in its final model. From this experiment, it is clear that neural networks can be pruned leading to more compact models at inference stage without compromising the accuracy. We also performed the same experiment with a wider neural network consisting of 100 nodes in the single hidden layer and provide the results in the appendix~\ref{AppendixB}. There again we show that our model can easily recover very sparse solution with competitive predictive performance.

\begin{figure}[H]
\centering
\begin{subfigure}{.5\textwidth}
  \centering
  \includegraphics[width=\linewidth]{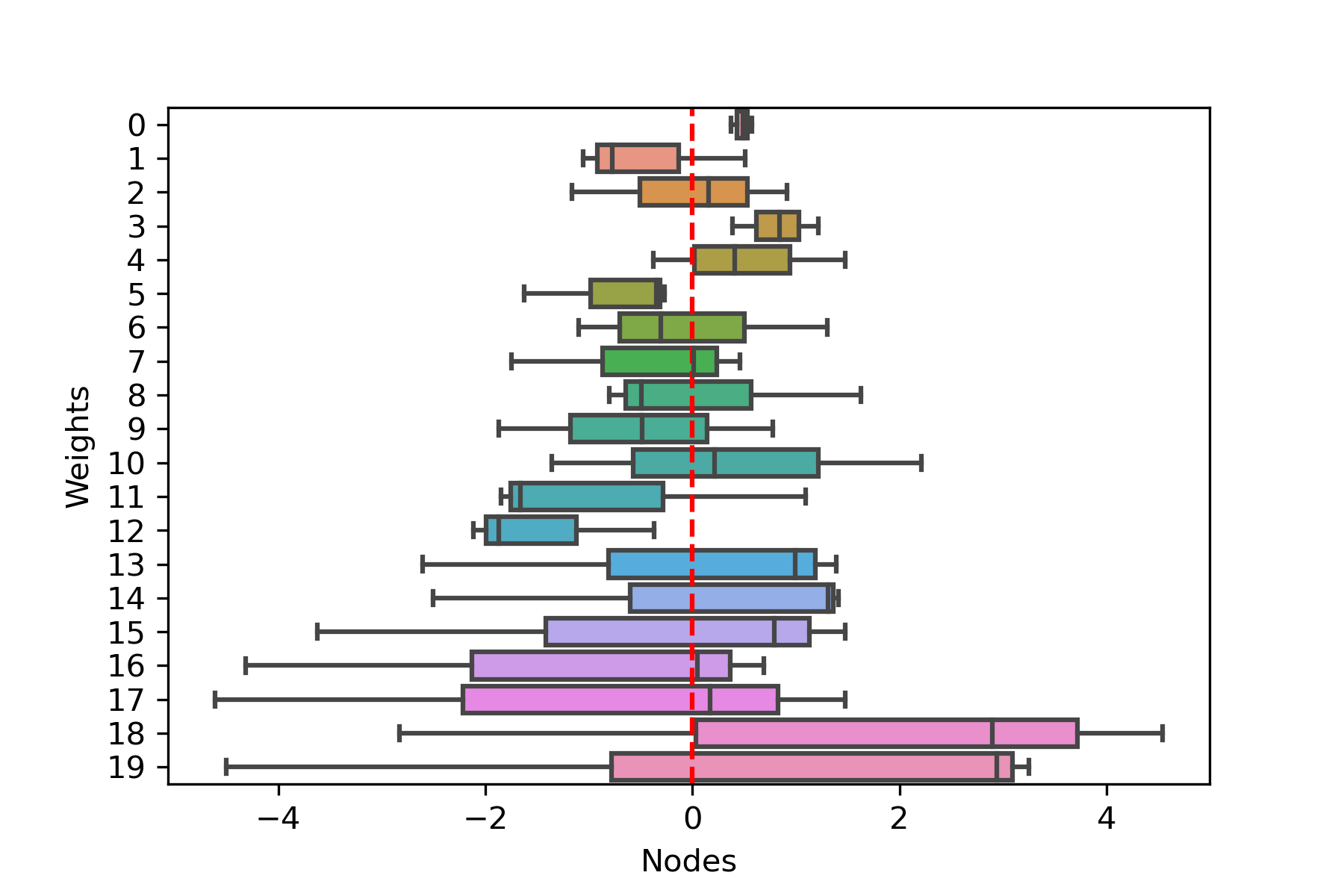}
  \caption{VBNN}
  \label{fig:accuracy1}
\end{subfigure}%
\begin{subfigure}{.5\textwidth}
  \centering
  \includegraphics[width=\linewidth]{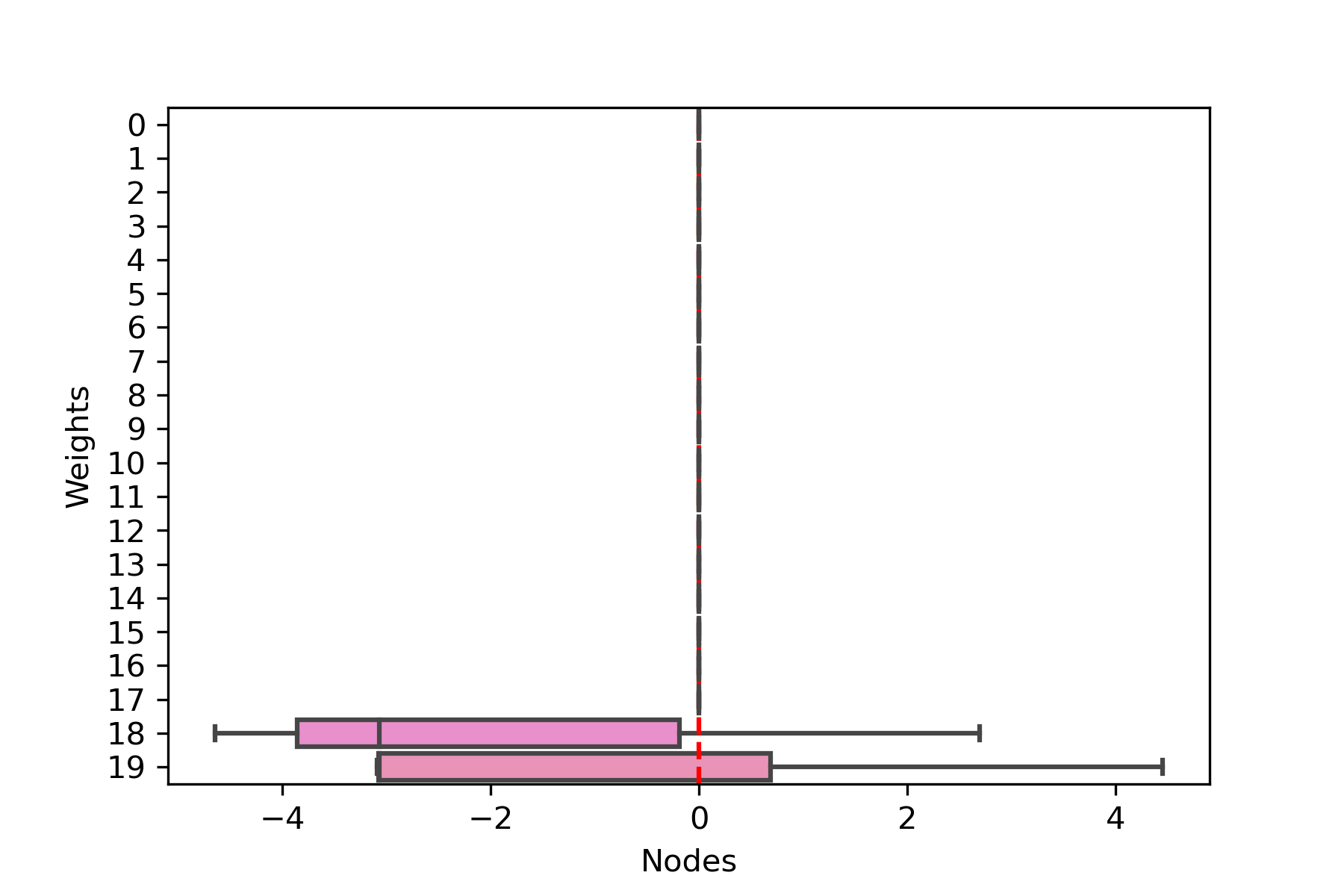}
  \caption{SS-IG}
  \label{fig:sparsity1}
\end{subfigure}
\caption{Node-wise weight magnitudes recovered by VBNN and proposed SS-IG model in the synthetic regression data generated using 2-2-1 network. The boxplots show the distribution of incoming weights into a given hidden layer node.}
\label{fig:simulation-1_20_nodes}
\end{figure}

\subsection{Simulation Study - II}
We consider a nonlinear regression example where we generate the data from the following model:
\begin{equation*}
    y = \frac{7x_2}{1+x_1^2} + \sin(x_3x_4) + 2x_5 + \varepsilon,
\end{equation*}
where $\varepsilon \sim N(0,1)$. Further all the covariates are i.i.d. $N(0,1)$ and independent of $\varepsilon$. We generated 3000 data entries to create the training data for the experiment. Additional 1000 observations were generated for testing. We modeled this data using 2-hidden layer neural network which consists of 20 neurons per hidden layer. Sigmoid activation function is administered for each model used for comparative analysis. \SJ{Table \ref{Node_selection_sim_results} provides the RMSEs on train and test dataset as well as layer-wise node sparsity estimates for SS-IG, SV-BNN, HS-BNN, and VBNN models.} Our model is extremely well at pruning redundant nodes which leads to the most compact model compared to the other sparse models: SV-BNN and HS-BNN. Moreover it exhibits lower root mean squared error (RMSE) values \SJ{on test data} among the sparse models while showing similar predictive performance compared to the densely connected VBNN. This experiment further underscores the major \SJ{benefit of our proposed approach} to generate very compact models which could reduce computational times and memory usage at inference stage.
\begin{table}[t] 
\caption{Performance of the proposed SS-IG, SV-BNN, HS-BNN, and VBNN models in simulation study II. Each model was trained for 10k epochs with learning rate $5\times10^{-3}$. Mean and S.D. of RMSE values and median sparsity estimates were calculated from last 1000 epochs (with jump of 10 giving us sample of 100). \SJ{The sparsity estimates are given as a tuple of 2 values representing layer-1 and layer-2 node sparsities.}}
\label{Node_selection_sim_results}
    \begin{center}
        \begin{tabular}{lccc}
            \toprule
            Model  & \multicolumn{1}{c}{Train RMSE} & \multicolumn{1}{c}{Test RMSE} & \multicolumn{1}{c}{Sparsity Estimate} \\
            \midrule
            SS-IG  & 1.2087$\pm$0.0490 & 1.1947$\pm$0.0587 & (0.35,0.05) \\
            SV-BNN & 1.2897$\pm$0.0323 & 1.2760$\pm$0.0363 & (0.45,0.35) \\
            HS-BNN & 1.2580$\pm$0.0305 & 1.2436$\pm$0.0394 & (1.00, 1.00)  \\
            VBNN & 1.1661$\pm$0.0335 & 1.1614$\pm$0.0349 & NA  \\
            \bottomrule
        \end{tabular}
    \end{center}
\end{table}

\subsection{UCI regression datasets} 
We apply our model to traditional UCI regression datasets \citep{Dua2019UCI} and contrast our performance against SV-BNN, HS-BNN, and VBNN models. We follow the protocol proposed by \citep{Hernandez-Lobato-2015} and train a single layer neural network with sigmoid activations. For smaller datasets - \textit{Concrete, Wine, Power Plant, Kin8nm}, we take 50 nodes in the hidden layer, while for larger datasets - \textit{Protein, Year}, we take 100 nodes in the hidden layer. We spilt data randomly while maintaining 9:1 train-test ratio in each case and for smaller datasets we repeat this technique 20 times. In \textit{Protein} data we perform 5 repetitions while in \textit{Year} data we use a single random split (more details in the appendix~\ref{AppendixB}). For the comparative analysis, we benchmark against SV-BNN, HS-BNN and VBNN. Moreover, VBNN test RMSEs serve as baseline in each dataset. \SJ{Table~\ref{UCI_regression_results} summarises our results including the sparsity estimate representing hidden layer-1 node sparsity (since there is only one hidden layer in the networks considered).} 

\begin{table}[b] 
\caption{Results on UCI regression datasets}
\label{UCI_regression_results}
    \begin{small}\addtolength{\tabcolsep}{-2pt}
    \begin{center}
        \begin{tabular}{lccccccc}
            \toprule
            \multicolumn{1}{c}{} & \multicolumn{1}{l}{} & \multicolumn{4}{c}{Test RMSE} & \multicolumn{2}{c}{Sparsity Estimate} \\ \cmidrule(lr){3-6}\cmidrule(lr){7-8}
             Dataset & $n(k_0)$  & \multicolumn{1}{c}{SS-IG} & \multicolumn{1}{c}{SV-BNN} & \multicolumn{1}{c}{HS-BNN} & \multicolumn{1}{c}{VBNN} & \multicolumn{1}{c}{SS-IG} & \multicolumn{1}{c}{SV-BNN} \\ 
            \midrule
            Concrete & 1030 (8) & 7.92$\pm$0.68 & 8.22$\pm$0.70 & 5.34$\pm$0.53 & 7.34$\pm$0.62 & 0.42$\pm$0.06 & 0.98$\pm$0.02 \\ 
            Wine & 1599 (11)& 0.66$\pm$0.05 & 0.65$\pm$0.05 & 0.66$\pm$0.05 & 0.64$\pm$0.05 & 0.18$\pm$0.05 & 0.87$\pm$0.04 \\ 
            Power Plant & 9568 (4) & 4.28$\pm$0.20 & 4.32$\pm$0.19 & 4.34$\pm$0.18 & 4.27$\pm$0.17 & 0.18$\pm$0.03 & 0.24$\pm$0.03 \\ 
            Kin8nm & 8192 (8) & 0.09$\pm$0.00 & 0.11$\pm$0.01 & 0.10$\pm$0.00 & 0.09$\pm$0.00 & 0.43$\pm$0.04 & 0.47$\pm$0.04 \\ 
            Protein & 45730 (9) & 4.85$\pm$0.05 & 4.93$\pm$0.06 & 4.59$\pm$0.02 & 4.78$\pm$0.06 & 0.81$\pm$0.03 & 0.93$\pm$0.03\\ 
            Year & 515345 (90) & 8.68$\pm$NA & 8.78$\pm$NA & 9.33$\pm$NA & 8.67$\pm$NA & 0.71$\pm$NA & 0.78$\pm$NA\\ 
            \bottomrule
        \end{tabular}
    \end{center}
    \end{small}
\end{table}

We achieve lower RMSEs compared to SV-BNN and HS-BNN in \textit{Power Plant, Kin8nm,} and \textit{Year} datasets and in other cases we achieve comparable RMSE values. In all the datasets, our predictive performance is close to the dense baseline of VBNN. \SJ{We provide node sparsity estimates in our SS-IG and SV-BNN models.} HS-BNN was not able to achieve sparse structure which is consistent with the results provided in the appendix of \citep{Ghosh-JMLR-2018}. In contrast to HS-BNN, our model sparsifies the model during training without requiring ad-hoc thresholding rule for pruning. Table \ref{UCI_regression_results} demonstrates that our model uniformly achieves better sparsity than SV-BNN. In particular, \textit{Concrete} and \textit{Wine} datasets show the high compressive ability of our model over SV-BNN leading to very compact models for inference.

\subsection{\SJ{Image classification datasets}}
Here, we benchmark the empirical performance of our proposed SS-IG method on network architectures and image classification datasets used in practice.
\vspace{2mm}

\noindent {\bf Baselines.} We compare our model against VBNN model which serves as a dense baseline to gauge the trade-off between predictive performance and sparsity. Moreover, to highlight the complementary behavior in memory and computational efficiency of node selection compared to edge selection achieved via Bayesian spike-and-slab prior framework, we compare our model against the edge selection model, SV-BNN.
\vspace{2mm}

\noindent {\bf Network architectures.} We consider 2 neural network model architectures: (i) multi-layer perceptron (MLP), and (ii) Lenet-Caffe. In MLP model, we take 2 hidden layers with 400 neurons in each layer. Output layer has 10 neurons since there are 10 classes in both datasets. Next, Lenet-Caffe model has 2 convolutional layers with 20 and 50 feature maps respectively with filter size $5\times5$ for both layers. In SS-IG model, for convolution layers, we prune output channels (similar to neurons in linear layers) using our spike-and-slab prior where each output channel is assigned an Bernoulli variable to collectively prune parameters incident on that channel. We apply $2\times2$ max pooling layer after each convolution layer. The flattened feature layer after second convolution layer has size $4*4*50=800$ serving as input to the fully connected block, where there are 2 hidden layers with 800 and 500 neurons respectively. The output layer has 10 neurons.
\vspace{2mm}

\noindent{\bf Datasets.} We apply each network architecture on 2 image classification datasets: (i) MNIST: dataset of 60,000 small square 28×28 pixel grayscale images of handwritten single digits between 0 and 9, and (ii) Fashion-MNIST: dataset of 60,000 small square 28×28 pixel grayscale images of items of 10 types of clothing. We preprocess the images in the MNIST data by dividing their pixel values by 126. In Fashion-MNIST data, we horizontally flip images at random with probability of 0.5.
\vspace{2mm}

\noindent {\bf Metrics.}
We quantify the predictive performance using the accuracy of the test data (MNIST and Fashion-MNIST). Besides the test accuracy, we evaluate our model against SV-BNN using the metrics that relate to the model compression and computational complexity. First the {\it compression ratio} is the ratio of number of nonzero weights in the compressed network versus the dense model and is an indicator of storage cost at test-time. Next, we present layer-wise node sparsities in MLP experiments to highlight the computational speedups at test-time. In Lenet-Caffe experiments, we provide the {\it floating point operations (FLOPs) ratio} which is the ratio of number of FLOPs required to predict y from x during test time in the compressed network versus its dense counterpart. We have detailed the FLOPs calculation in neural networks in Appendix~\ref{AppendixB}.
\vspace{2mm}

\noindent {\bf Nonlinear activation.} We use swish activations \citep{Elfwing-et-al-2018, Ramachandran-et-al-2017} instead of ReLUs in our proposed SS-IG model to avoid the dying neuron problem where ReLU neurons become inactive and only output 0 for any input \citep{Lu-2020-DyingReLU}. Specifically in large scale datasets turning off a node with more than 100 incoming edges adversely impacts the training process of ReLU networks. Smoother activation functions such as sigmoid, tanh, swish etc help alleviate this problem. We choose swish since it has the best performance. For VBNN and SV-BNN, we use ReLU activations as recommended by their authors.

\subsubsection*{MLP Experiments}
The results of MLP network experiments on MNIST and Fashion-MNIST are presented in Figure~\ref{fig:MLP-Experiments}. We provide test data accuracy, model compression ratio, and layer-wise node sparsities in each experiment. 

\begin{figure}[htp]
\centering
\begin{subfigure}[b]{0.375\textwidth}
    \centering
    \includegraphics[width=\textwidth]{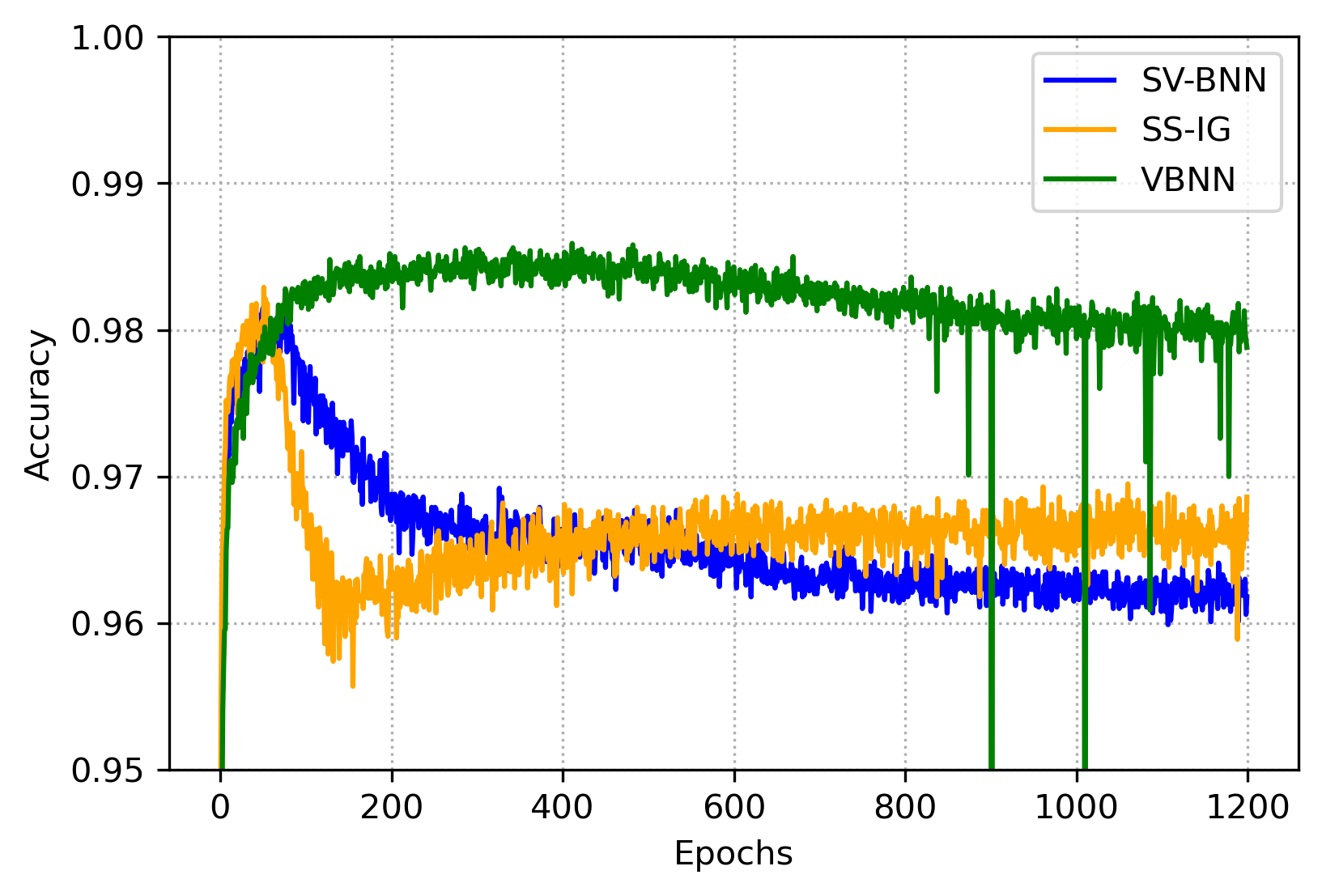}
    \caption{{\small Test accuracy}}    
    \label{fig:MLP-MNIST-Test-Acc}
\end{subfigure}
\begin{subfigure}[b]{0.375\textwidth}
    \centering
    \includegraphics[width=\textwidth]{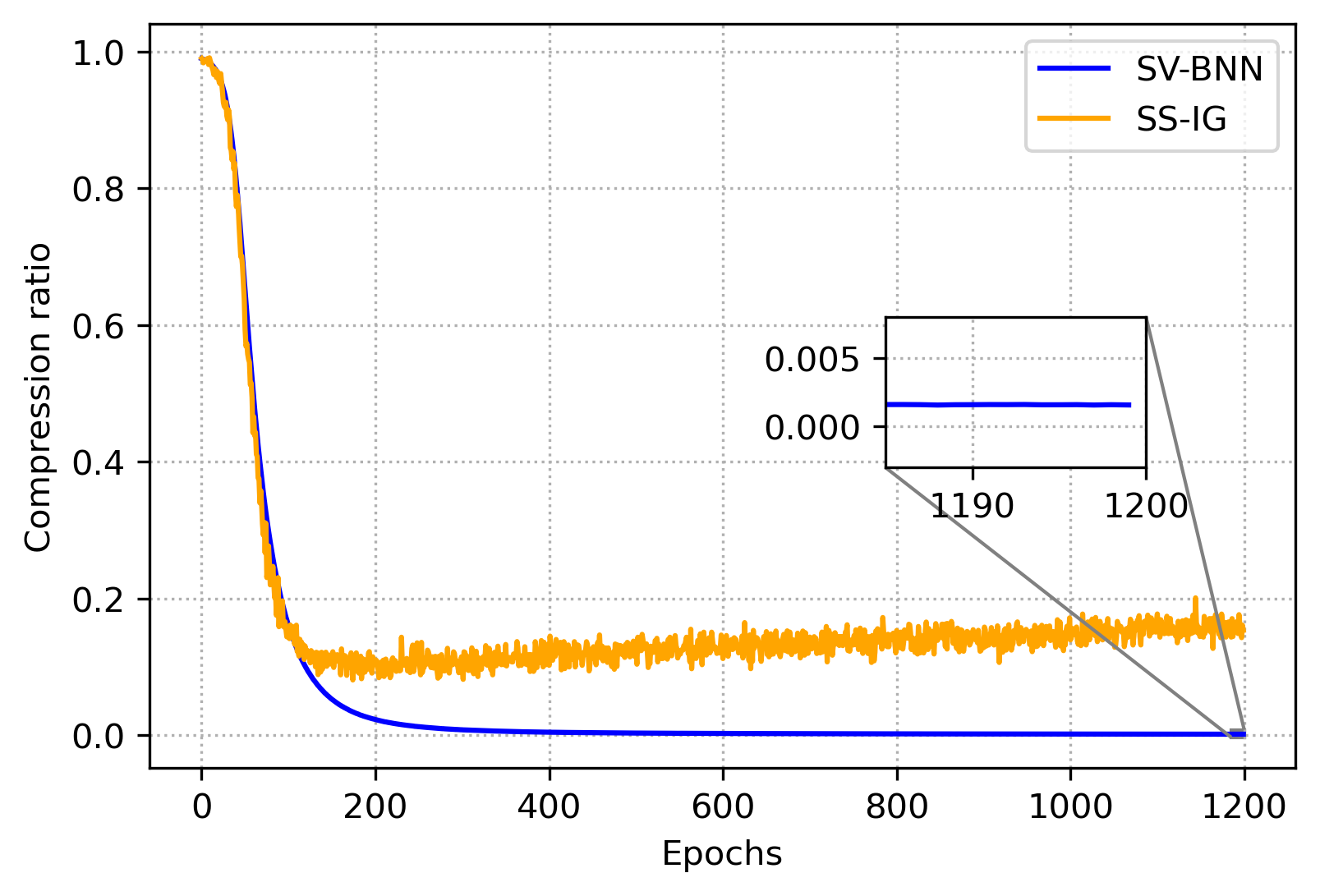}
    \caption{{\small Compression ratio}}    
    \label{fig:MLP-MNIST-Compression-Ratio}
\end{subfigure}
\vskip\baselineskip
\begin{subfigure}[b]{0.375\textwidth}
    \centering
    \includegraphics[width=\textwidth]{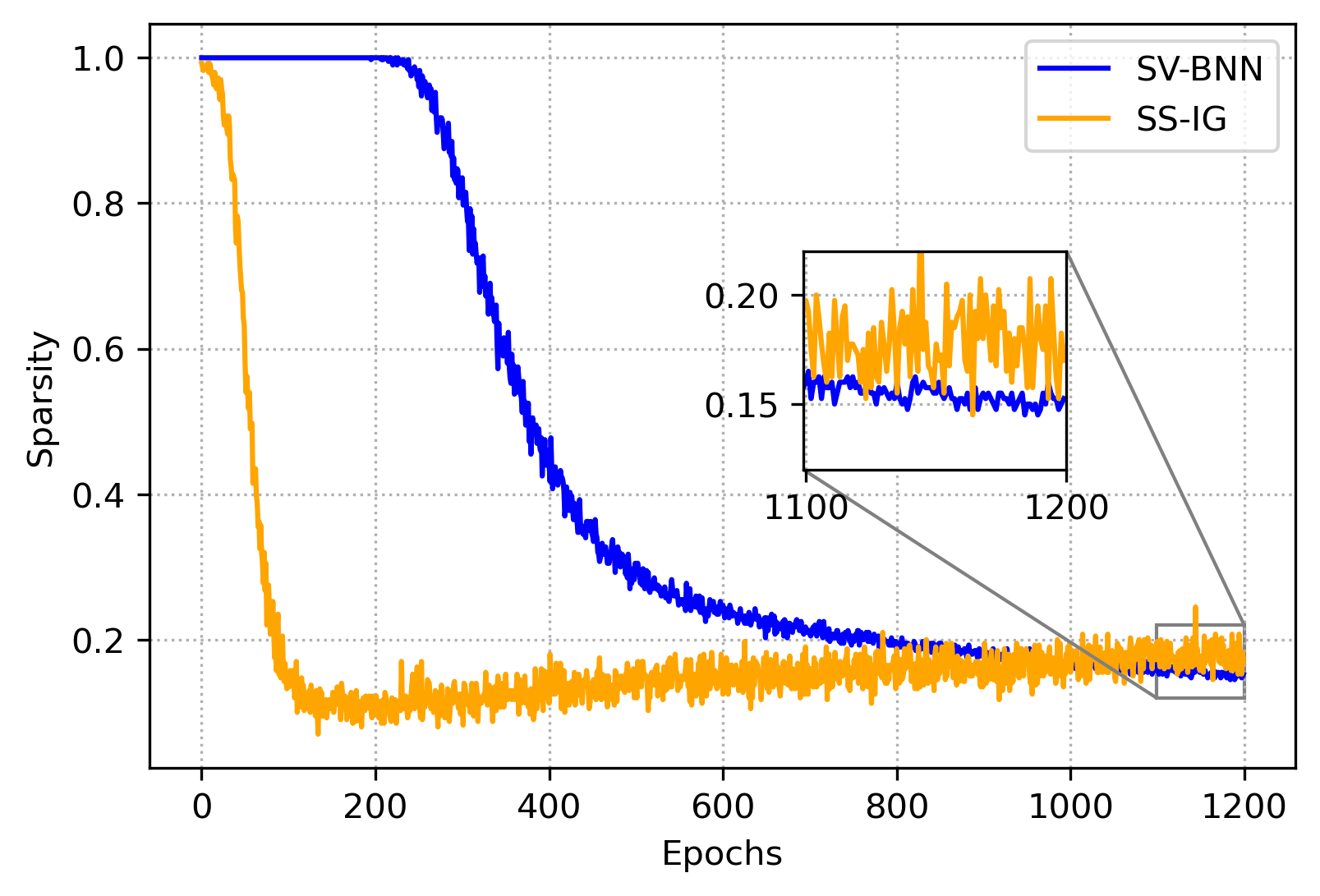}
    \caption{{\small Layer-1 node sparsity}}    
    \label{fig:MLP-MNIST-layer1-sparsity}
\end{subfigure}
\begin{subfigure}[b]{0.375\textwidth}
    \centering
    \includegraphics[width=\textwidth]{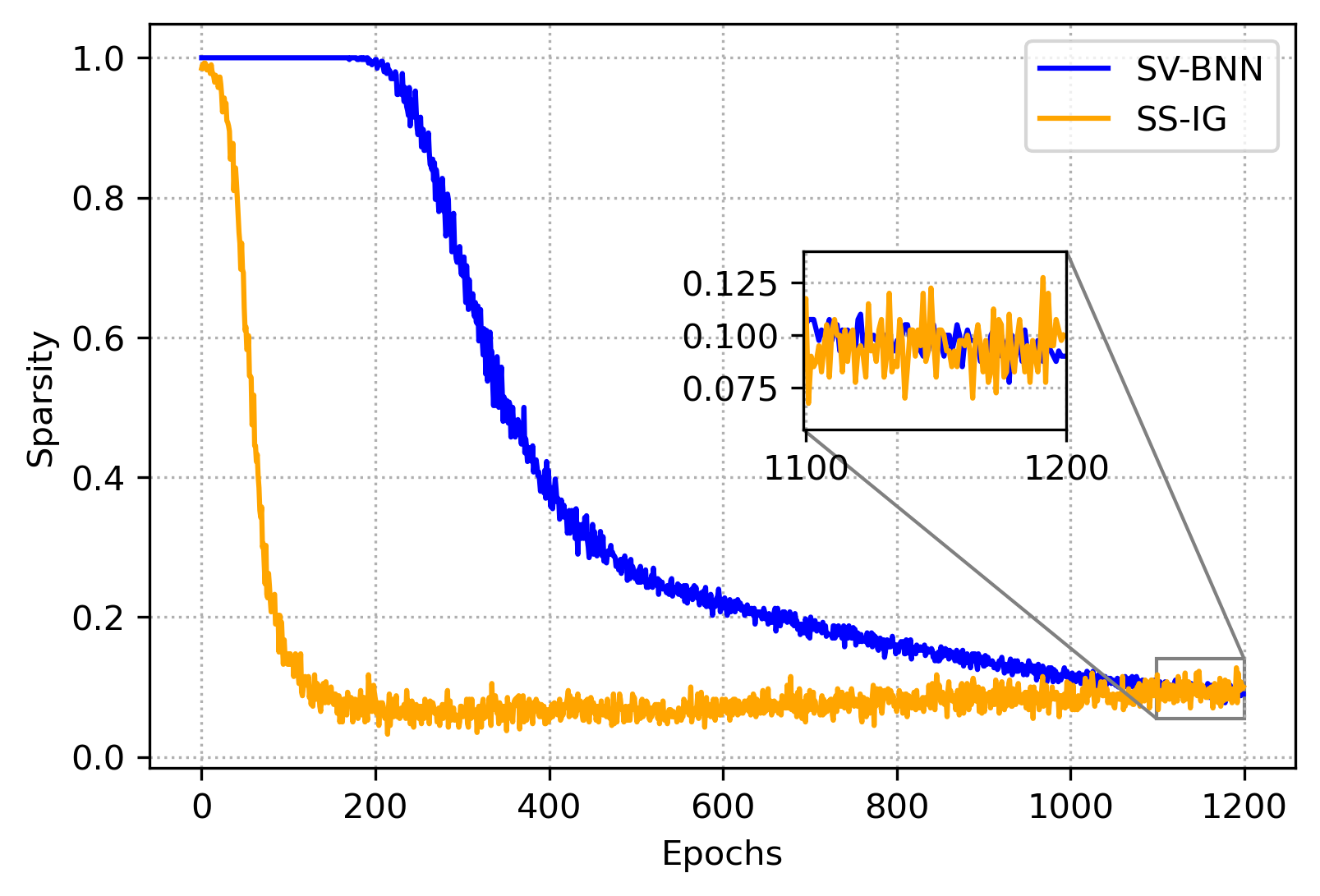}
    \caption{{\small Layer-2 node sparsity}}    
    \label{fig:MLP-MNIST-layer2-sparsity}
\end{subfigure}
\vskip\baselineskip
\begin{subfigure}[b]{0.375\textwidth}
    \centering
    \includegraphics[width=\textwidth]{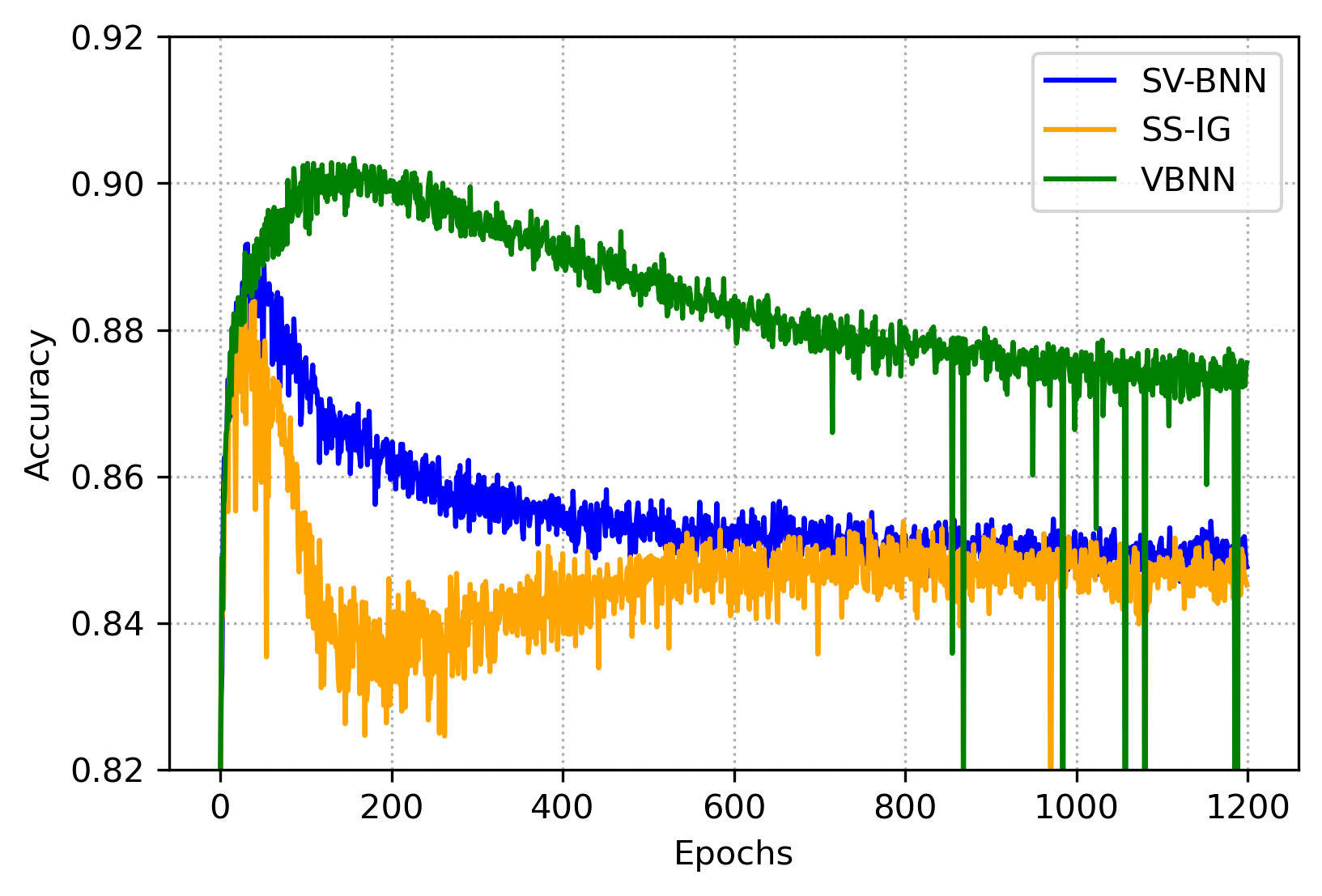}
    \caption{{\small Test accuracy}}    
    \label{fig:MLP-Fashion-MNIST-Test-Acc}
\end{subfigure}
\begin{subfigure}[b]{0.375\textwidth}
    \centering
    \includegraphics[width=\textwidth]{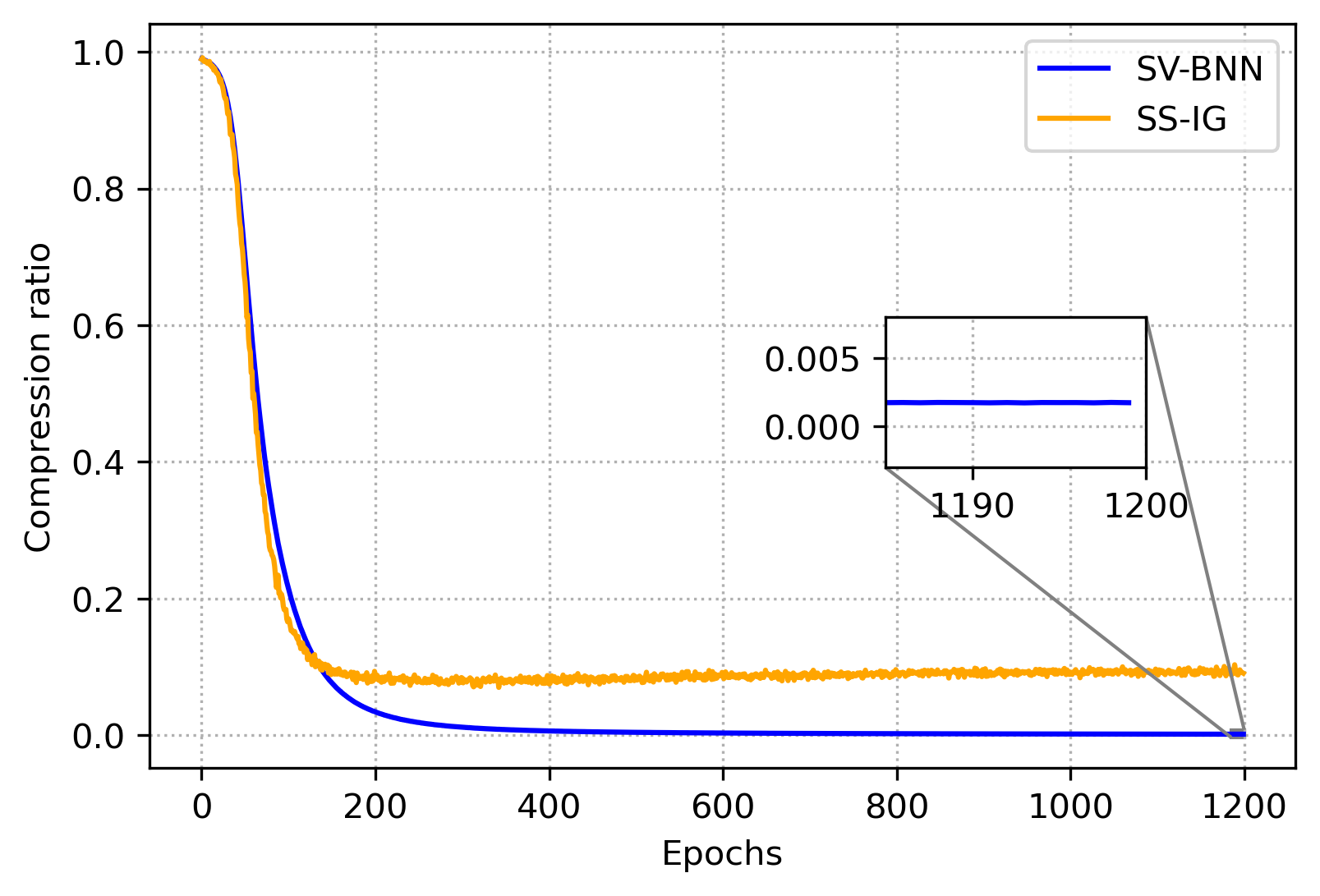}
    \caption{{\small Compression ratio}}    
    \label{fig:MLP-Fashion-MNIST-Compression-Ratio}
\end{subfigure}
\vskip\baselineskip
\begin{subfigure}[b]{0.375\textwidth}
    \centering
    \includegraphics[width=\textwidth]{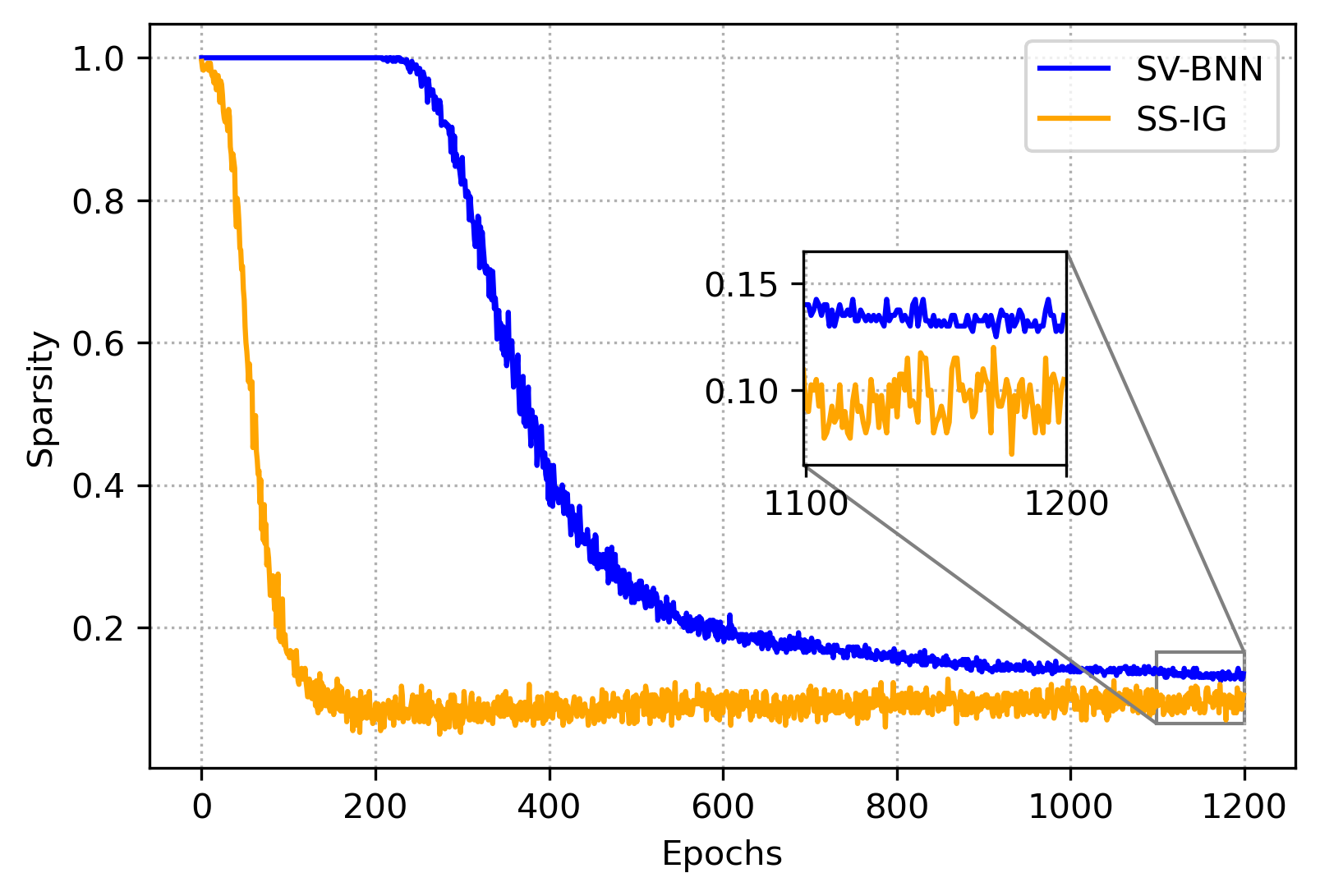}
    \caption{{\small Layer-1 node sparsity}}    
    \label{fig:MLP-Fashion-MNIST-layer1-sparsity}
\end{subfigure}
\begin{subfigure}[b]{0.375\textwidth}
    \centering
    \includegraphics[width=\textwidth]{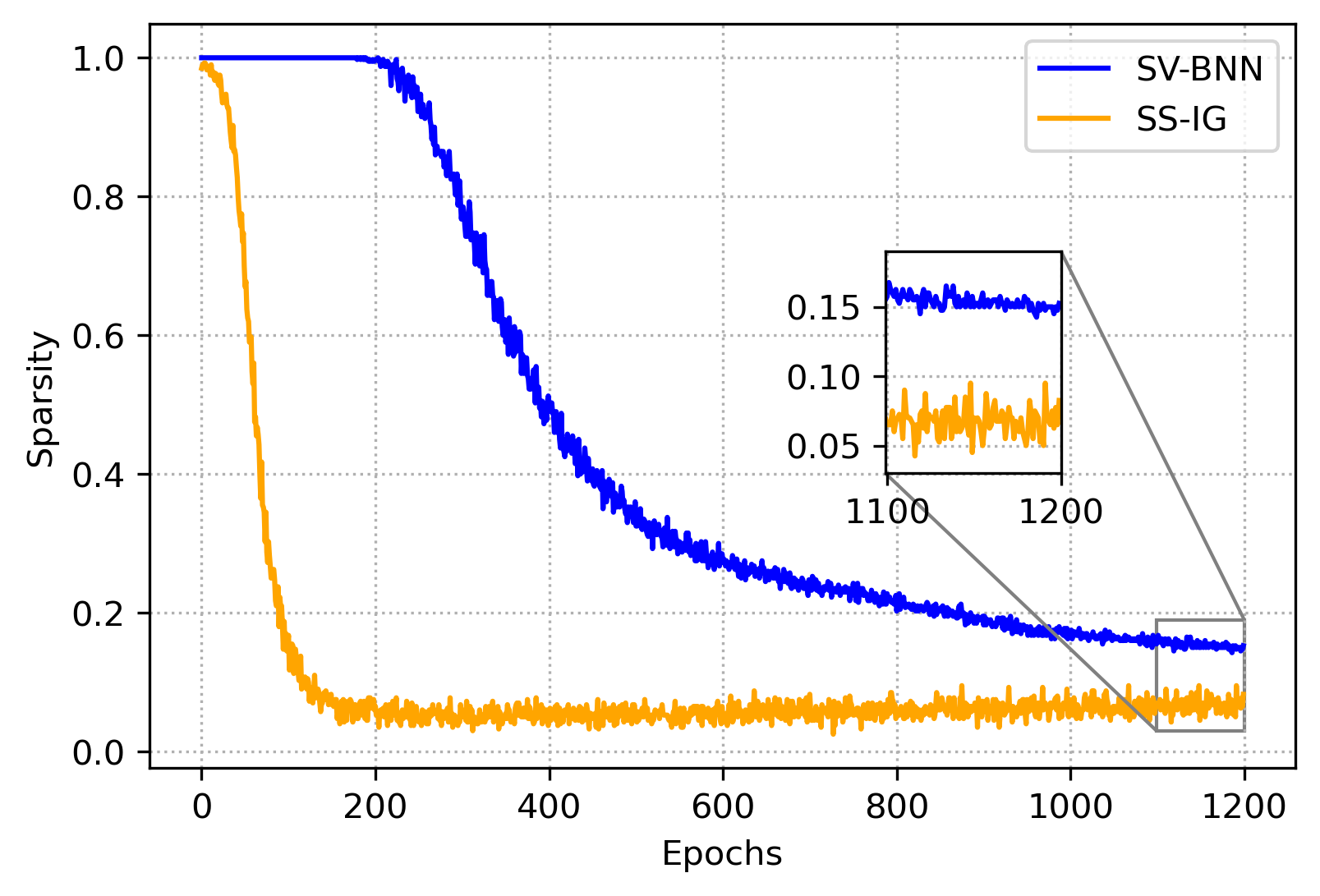}
    \caption{{\small Layer-2 node sparsity}}    
    \label{fig:MLP-Fashion-MNIST-layer2-sparsity}
\end{subfigure}
\caption{MLP architecture experiment results. First two rows (a)-(d) represent the MLP on MNIST experiment results. Bottom two rows (e)-(h) represent the MLP on Fashion-MNIST experiment results.}
\label{fig:MLP-Experiments}
\end{figure}

In MLP-MNIST experiment (Figure~\ref{fig:MLP-MNIST-Test-Acc} - \ref{fig:MLP-MNIST-layer2-sparsity}), we observe that VBNN and SS-IG models only require $\sim 400$ epochs to achieve stable predictive performance (Figure~\ref{fig:MLP-MNIST-Test-Acc}). In contrast, SV-BNN slightly degrades after 600 epochs and takes longer to achieve convergence in layer-wise node sparsities compared to our approach (Figure~\ref{fig:MLP-MNIST-layer1-sparsity} and \ref{fig:MLP-MNIST-layer2-sparsity}). Moreover, for SS-IG model, we observe that as we start to learn sparse network our model shows peak test accuracy when most of the nodes are present in the model and it starts to drop as we learn sparser network and ultimately the test accuracy stabilizes when the node sparsities converge. Furthermore, SV-BNN has better model compression ratio (Figure~\ref{fig:MLP-MNIST-Compression-Ratio}) in this experiment at the expense of lower predictive performance. Our method is prunes off $\sim 80\%$ of first hidden layer nodes and $\sim 90\%$ of second hidden layer nodes at the expense of $\sim 2\%$ accuracy loss due to sparsification compared to the dense VBNN.

In MLP-Fashion-MNIST experiment (Figure~\ref{fig:MLP-Fashion-MNIST-Test-Acc} - \ref{fig:MLP-Fashion-MNIST-layer2-sparsity}), we observe that VBNN model takes $\sim 200$ epochs and our model takes $\sim 600$ epochs for convergence. SV-BNN model takes longer to achieve convergence in layer-wise node sparsities (Figure~\ref{fig:MLP-Fashion-MNIST-layer1-sparsity} and \ref{fig:MLP-Fashion-MNIST-layer2-sparsity}). We also observe the complementary behavior of our model and SV-BNN in memory and computational efficiency where our model achieves better layer-wise node sparsities and  SV-BNN has better model compression ratio (Figure~\ref{fig:MLP-Fashion-MNIST-Compression-Ratio}) with both models having similar predictive performance (Figure~\ref{fig:MLP-Fashion-MNIST-Test-Acc}). Furthermore, our method prunes off $\sim 90\%$ of first hidden layer nodes and $\sim 92\%$ of second hidden layer nodes at the expense of $\sim 3\%$ accuracy loss due to sparsification compared 
to the densely connected VBNN.

\subsubsection*{Lenet-Caffe Experiments}
The results of more complex Lenet-Caffe network experiments on MNIST and Fashion-MNIST are presented in Figure~\ref{fig:Lenet-Caffe-Experiments}. We provide test data accuracy, model compression ratio, and FLOPs ratio in each experiment over 1200 epochs. Here, FLOPs ratio serve as a collective indicator of layer-wise node sparsities since FLOPs are directly related to how many neurons or channels are remaining in linear or convolution layers respectively.

\begin{figure}[h]
\centering
\begin{subfigure}[b]{0.325\textwidth}
    \centering
    \includegraphics[width=\textwidth]{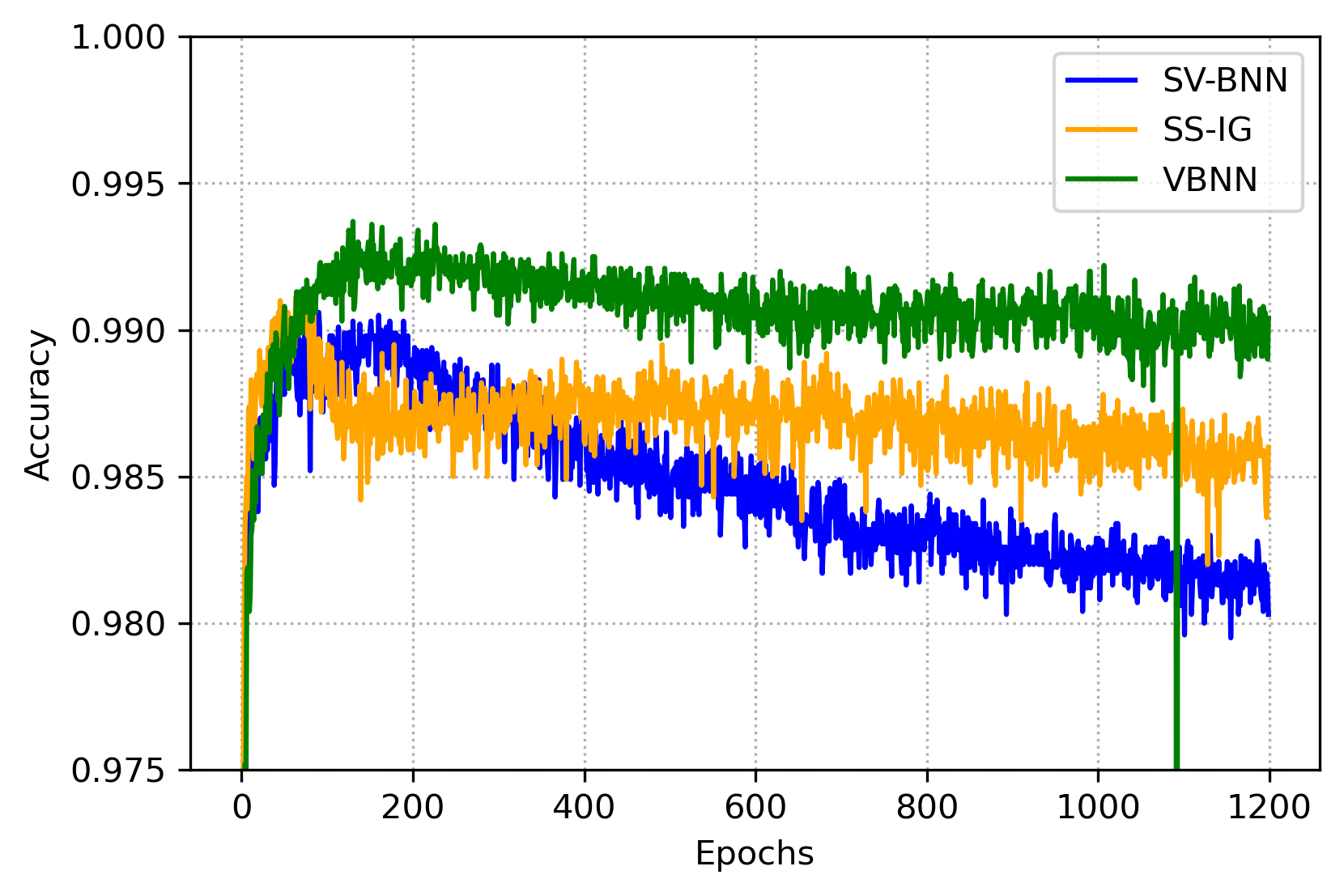}
    \caption{{\small Test accuracy}}    
    \label{fig:Lenet-Caffe-MNIST-Test-Acc}
\end{subfigure}
\begin{subfigure}[b]{0.325\textwidth}
    \centering
    \includegraphics[width=\textwidth]{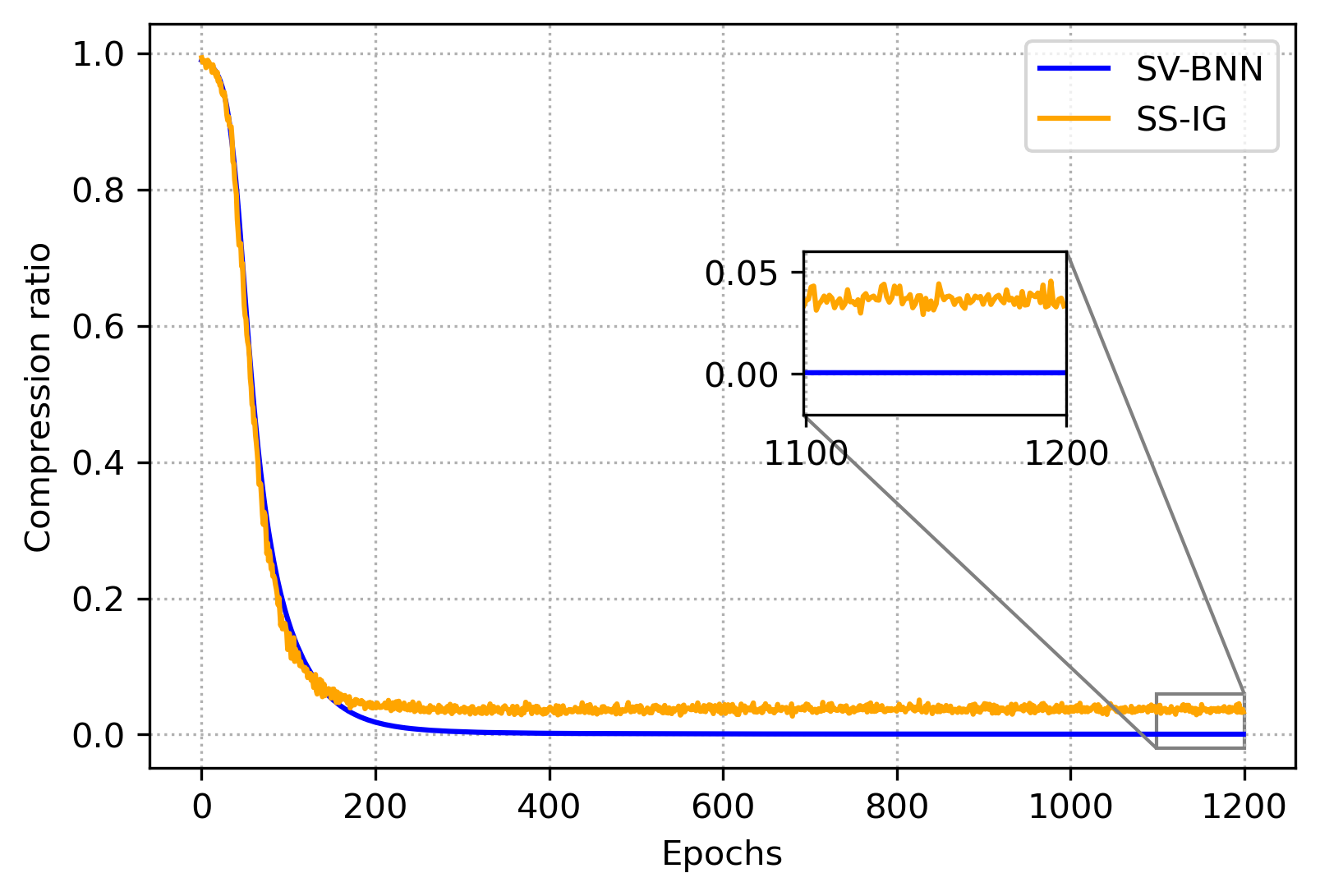}
    \caption{{\small Compression ratio}}    
    \label{fig:Lenet-Caffe-MNIST-Compression-Ratio}
\end{subfigure}
\begin{subfigure}[b]{0.325\textwidth}
    \centering
    \includegraphics[width=\textwidth]{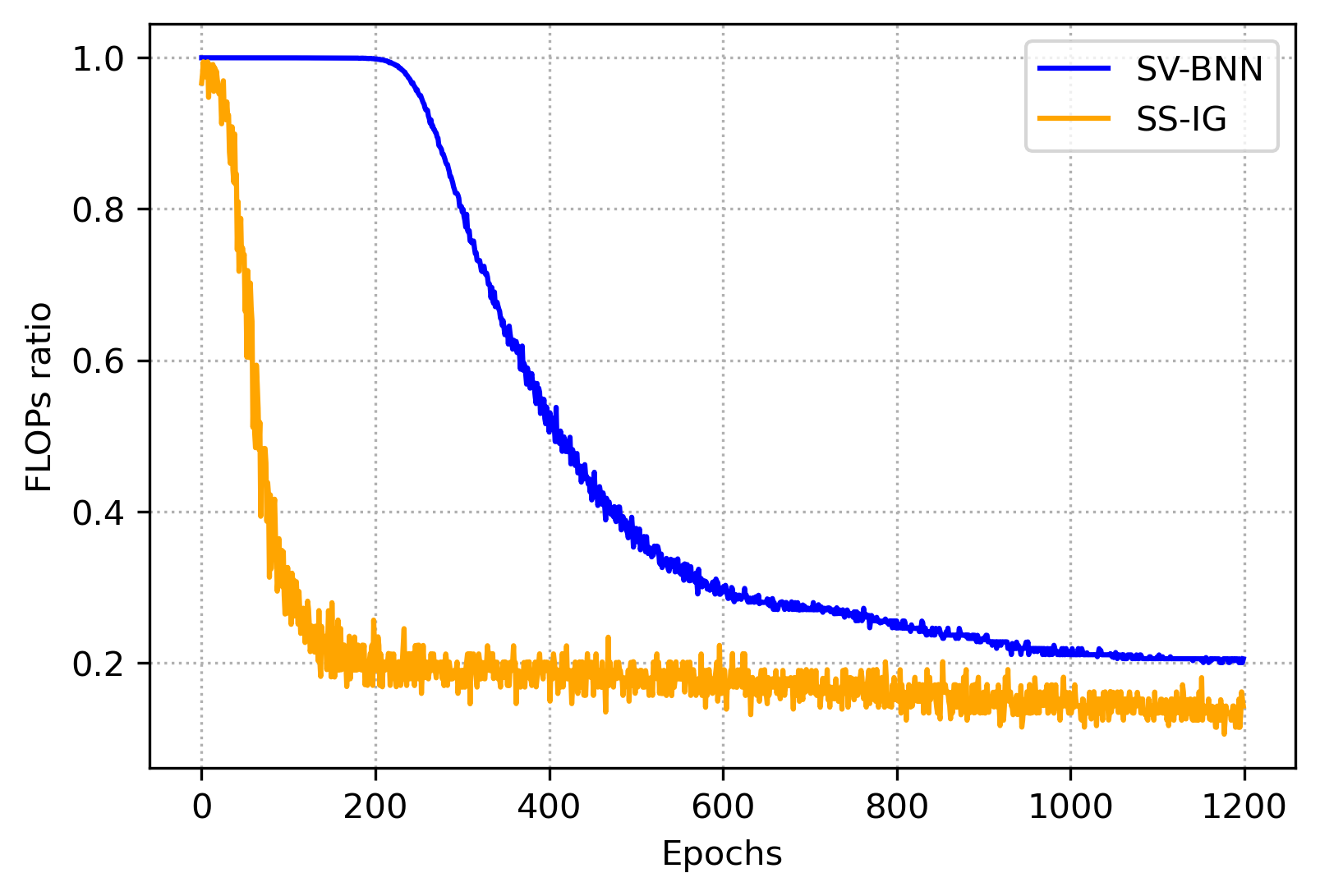}
    \caption{{\small FLOPs ratio}}    
    \label{fig:Lenet-Caffe-MNIST-Flops-Ratio}
\end{subfigure}
\vskip\baselineskip
\begin{subfigure}[b]{0.325\textwidth}
    \centering
    \includegraphics[width=\textwidth]{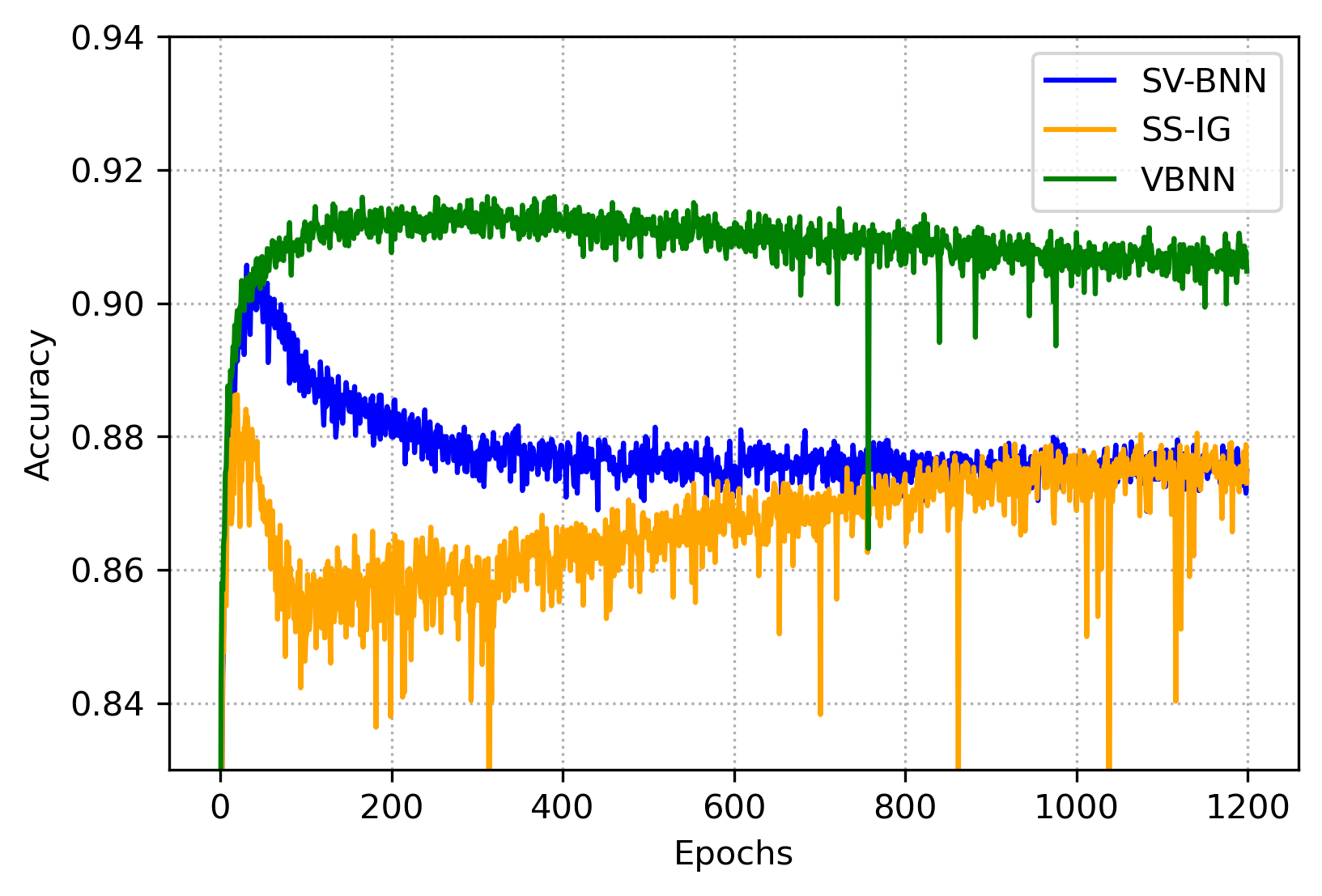}
    \caption{{\small Test accuracy}}    
    \label{fig:Lenet-Caffe-Fashion-MNIST-Test-Acc}
\end{subfigure}
\begin{subfigure}[b]{0.325\textwidth}
    \centering
    \includegraphics[width=\textwidth]{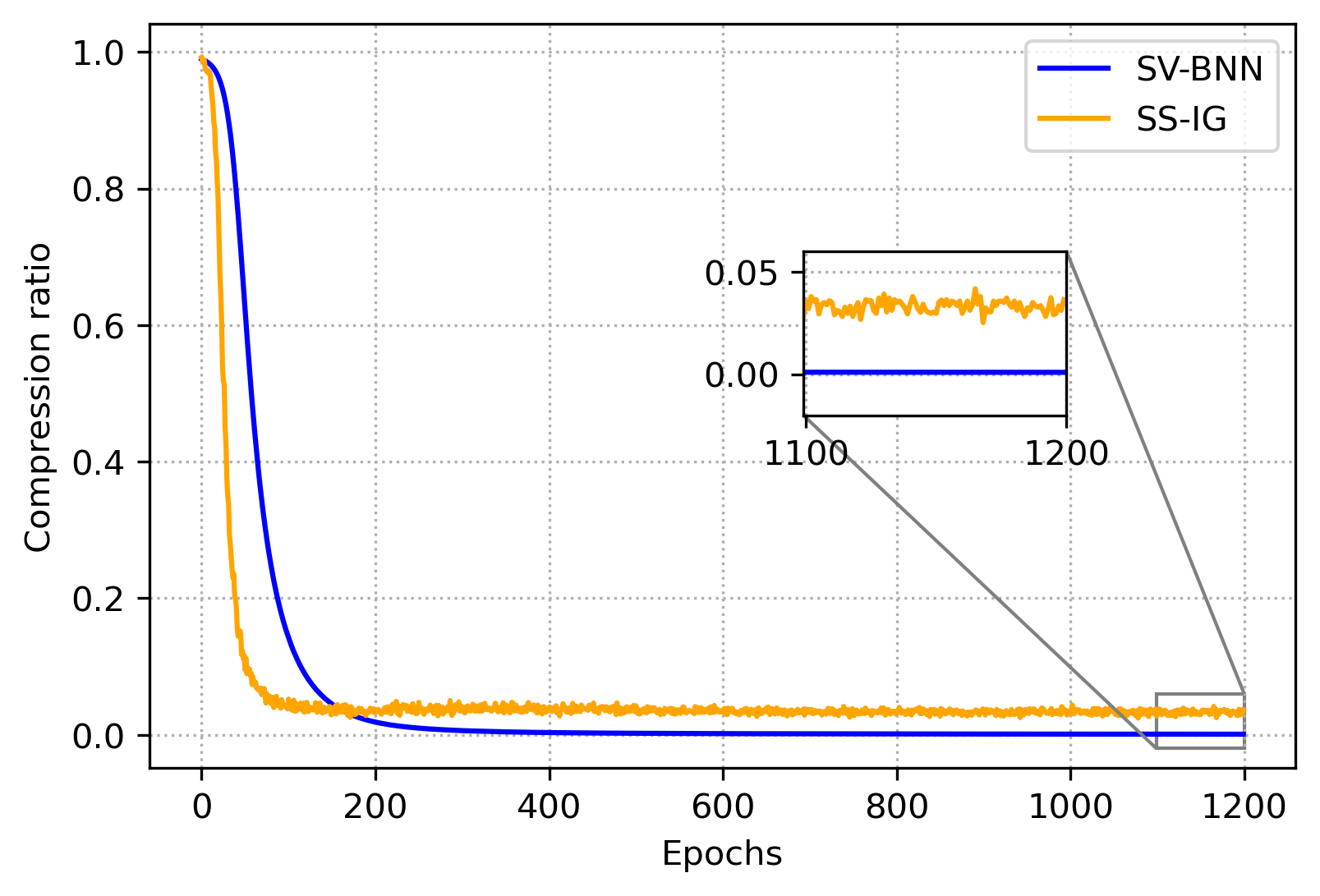}
    \caption{{\small Compression ratio}}    
    \label{fig:Lenet-Caffe-Fashion-MNIST-Compression-Ratio}
\end{subfigure}
\begin{subfigure}[b]{0.325\textwidth}
    \centering
    \includegraphics[width=\textwidth]{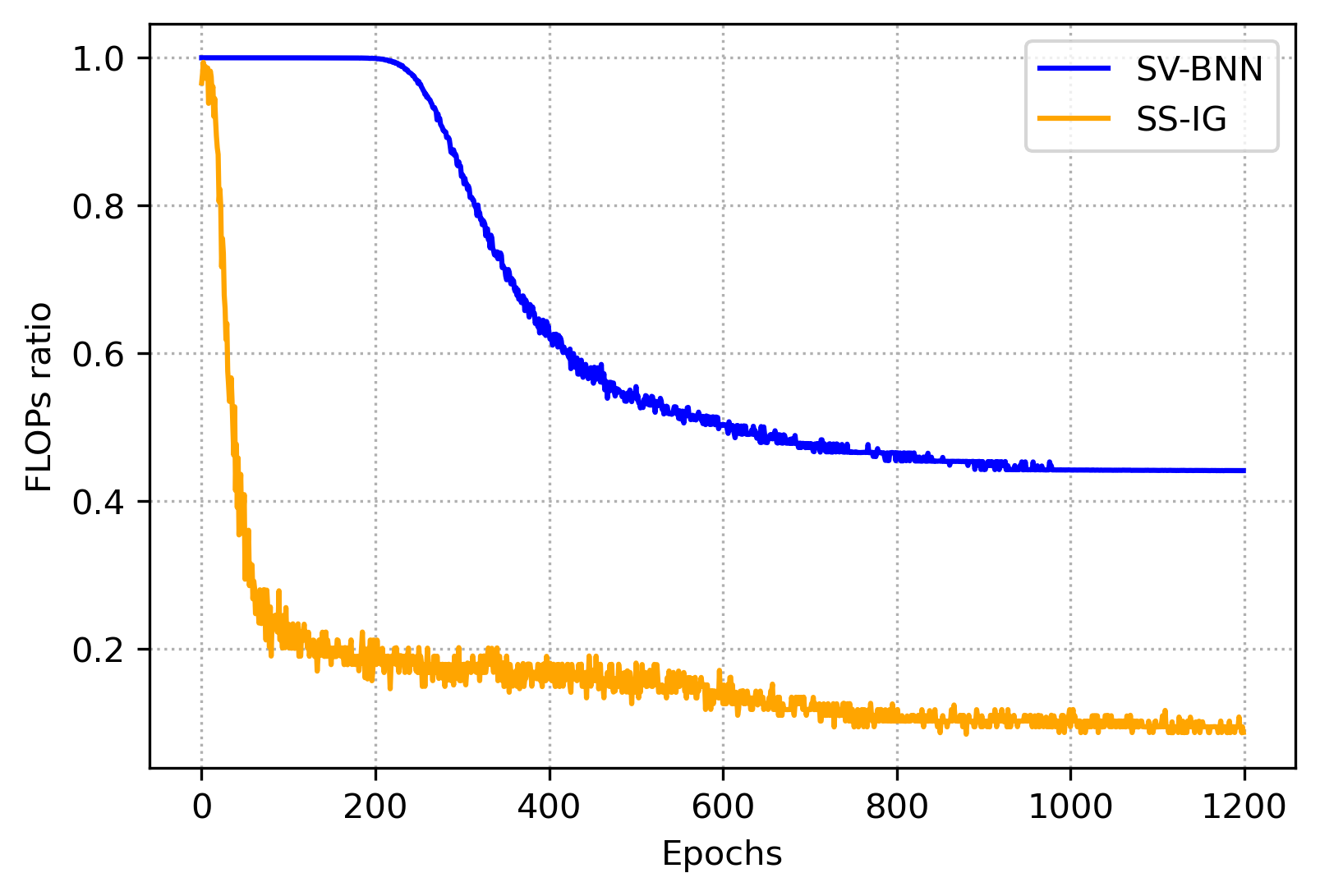}
    \caption{{\small FLOPs ratio}}    
    \label{fig:Lenet-Caffe-Fashion-MNIST-Flops-Ratio}
\end{subfigure}
\caption{Lenet-Caffe architecture experiment results. Top row (a)-(c) represent the Lenet-Caffe on MNIST experiment results. Bottom row (d)-(f) represent the Lenet-Caffe on Fashion-MNIST experiment results.}
\label{fig:Lenet-Caffe-Experiments}
\end{figure}

In Lenet-Caffe-MNIST experiment (Figure~\ref{fig:Lenet-Caffe-MNIST-Test-Acc} - \ref{fig:Lenet-Caffe-MNIST-Flops-Ratio}), we observe that our model has better predictive accuracy than SV-BNN (Figure~\ref{fig:Lenet-Caffe-MNIST-Test-Acc}). Moreover, we achieve $~10\%$ more reduction in Flops (Figure~\ref{fig:Lenet-Caffe-MNIST-Flops-Ratio})) compared to SV-BNN whereas SV-BNN achieves better model compression than our approach (Figure~\ref{fig:Lenet-Caffe-MNIST-Compression-Ratio}). Lastly, our method is able to reduce the FLOPs of the model during inference at test-time by $~90\%$ at the expense of $\sim 0.5\%$ accuracy loss due to sparsification compared to the densely connected VBNN.

In Lenet-Caffe-Fashion-MNIST experiment (Figure~\ref{fig:Lenet-Caffe-Fashion-MNIST-Test-Acc} - \ref{fig:Lenet-Caffe-Fashion-MNIST-Flops-Ratio}), we observe that both SS-IG and SV-BNN have similar test accuracies at convergence (Figure~\ref{fig:Lenet-Caffe-Fashion-MNIST-Test-Acc}). However, our model has $~40\%$ less FLOPs (Figure~\ref{fig:Lenet-Caffe-Fashion-MNIST-Flops-Ratio}) during inference stage compared to SV-BNN which again achieves better model compression (Figure~\ref{fig:Lenet-Caffe-Fashion-MNIST-Compression-Ratio}). This highlights the complementary nature of our method of node selection that leads to a structurally sparse model with significantly lower (almost 5 times) FLOPs compared to weight pruning approach, SV-BNN, which induces unstructured sparsity in the pruned network leading to significant model compression with low storage cost. Lastly, our method leads to a sparse model with only $~8\%$ of the FLOPs as compared to VBNN at the expense of $\sim 3\%$ accuracy loss underscoring the trade-off between predictive accuracy and sparsity. 

\section{Conclusion and Discussion} 

\label{Sec:Discussion}
Deep learning has been harnessed by big industrial corporations in recent years to improve their products. However, as deep learning models are pushed into smaller and smaller embedded devices, such as smart cameras recognizing visitors at your front door, designing resource efficient neural networks for real-time, on-device inference is of practical importance. Our work addresses this computational bottleneck by compressing neural networks by inducing structured sparsity during training. The estimation of posterior allows us to quantify uncertainties around the parameter estimates which can be vital in medical diagnostics. 

In this paper, we have proposed sparse deep Bayesian neural networks using spike-and-slab priors for optimal node recovery. Our method incorporates layer-wise prior inclusion probabilities and recovers underlying structurally sparse model effectively. Our theoretical developments highlight the conditions required for the posterior consistency of the variational posterior to hold. With layer-wise characterisation of prior inclusion probabilities we show that the proposed sparse BNN approximations can achieve predictive performance comparable to dense networks. Our results relax the constraints of equal number of nodes and uniform bounds on weights thereby achieving optimal node recovery on more generic neural network structure. The closeness of a true function to the topology induced by layer-wise node distribution depends on the degree of smoothness of the true underlying function. In this work, this has not been studied in depth and forms a promising direction for future work.

Note, in contrast to previous works, our work assumes a spike-and-slab prior on the entire vector of incoming weights and bias onto a node. \SJ{We underscore the fact that node selection has complementary behavior with edge selection approaches as established by our empirical experiments. Node selection offers significant computational speedup whereas edge selection achieves significant model compression at test-time. The demonstration of the efficacy of our node selection approach opens the avenue for exploration of sophisticated group sparsity priors for node selection. Our detailed experiments show the subnetwork selection ability of our method which underscores the notion that deep neural networks can be heavily pruned without losing predictive performance. The experiment with convolution neural network (Lenet-Caffe) highlights the generalizability of our approach from mere multi layer perceptron to complex deep learning models.} Although our method performs model reduction while maintaining predictive power, some further improvements may be obtained by choosing the number of layers in a data-driven fashion and can be a part of future work. 


\appendix

\newpage

\section{Proofs of theoretical results} 

\label{AppendixA}
\subsection{Definitions}

\begin{definition}[Sieve] \label{def:sieve}
Consider a sequence of function classes $\mathcal{F}_1\subseteq \mathcal{F}_2 \subseteq \dots \subseteq \mathcal{F}_n \subseteq \mathcal{F}_{n+1} \subseteq \dots \subseteq \mathcal{F}$, where $\forall f \in \mathcal{F}, \hspace{1mm} \exists \hspace{1mm} f_n \in \mathcal{F}_n$ s.t. $d(f,f_n)\to 0$ as $n \to \infty$ where $d(.,.)$ is some pseudo-metric on $\mathcal{F}$. More precisely, $\cup_{n=1}^\infty \mathcal{F}_n$ is dense in $\mathcal{F}$. $\mathcal{F}_n$ is called a sieve space of $\mathcal{F}$ with respect to the pseudo-metric $d(.,.)$, and the sequence $\{f_n\}$ is called a sieve \citep{Grenander-1981}. 
\end{definition}

\begin{definition}[Covering number] \label{def:covering-no}
Let $(V,||.||)$ be a normed space,  and $\mathcal{F} \subset V$. $\{V_1,\cdots, V_N \}$ is an $\varepsilon-$covering of $\mathcal{F}$ if $\mathcal{F} \subset \cup_{i=1}^N B(V_i,\varepsilon)$, or equivalently, $\forall$ $\varrho \in \mathcal{F}$, $\exists$ $i$ such that $||\varrho-V_i||<\varepsilon$. The covering number of $\mathcal{F}$ denoted by $N(\varepsilon,\mathcal{F},||.||)=\min\{n: \exists\: \varepsilon-\text{ covering over }\mathcal{F}\text{ of size } n \}$ \citep{Pollard-1991}.
\end{definition}

\subsection{General Lemmas}

\begin{lemma}
    \label{lem:mod-kl}
    Let $g_1$ and $g_2$ be any two density functions.
 Then
$$E_{g_1}(\left|\log (g_1/g_2)\right|) \leq d_{\rm KL}(g_1,g_2)+2/e$$  
\end{lemma}

\begin{proof}
Refer to Lemma 4 in \cite{Lee-2000}.
\end{proof}

\begin{lemma} \label{lem:mixture-density}
 For any $K>0$, let $\bm{a},\bm{a^0}\in[0,1]^K$ such that $\sum_{k=1}^K a_k = \sum_{k=1}^K a^0_k = 1$, then the KL divergence between mixture densities $\sum_{k=1}^K a_k g_k$ and $\sum_{k=1}^K a^0_k g^0_k$ is bounded as 
   $$d_{\rm KL}\left(\sum_{k=1}^K a^0_k g^0_k , \sum_{k=1}^K a_k g_k\right) \leq  d_{\rm KL}(\bm{a^0}, \bm{a}) + \sum_{k=1}^K a^0_k d_{\rm KL}(g_k^0, g_k)$$
\end{lemma} 
\begin{proof}
Refer to Lemma 6.1 in \cite{Cherief-Abdellatif-Alquier-2018}.
\end{proof}  

\begin{lemma}
\label{lem:kl-upp}
$$d_{\rm KL}(\widetilde{\pi}^*,\widetilde{\pi}(|\mathcal{D}))\leq d_{\rm KL}(\pi^*,\pi(|\mathcal{D}))$$
\end{lemma}
\begin{proof}
 Using Lemma \ref{lem:mixture-density} with $\bolda^{0}=\pi^*(\boldz)$, $\bolda=\pi(\boldz|\mathcal{D})$, $g^0=\pi^*(\btheta|\boldz)$ and $g=\pi(\btheta|\boldz,\mathcal{D})$, we get
\begin{align*}
    d_{\rm KL}(\widetilde{\pi}^*,\widetilde{\pi}(|\mathcal{D}))&=d_{\rm KL}(\sum_{\boldz} \pi^*(\btheta|\boldz)\pi^*(\boldz),\sum_{\boldz} \pi(\btheta|\boldz,\mathcal{D})\pi(\boldz|\mathcal{D}))\\
    &\leq d_{\rm KL}(\pi^*(\boldz),\pi(\boldz|\mathcal{D}))+\sum_{\boldz}d_{\rm KL}(\pi^*(\btheta|\boldz),\pi(\btheta|\boldz,\mathcal{D}))\pi^*(\boldz)\\
    &=d_{\rm KL}(\pi^*(\btheta,\boldz),\pi(\btheta,\boldz|\mathcal{D}))=d_{\rm KL}(\pi^*,\pi(|\mathcal{D}))
\end{align*}
\end{proof}

\begin{lemma}
\label{lem:covering-1}
For any 1-Lipschitz continuous activation function $\psi$ such that $\psi(x)\leq x \hspace{1mm}\forall x\geq0$, 
$$ N(\delta,\calF(L,\bm{k},\bm{s},\bm{B}),||.||_\infty) \leq \sum_{s^*_L \leq s_L} \cdots \sum_{s^*_0 \leq s_0} \left[ \prod_{l=0}^L \left(\frac{B_l}{\delta B_l/(2(L+1)( \prod_{j=0}^L B_j))} k_{l+1}\right)^{s_l } \right]$$
where $N$ denotes the covering number.
\end{lemma}
\begin{proof}

Given a neural network
$$ \eta(\boldx)=\boldv_L+\boldW_L \psi(\boldv_{L-1}+\boldW_{L-1} \psi (\boldv_{L-2}+\boldW_{L-2}\psi( \cdots \psi(\boldv_1+\boldW_1\psi(\boldv_0+\boldW_0 \boldx))) $$

for $l \in \{1,\cdots,L\}$, we define $A_l^+\eta:[0,1]^p \to \R^{k_l}$,
$$ A_l^+\eta(\boldx)= \psi(\boldv_{l-1}+\boldW_{l-1} \psi (\boldv_{l-2}+\boldW_{l-2}\psi( \cdots \psi(\boldv_1+\boldW_1\psi(\boldv_0+\boldW_0 \boldx))) $$
and $A_l^-\eta:\R^{k_{l-1}} \to \R^{k_{L+1}}$,
$$ A_l^-\eta(\boldy)= \boldv_L+\boldW_L \psi(\boldv_{L-1}+\boldW_{L-1} \psi (\cdots \psi(\boldv_l+\boldW_l\psi(\boldv_{l-1}+\boldW_{l-1} \boldy))) $$
The  above framework is also used in the proof of lemma 5 in \cite{Schmidt-Hieber-2017}. Next, set $A_0^+\eta(\boldx)=A_{L+2}^-\eta(\boldx)=\boldx$ and further note that for $\eta \in \calF(L,\bm{k})$, $|A_l^+\eta(\boldx)|_\infty \leq \prod_{j=0}^{l-1}B_j$ where $\bm{k}= (p,k_1,\cdots,k_L,k_{L+1})$ and $k_{L+1}=1$. Next, we derive upper bound on Lipschitz constant of $A_l^-\eta$.
\begin{equation}
\label{lipschitz_equality}
|\boldW_L A_L^+\eta(\boldx_1)- \boldW_L A_L^+\eta(\boldx_2)|_\infty = |A_l^-\eta(A_{l-1}^+\eta(\boldx_1))- A_l^-\eta(A_{l-1}^+\eta(\boldx_2))|_\infty 
\end{equation} 

l.h.s. is bounded above by $\prod_{j=0}^{L}B_j$ and r.h.s consists of composition of Lipschitz functions $A_l^-\eta$ and $A_{l-1}^+\eta$ with $C_1$ and $C_2$ being corresponding Lipschitz constants. So we can bound r.h.s. by,

$$ |A_l^-\eta(A_{l-1}^+\eta(\boldx_1))- A_l^-\eta(A_{l-1}^+\eta(\boldx_2))|_\infty \leq C_1 C_2 ||\boldx_1-\boldx_2||_\infty \quad \forall \boldx_1,\boldx_2 \in \R^p $$

If we choose $\boldx_1=\boldx \in [0,1]^p$ and $\boldx_2=\bm{0}$ then,
$$ |A_l^-\eta(A_{l-1}^+\eta(\boldx))-A_l^-\eta(A_{l-1}^+\eta(\bm{0}))|_\infty \leq C_1 C_2  \quad \forall \boldx \in [0,1]^p $$

Since $C_2$ is Lipschitz constant for $A_{l-1}^+\eta$ and we know that $|A_{l-1}^+\eta|_\infty \leq \prod_{j=0}^{l-2}B_j$. So we get $C_2\leq 2\prod_{j=0}^{l-2}B_j$. We use this in above expression,
\begin{equation}
\label{lipschitz_rhs_bound}
|A_l^-\eta(A_{l-1}^+\eta(\boldx))-A_l^-\eta(A_{l-1}^+\eta(\bm{0}))|_\infty \leq 2 C_1 \prod_{j=0}^{l-2}B_j  \quad \forall \boldx \in [0,1]^p    
\end{equation} 

Next we know that l.h.s. of \eqref{lipschitz_rhs_bound} can be bounded above by $2\prod_{j=0}^{L}B_j$ because of \eqref{lipschitz_equality}. So we get bound on Lipschitz constant of $A_l^-\eta$,
\begin{equation*}
2C_1 \prod_{j=0}^{l-2}B_j  \leq  2\prod_{j=0}^{L}B_j \implies C_1 \leq \prod_{j=l-1}^{L}B_j
\end{equation*} 

Let $\eta,\eta^* \in \calF(L,\bm{k},\bm{s},\bm{B})$ be two neural networks with $\overline{\boldW}_l=(\boldv_l,\boldW_l)$ and $\overline{\boldW}^*_l=(\boldv_l^*,\boldW_l^*)$ respectively. Here, we define $\overline{\bdelta}_l$ using the $L_1$ norms of the rows of $\overline{\boldD}_l = \overline{\boldW}_l - \overline{\boldW}^*_l$ as follows
$$\overline{\boldD}_l=(
\overline{\bm{d}}_{l1}^\top, \cdots,
\overline{\bm{d}}_{lk_{l+1}}^\top )^{\top} \hspace{5mm} \overline{\bdelta}_l=(||\overline{\bm{d}}_{l1}||_1, \cdots,
||\overline{\bm{d}}_{lk_{l+1}}||_1 )$$
We choose $\eta,\eta^*$ such that $||\overline{\bdelta}_l||_\infty \leq \zeta B_l$. This also means that all parameters in each layer of these two networks are at most $\zeta B_l$ distance away from each other.  Then, we can bound the absolute difference between these two neural networks by,
\begin{align}
    \nonumber &|\eta(\boldx)-\eta^*(\boldx)| \\
    \nonumber & \leq \sum_{l=1}^{L+1} |A_{l+1}^-\eta(\psi(\boldv_{l-1}+\boldW_{l-1}A_{l-1}^+\eta^*(\boldx))) - A_{l+1}^-\eta(\psi(\boldv_{l-1}^*+\boldW_{l-1}^*A_{l-1}^+\eta^*(\boldx)))| \\
    \nonumber & \leq \sum_{l=1}^{L+1} \left( \prod_{j=l}^L B_j\right)  ||\psi(\boldv_{l-1}+\boldW_{l-1}A_{l-1}^+\eta^*(\boldx)) - \psi(\boldv_{l-1}^*+\boldW_{l-1}^*A_{l-1}^+\eta^*(\boldx))||_\infty \\
    \nonumber & \leq \sum_{l=1}^{L+1} \left( \prod_{j=l}^L B_j \right)  ||\boldv_{l-1} - \boldv_{l-1}^* + (\boldW_{l-1} - \boldW_{l-1}^*) A_{l-1}^+\eta^*(\boldx))||_\infty \\
    \nonumber & \leq \sum_{l=1}^{L+1} \left( \prod_{j=l}^L B_j \right)   ||\overline{\bdelta}_{l-1}||_\infty ||A_{l-1}^+\eta^*(\boldx))||_\infty \\
     & \leq \sum_{l=1}^{L+1} \left( \prod_{j=l}^L B_j \right)  \zeta B_{l-1} \prod_{j=0}^{l-2} B_j  = \zeta (L+1) \left( \prod_{j=0}^L B_j \right) \label{l1_bound_NN}
\end{align}

Recall that we have at most $k_l$ number of nodes in each layer and there are $\binom{k_{l+1}}{s_l} \leq k_{l+1}^{s_l}$ combinations of nodes to choose $s_l$ active nodes in the given layer. Since supremum norm of $L_1$ norms of the rows of $\overline{\boldW_l}$ is bounded above by $B_l$ in our family of neural networks $\calF(L,\bm{k},\bm{s},\bm{B})$ so we can discretize these $L_1$ norms  
with grid size $\delta B_l/(2(L+1)( \prod_{j=0}^L B_j))$ and obtain upper bound on covering number as follows
\begin{align}
    \nonumber N(\delta,\calF(L,\bm{k},\bm{s},\bm{B}),||.||_\infty) & \leq \sum_{s^*_L \leq s_L} \cdots \sum_{s^*_0 \leq s_0} \left[ \prod_{l=0}^L \left(\frac{B_l}{\delta B_l/(2(L+1)( \prod_{j=0}^L B_j))} k_{l+1}\right)^{s_l } \right] \\
     & \leq \prod_{l=0}^L \left(2\delta^{-1}(L+1) \left( \prod_{j=0}^L B_j \right) k_{l+1} \right)^{(s_l+1)} \label{covering_num_bound}
\end{align}

\end{proof}

\begin{lemma}
\label{lem:psi-l-bound}
Let $\btheta^* = \arg \min_{\btheta \in \calF(L,\boldk,\bolds,\boldB)} \left\|\eta_{\btheta}-\eta_{0}\right\|_{\infty}^{2}$ and $\widetilde{W}_l=\sup_{i}||\overline{\boldw}_{li}-\overline{\boldw}_{li}^*||_1$,  then for any density $q=\prod_{j=0}^Lq(\theta_j)$,
\begin{align}
\label{e:q-bound}
\nonumber &\int ||\eta_{\btheta} - \eta_{\btheta^*}||^2_2 q(\btheta)d\btheta \leq \sum_{j=0}^L c_{j-1}^2 \int \widetilde{W}_{j}^2 q_j(\theta_j) d\theta_j \prod_{m=j+1}^L \int (\widetilde{W}_m+B_m)^2 q(\btheta)d\btheta \\
 \nonumber   & \enskip + 2\sum_{j=0}^{L} \sum_{j'=0}^{j-1} c_{j-1}c_{j'-1} \int \widetilde{W}_j (\widetilde{W}_j+B_j) q_j(\theta_j)d\theta_j \prod_{m=j+1}^L \int (\widetilde{W}_m+B_m)^2 q(\btheta)d\btheta  \\
    & \enskip \enskip \times \int \widetilde{W}_{j'} q_{j'}(\theta_{j'}) d\theta_{j'} \prod_{m=j'+1}^{j-1} \int (\widetilde{W}_m+B_m)q(\btheta)d\btheta 
    \end{align}
where $c_{j-1}\leq\prod_{m=0}^{j-1}B_m$.
\end{lemma}

\begin{proof} 

Let $\eta_{\btheta}^l$ be the partial networks defined as
\begin{align*} 
    \begin{cases}                 \eta_{\btheta}^0(\boldx):=\psi(\boldW_0\boldx+\boldv_0),\\
    \eta_{\btheta}^l(\boldx):= \psi(\boldW_l\eta_{\btheta}^{l-1}(\boldx)+\boldv_l), \\
    \eta_{\btheta}^L(\boldx):= \boldW_L\eta_{\btheta}^{L-1}(\boldx)+\boldv_L.
    \end{cases} 
\end{align*}
Similar to the proof of theorem 2 in \cite{Cherief-Abdellatif-2020}, define 
$$\varphi_l(\btheta)=\sup_{x \in [0,1]^p} \sup_{1 \leq i \leq k_{l+1}}|\eta_{\btheta}^l(\boldx)_i -\eta_{\btheta^*}^l (\boldx)_i|.$$ 
We next show by induction $$\varphi_l(\btheta) \leq \sum_{j=0}^l\widetilde{W}_j c_{j-1}R_{j+1}^l$$
where we define $c_l=\max(\sup_{x \in [0,1]^p} \sup_{1\leq i \leq k_{l+1}}|\eta_{\btheta^*}^l(\boldx)_i|,1)$, $c_0 = 1$, $R_{j+1}^l=\prod_{m=j+1}^l (\widetilde{W}_m+B_m)$.

Claim: $c_l\leq B_l c_{l-1}$. Note
\begin{align*}
 c_l &\leq   \sup_{x \in [0,1]^p} \sup_{1\leq i \leq k_{l+1}}(|{\boldw_{li}^*}^\top \eta^{l-1}_{\btheta^*}(\boldx)|+|v_{li}|)\\
    &\leq  \sup_{x \in [0,1]^p} \sup_{1\leq i \leq k_{l+1}} (\sum_{j=1}^{k_l} |w^*_{lij}||\eta_{\btheta^*}^{l-1}(\boldx)_j|+|v_{li}|)\\
    &\leq  \sup_{1\leq i \leq k_{l+1}}(c_{l-1}\sum_{j=1}^{k_l} |w^*_{lij}|+c_{l-1}|v_{li}|)\\
    &\leq c_{l-1} \sup_{1\leq i \leq k_{l+1}} ||\overline{\boldw}_{li}^*||_1=B_l c_{l-1}
\end{align*}
where the above result holds since  $\sup_i ||\overline{\boldw}_{li}^*||_1 \leq B_l$.
Next,
\begin{align*}
\varphi_l(\btheta)&\leq \sup_{x \in [0,1]^p} \sup_{1\leq i \leq k_{l+1}}(\sum_{j=1}^{k_l} |w_{lij} \eta_{\btheta}^{l-1}(\boldx)_j-w^*_{lij} \eta_{\btheta^*}^{l-1}(\boldx)_j|+|v_{li}-v_{li}^*|)\\
&\leq \sup_{x \in [0,1]^p} \sup_{1\leq i \leq k_{l+1}}(\sum_{j=1}^{k_l} |w_{lij} \eta_{\btheta}^{l-1}(\boldx)_j-w^*_{lij} \eta_{\btheta}^{l-1}(\boldx)_j|\\
&\hspace{30mm}+|w^*_{lij} \eta_{\btheta}^{l-1}(\boldx)_j-w^*_{lij} \eta_{\btheta^*}^{l-1}(\boldx)_j|+|v_{li}-v_{li}^*|)\\
&\leq \sup_{x \in [0,1]^p} \sup_{1\leq i \leq k_{l+1}}(\sum_{j=1}^{k_l}|w_{lij} -w^*_{lij}|| \eta_{\btheta}^{l-1}(\boldx)_j|\\
&\hspace{30mm}+\sum_{j=1}^{k_l}|w_{lij}^*| | \eta_{\btheta}^{l-1}(\boldx)_j-\eta_{\btheta^*}^{l-1}(\boldx)_j|+|v_{li}-v_{li}^*|)\\
&\leq \sup_{x \in [0,1]^p} \sup_{1\leq i \leq k_{l+1}}(\sum_{j=1}^{k_l}|w_{lij} -w^*_{lij}|| \eta_{\btheta}^{l-1}(\boldx)_j-\eta_{\btheta^*}^{l-1}(\boldx)_j|\\
&\hspace{30mm}+\sum_{j=1}^{k_l}|w_{lij} -w^*_{lij}||\eta_{\btheta^*}^{l-1}(\boldx)_j|+|v_{li}-v_{li}^*|) +\varphi_{l-1}(\btheta)B_l\\
&\leq \widetilde{W}_l(\varphi_{l-1}(\btheta)+c_{l-1})+\varphi_{l-1}(\btheta)B_l=\varphi_{l-1}(\btheta)(\widetilde{W}_l+B_l)+c_{l-1}\widetilde{W}_l
\end{align*}
Now applying recursion we get
\begin{align*}
    \varphi_l(\btheta)&\leq (\varphi_{l-2}(\btheta)(\widetilde{W}_{l-1}+B_{l-1})+c_{l-2}\widetilde{W}_{l-1})(\widetilde{W}_l+B_l)+c_{l-1}\widetilde{W}_l\\
   &=\varphi_{l-2}(\btheta)(\widetilde{W}_l+B_l)(\widetilde{W}_{l-1}+B_{l-1})+c_{l-2}\widetilde{W}_{l-1}(\widetilde{W}_l+B_l)+c_{l-1}\widetilde{W}_l
\end{align*}
Repeating this we get
\begin{align*}
    \varphi_l(\btheta)& \leq \varphi_0(\btheta) \prod_{j=1}^l (\widetilde{W}_{j}+B_j)+\sum_{j=1}^{l}c_{j-1} \widetilde{W}_j \prod_{u=j+1}^l (\widetilde{W}_j+B_j)\\
   & =\widetilde{W}_0\prod_{j=1}^l (\widetilde{W}_{j}+B_j)+\sum_{j=1}^{l}B_1 \cdots B_{j-1} \widetilde{W}_j \prod_{u=j+1}^l  (\widetilde{W}_j+B_j)\\
   & = \sum_{j=0}^{l}B_1 \cdots B_{j-1} \widetilde{W}_j \prod_{u=j+1}^l  (\widetilde{W}_j+B_j)  = \sum_{j=0}^l\widetilde{W}_j c_{j-1}R_{j+1}^l
\end{align*}

\begin{align*}
    & \int ||\eta_{\btheta} - \eta_{\btheta^*}||^2_2 q(\btheta)d\btheta \leq \int ||\eta_{\btheta} - \eta_{\btheta^*}||^2_\infty q(\btheta)d\btheta = \int \varphi_L^2(\btheta) q(\btheta)d\btheta \\
    &=  \int  (\sum_{j=0}^L \widetilde{W}_j c_{j-1}R_{j+1}^L)^2 q(\btheta)d\btheta\\
    &=\sum_{j=0}^L c_{j-1}^2 \int \widetilde{W}_{j}^2 (R_{j+1}^L)^2 q(\btheta)d\btheta+2\sum_{j=0}^{L} \sum_{j'=0}^{j-1} c_{j-1}c_{j'-1} \int \widetilde{W}_j \widetilde{W}_{j'} R_{j+1}^L R_{j'+1}^L q(\btheta)d\btheta \\
    & =\sum_{j=0}^L c_{j-1}^2 \int \widetilde{W}_{j}^2 \left(\prod_{m=j+1}^L (\widetilde{W}_m+B_m) \right)^2 q(\btheta)d\btheta \\
    & \enskip + 2\sum_{j=0}^{L} \sum_{j'=0}^{j-1} c_{j-1}c_{j'-1} \int \widetilde{W}_j \widetilde{W}_{j'} \prod_{m=j+1}^L (\widetilde{W}_m+B_m) \prod_{m=j'+1}^L (\widetilde{W}_m+B_m)q(\btheta)d\btheta
\end{align*}
The proof follows by noting $q(\btheta)=\prod_{j=0}^L q(\theta_j)$.
\end{proof}

\begin{lemma}
\label{lem:covering}
    Suppose Lemma \ref{lem:test} and Lemma \ref{lem:prior} in the main paper hold, with dominating probability
    $$\log \int_{\mathcal{H}_{\epsilon_n}^c}\frac{P_{\btheta}^n}{P_0^n}\pi(\btheta) d\btheta \leq - \frac{Cn{\epsilon_n^2}}{\sum u_l}$$
\end{lemma}

\begin{proof}

Let $\mathcal{F}_n=\mathcal{F}(L,\boldk,\bolds^\circ,\boldB^\circ)$, $s_{l}^\circ+1 = n{\epsilon_n^2} /\sum_{j=0}^L u_{j} $, $\log B_l^\circ=n{\epsilon_n^2}/((L+1)\sum_{j=0}^L (s_{j}^\circ+1) )$ and
$\mathcal{H}_{ {\epsilon_n}}=\{\btheta: d_{\rm H}(P_0,P_{\btheta})<{\epsilon_n}\}$ is the Hellinger neighborhood of size ${\epsilon_n}$
\begin{align*}
     \int_{\mathcal{H}_{\epsilon_n}^c}\frac{P_{\btheta}^n}{P_0^n}\widetilde{\pi}(\btheta) d\btheta& \leq 
     \int_{\mathcal{H}_{\epsilon_n}^c \cap \mathcal{F}_n}  \frac{P_{\btheta}^n}{P_0^n}\widetilde{\pi}(\btheta) d\btheta+ \int_{\mathcal{F}_n^c} \frac{P_{\btheta}^n}{P_0^n}\widetilde{\pi}(\btheta) d\btheta\\
     &\leq \int_{\mathcal{H}_{\epsilon_n}^c \cap \mathcal{F}_n}  \frac{P_{\btheta}^n}{P_0^n}\widetilde{\pi}(\btheta) d\btheta+\exp\left(-\frac{(C_0/2)n{\epsilon_n^2}}{\sum u_l}\right)
\end{align*}
where the last inequality follows from Lemma \ref{lem:prior} because by Markov's inequality
\begin{align*}
    \prob_{P_0^n}\left(\int_{\mathcal{F}_n^c} \frac{P_{\btheta}^n}{P_0^n}\widetilde{\pi}(\btheta) d\btheta>\exp\left(-\frac{(C_0/2)n{\epsilon_n^2}}{\sum u_l}\right)\right)&\leq \exp\left(\frac{(C_0/2)n{\epsilon_n^2}}{\sum u_l}\right) \mathbb{E}_{P_0^n}\left(\int_{\mathcal{F}_n^c} \frac{P_{\btheta}^n}{P_0^n}\widetilde{\pi}(\btheta) d\btheta\right)\\
    &\leq \exp\left(\frac{(C_0/2)n{\epsilon_n^2}}{\sum u_l}\right)\widetilde{\Pi}(\mathcal{F}_n^c)=\exp\left(-\frac{(C_0/2)n{\epsilon_n^2}}{\sum u_l} \right)\to 0
\end{align*}
Further,
\begin{align*}
    \int_{\mathcal{H}_{\epsilon_n}^c \cap \mathcal{F}_n}  \frac{P_{\btheta}^n}{P_0^n}\widetilde{\pi}(\btheta) d\btheta &\leq  \underbrace{\int_{\mathcal{H}_{\epsilon_n}^c \cap \mathcal{F}_n} \phi \frac{P_{\btheta}^n}{P_0^n}\widetilde{\pi}(\btheta) d\btheta}_{T_1}+ \underbrace{\int_{\mathcal{H}_{\epsilon_n}^c \cap \mathcal{F}_n} (1-\phi) \frac{P_{\btheta}^n}{P_0^n}\widetilde{\pi}(\btheta) d\btheta}_{T_2}
\end{align*}

Next, borrowing steps from proof of theorem 3.1 in \cite{Pati}, we have $\E_{P_0^n}(\phi)\leq \exp(-C_1n{\epsilon_n^2})$, thus for any $C_1'<C_1$, $\phi\leq \exp(-C_1'n{\epsilon_n^2})$ with probability at least $1-\exp(-(C_1-C_1')n{\epsilon_n^2})$. Thus,
$$T_1 \leq \exp(-C_1'n{\epsilon_n^2})T_1+T_2$$
which implies with dominating probability $T_1 \leq T_2$. Thus, it only remains to show $T_2\leq \exp(-C_2'(n{\epsilon_n^2})/(\sum u_l))$ for some $C_2'>0$. This is true since 

\begin{align*}
    \prob_{P_0^n} (T_2>e^{-\frac{C_2n{\epsilon_n^2}}{\sum u_l}}) &\leq e^{C_2\frac{n{\epsilon_n^2}}{\sum u_l}} \E_{P_0^n}(T_2)\leq e^{\frac{C_2n{\epsilon_n^2}}{\sum u_l}}\int_{\mathcal{H}_{\epsilon_n}^c \cap \mathcal{F}_n}\E_{P_{\btheta}}(1-\phi) \widetilde{\pi}(\btheta)d\btheta\\
    &\leq e^{\frac{C_2n{\epsilon_n^2}}{\sum u_l}}\int_{\mathcal{H}_{\epsilon_n}^c \cap \mathcal{F}_n} e^{-C_2 nd^2_{\rm H}(P_0,P_{\btheta})}\widetilde{\pi}(\btheta)d\btheta\\
    &\leq e^{\frac{C_2n{\epsilon_n^2}}{\sum u_l}} e^{-C_2n {\epsilon_n^2} }\int_{\mathcal{H}_{\epsilon_n}^c \cap \mathcal{F}_n} \widetilde{\pi}(\btheta)d\btheta \leq \exp(-C_2'n {\epsilon_n^2}/\sum u_l )
\end{align*}
Therefore, for sufficiently large $n$ and $C=\min(C_0/2,C_2')/2$
\begin{align*}
     \int_{\mathcal{H}_{\epsilon_n}^c}\frac{P_{\btheta}^n}{P_0^n}\widetilde{\pi}(\btheta) d\btheta& \leq 2\exp(-C_2'n {\epsilon_n^2}/\sum u_l )+\exp(-(C_0/2)n{\epsilon_n^2}/\sum u_l )    \leq \exp(-Cn{\epsilon_n^2}/\sum u_l)
\end{align*}
\end{proof}

\begin{lemma}
\label{lem:kl-denominator}
    Suppose Lemma \ref{lem:kl} part 1. in the main paper holds, then  for any $M_n \to \infty$ , with dominating probability,
    $$ \log  \int \frac{P_0^n}{P_{\btheta}^n}\widetilde{\pi}(\btheta)d\btheta\leq nM_n(\sum r_l+\xi)$$
\end{lemma}

\begin{proof} 
	By Markov's inequality,
	\begin{align}
	\nonumber \prob_{P_0^n}\left(\log  \int  \frac{P_0^n}{P_{\btheta}^n}\widetilde{\pi}(\btheta)\geq nM_n(\sum r_l+\xi)\right)&\leq\frac{1}{n M_n(\sum r_l+\xi) }\E_{P_0^n}\left|\log\int  \frac{P_{\btheta}^n}{P_0^n}\widetilde{\pi}(\btheta)d\btheta\right|\\
	\nonumber 	&\hspace{-5mm}=\frac{1}{n M_n(\sum r_l+\xi)}\int \left|\log\int  \frac{P_{\btheta}^n}{P_0^n}\widetilde{\pi}(\btheta)d\btheta\right|P_0^n d\mu\\
	\nonumber &\hspace{-5mm}\leq \frac{1}{n M_n(\sum r_l+\xi)}\left(d_{\rm KL}(P_0^n,L^*)+\frac{2}{e}\right)
	\end{align}
	where  $L^*=\int P_{\btheta}^n\widetilde{\pi}(\btheta)d\btheta$ and  the last inequality follows from Lemma \ref{lem:mod-kl}.
	\begin{align}
	\nonumber	d_{\rm KL}(P_0^n,L^*)&=\E_{P_0^n}\left(\log \frac{P_0^n}{\int P_{\btheta}^n\widetilde{\pi}(\btheta)d\btheta}\right)\leq \E_{P_0^n}\left(\log \frac{P_0^n}{\int_{N_{\sum r_l+\xi}} P_{\btheta}^n\widetilde{\pi}(\btheta)d\btheta }\right)\\
	\nonumber 	&\leq \int_{\mathcal{N}_{\sum r_l+\xi}} \widetilde{\pi}(\btheta)d\btheta+\int_{\mathcal{N}_{\sum r_l+\xi}} d_{\rm KL}(P_0^n,P_{\btheta}^n)\widetilde{\pi}(\btheta)d\btheta \hspace{5mm}\text{Jensen's inequality}\\
	\nonumber &\leq -\log e^{-nC(\sum r_l+\xi)}+n(\sum r_l+\xi)=n(C+1)(\sum r_l+\xi)
	\end{align}
	where the last inequality follows from Lemma \ref{lem:kl} part 1. in the main paper.
	The proof follows by noting $C/M_n \to 0$.
\end{proof}

\begin{lemma}
\label{lem:q-determination}
   Suppose Lemma \ref{lem:kl} part 2. in the main paper holds, then  for any $M_n \to \infty$ , with dominating probability,
   $$d_{\rm KL}(q,\pi)+\sum_{\boldz}\int \log \frac{P_0^n}{P_{\btheta}^n} q(\btheta,\boldz)d\btheta \leq n M_n(\sum r_l+\xi)$$
 \end{lemma}
 
\begin{proof}

\noindent


\noindent By Markov's inequality we have
\begin{align*}
&\prob_{P_0^n}\left(d_{\rm KL}(q,\pi)+ \sum_{\boldz} \int  q(\btheta,\boldz) \log  \frac{P_0^n}{P_{\btheta}^n}d\btheta> n M_n(\sum r_l+\xi) \right) \\
&\leq \frac{1}{nM_n(\sum r_l+\xi)}\left(d_{\rm KL}(q,\pi)+\E_{P_0^n}\left|\sum_{\boldz} \int q(\btheta,\boldz)\log  \frac{P_0^n}{P_{\btheta}^n}d\btheta\right|\right)\\
&\leq \frac{1}{nM_n(\sum r_l+\xi)} \left(d_{\rm KL}(q,\pi)+\E_{P_0^n}\left(\sum_{\boldz} \int q(\btheta,\boldz) \left|\log  \frac{P_{\btheta}^n}{P_0^n}\right|d\btheta\right)\right)\\
&=\frac{1}{nM_n(\sum r_l+\xi)} \left(d_{\rm KL}(q,\pi)+\sum_{\boldz} \int q(\btheta,\boldz)\int \left|\log \frac{P_0^n}{P_{\btheta}^n}\right| P_0^n d\mu d\btheta\right)
\end{align*}
By Lemma \ref{lem:mod-kl}, we get
\begin{align*}
    &\leq \frac{1}{nM_n(\sum r_l+\xi)} \left(d_{\rm KL}(q,\pi)+\sum_{\boldz} \int q(\btheta,\boldz)\left(d_{\rm KL}(P_0^n,P_{\btheta}^n)+\frac{2}{e}\right)d\btheta \right) \\
    &=\frac{1}{nM_n(\sum r_l+\xi)}\left(d_{\rm KL}(q,\pi)+n \sum_{\boldz} \int q(\btheta,\boldz) 
   d_{\rm KL}(P_0,P_{\btheta})d\btheta+\frac{2}{e}\right)\\
    &=\frac{C}{n  M_n(\sum r_l+\xi)} \left (n (\sum r_l+\xi)+(2/e)\right) \to 0 
\end{align*}
where the last line in the above holds due to Lemma 4.3 part 2. in the main paper.
\end{proof}

\subsection{Proof of Lemmas and Corollary in the main paper}
\noindent {\bf Proof of Lemma \ref{lem:test}}


Take  $s_l^\circ+1 = (n\epsilon_n^2 )/(\sum_{j=0}^L u_{j} )$  
and $\log B_l^\circ=(n\epsilon_n^2)/((L+1)\sum_{j=0}^L (s_{j}^\circ+1) )$.

We know from Lemma 2 of \cite{Ghoshal-Van-der-Vaart-2007} that, there exists a function $\varphi \in [0,1]$, such that
\begin{align*}
    \E_{P_0}(\varphi) & \leq \exp \{-nd^2_{\rm H}(P_{\btheta_1},P_0)/2 \} \\
    \E_{P_{\btheta}}(1-\varphi) & \leq \exp \{-nd^2_{\rm H}(P_{\btheta_1},P_0)/2 \}
\end{align*}
for all $P_{\btheta} \in \calF(L,\bm{k},\bm{s}^\circ,\bm{B}^\circ)$ satisfying $d_{\rm H}(P_{\btheta},P_{\btheta_1}) \leq d_{\rm H}(P_0,P_{\btheta_1})/18$.

Let $H=N(\epsilon_n/19,\calF(L,\bm{k},\bm{s}^\circ,\bm{B}^\circ),d_{\rm H}(.,.))$ denote the covering number of $\calF(L,\bm{k},\bm{s}^\circ,\bm{B}^\circ)$, i.e., there exist $H$ Hellinger balls of radius $\epsilon_n/19$, that entirely cover $\calF(L,\bm{k},\bm{s}^\circ,\bm{B}^\circ)$. 
For any $\btheta \in \calF(L,\bm{k},\bm{s}^\circ,\bm{B}^\circ)$ w.l.o.g we assume $P_{\btheta}$ belongs to the Hellinger ball centered at $P_{\btheta_h}$ and if $d_{\rm H}(P_{\btheta},P_0) > \epsilon_n$, then we must have that $d_{\rm H}(P_0,P_{\btheta_h}) > (18/19)\epsilon_n$ and there exists a testing function $\varphi_h$, such that 
 \begin{align*}
    \E_{P_0}(\varphi_h) & \leq \exp \{-nd^2_{\rm H}(P_{\btheta_h},P_0)/2 \} \\
    & \leq \exp \{-((18^2/19^2)/2)n{\epsilon_n^2}  \} \\   
    \E_{P_{\btheta}}(1-\varphi_h) & \leq \exp \{-nd^2_{\rm H}(P_{\btheta_h},P_0)/2 \} \\
    & \leq \exp \{-n(d_{\rm H}(P_0,P_{\btheta})-\epsilon_n/19)^2/2 \} \\
    & \leq \exp \{-((18^2/19^2)/2)nd^2_{\rm H}(P_0,P_{\btheta}) \}.
\end{align*}
Next we define $\phi=\max_{h=1,\cdots,H} \varphi_h$. Then we must have
\begin{align*}
    \E_{P_0}(\phi) & \leq \sum_h \E_{P_0}(\varphi_h) \leq H \exp \{-((18^2/19^2)/2)n{\epsilon_n^2}  \} \\
    & \leq \exp \{-((18^2/19^2)/2)n{\epsilon_n^2} - \log H \}
\end{align*}
Using Lemma \ref{lem:covering-1} with $\bm{s}=\bolds^\circ$ and $\bm{B}=\boldB^\circ$, we get
\begin{align*}
    &\log H  = \log N(\epsilon_n/19,\mathcal{F}(L,\bm{k},\bm{s}^\circ,\bm{B}^\circ),d_{\rm H}(.,.)) \\
    & \leq \log N(\sqrt{8}\sigma^2_e \epsilon_n/19,\calF(L,\bm{k},\bm{s}^\circ,\bm{B}^\circ),||.||_\infty) \\
    & \leq \log \left[ \prod_{l=0}^L \left(\frac{38}{\sqrt{8}\sigma^2_e \epsilon_n}(L+1) \left( \prod_{j=0}^L B_j^\circ \right) k_{l+1}\right)^{(s_{l}^\circ+1)} \right] \\
    & =  \sum_{l=0}^L (s_{l}^\circ+1) \log \left(\frac{38}{\sqrt{8}\sigma^2_e \epsilon_n}(L+1) \left( \prod_{j=0}^L B_j^\circ \right) k_{l+1}\right) \\
    & \leq C \left[\sum_{l=0}^L (s_{l}^\circ+1) \left( \log \frac{1}{\epsilon_n} + \log (L+1) +  \sum_{j=0}^L \log B_j^\circ + \log k_{l+1} \right) \right] \\
    &\leq  C \sum_{l=0}^L (s_{l}^\circ+1) (\log n+\log (L+1)+\sum_{j=0}^{L} \log B_j^\circ+\log k_{l+1}) \\
    &\leq  C \sum_{l=0}^L (s_{l}^\circ+1) (\log n+\log (L+1)+\sum_{j=0}^{L} \log B_j^\circ+\log k_{l+1}+\log (k_l+1)) \leq C n \epsilon_n^2
    %
    %
\end{align*}
where, C in each step is different which tends to absorb the extra constants in it. First inequality holds due to the following 
$$d^2_{\rm H}(P_{\btheta},P_0) \leq 1-\exp \left\{-\frac{1}{8\sigma^2_e}||\eta_0-\eta_{\btheta}||_\infty^2 \right\} $$
and $\epsilon_n = o(1),$ the second inequality is due to \eqref{covering_num_bound}, and fourth inequality is due to $s_{l}^\circ\log(1/\epsilon_n) \asymp s_{l}^\circ \log n$. Therefore,
$$\E_{P_0}(\phi) \leq \sum_h \E_{P_0}(\varphi_h) = \exp \{-C_1n{\epsilon_n^2} \} $$
for some $C_1=(18^2/19^2)/2-1/4$. On the other hand, for any $\btheta$, such that $d_{\rm H}(P_{\btheta},P_0)\geq \epsilon_n$, say $P_{\btheta}$ belongs to the $h$th Hellinger ball, then we have 
$$ \E_{P_{\btheta}}(1-\phi) \leq \E_{P_{\btheta}}(1-\varphi_h) \leq \exp \{-C_2nd^2_{\rm H}(P_0,P_{\btheta}) \}$$
where $C_2=(18^2/19^2)/2$. This concludes the proof. \hfill \qedsymbol

\noindent {\bf Proof of Lemma \ref{lem:prior}}

\begin{align}
\label{e:ass-lem-prior}
{\it Assumption:}  \hspace{5mm} s_{l}^\circ+1 = (n\epsilon_n^2 )/(\sum_{j=0}^L u_{j} ),\: \lambda_{l} k_{l+1}/s_l^\circ \to 0,\: \sum u_l\log L=o(n\epsilon_n^2)
\end{align}

\begin{align*}
    \widetilde{\Pi}(\mathcal{F}(L,\bm{k},\bolds^\circ,\boldB^\circ)^c) & \leq \widetilde{\Pi}\left(\bigcup_{l=0}^L \{ ||\widetilde{\boldw}_l||_0>s_l^\circ\} \right) + \widetilde{\Pi}\left(\bigcup_{l=0}^L \{ ||\widetilde{\boldw}_l||_\infty > B_l^\circ \} \right) \\
    & \leq \sum_{l=0}^L \widetilde{\Pi}(||\widetilde{\boldw}_l||_0>s_l^\circ) + \sum_{l=0}^L \widetilde{\Pi}(||\widetilde{\boldw}_l||_\infty > B_l^\circ) \\
    & = \sum_{l=0}^L \sum_{\boldz} \Pi(||\widetilde{\boldw}_l||_0>s_l^\circ|\boldz)\pi(\boldz)+\sum_{l=0}^L \sum_{\boldz} \Pi(||\widetilde{\boldw}_l||_\infty > B_l^\circ|\boldz)\pi(\boldz)\\
    & \leq  \sum_{l=0}^L \prob\left(\sum_{i=1}^{k_{l+1}} z_{li}>s_{l}^\circ\right) + \sum_{l=0}^L \prob\left( \sup_{i=1,\cdots,k_{l+1}} ||\overline{\boldw}_{li}||_1 > B_l^\circ\Big|\boldz\right)
\end{align*}
where $\widetilde{\boldw}_l=(||\overline{\boldw}_{l1}||_1,\cdots,||\overline{\boldw}_{lk_{l+1}}||_1)^T$ and the last inequality holds since $\Pi(||\widetilde{\boldw}_l||_0>s_l^\circ|\boldz)\leq 1$, $\Pi(||\widetilde{\boldw}_l||_0>s_l^\circ|\boldz)=1$ iff $\sum z_{li}>\bolds_l^\circ$ and  $\pi(\boldz)\leq 1$. We will now break the proof in two parts as follows.

\noindent{\it Part 1.}  
\begin{align*}
    \sum_{l=0}^L \prob\left(\sum_{i=1}^{k_{l+1}} z_{li}>s_{l}^\circ \right)&= \sum_{l=0}^L \prob\left(\sum_{i=1}^{k_{l+1}} z_{li}-k_{l+1}\lambda_l>s_{l}^\circ-k_{l+1}\lambda_l\right)
\end{align*}
By Bernstein inequality
\begin{align*}
    &\leq \sum_{l=0}^L \exp\left(\frac{-1/2(s_l^\circ-k_{l+1}\lambda_l)^2}{k_{l+1}\lambda_l(1-\lambda_l)+1/3(s_l^\circ-k_{l+1}\lambda_l)}\right) \leq \sum_{l=0}^L \exp\left(\frac{-1/2(s_l^\circ-k_{l+1}\lambda_l)^2}{k_{l+1}\lambda_l+1/3(s_l^\circ-k_{l+1}\lambda_l)}\right) \\
    & = \sum_{l=0}^L \exp\left(\frac{-s_l^\circ/2(1-k_{l+1}\lambda_l/s_l^\circ)^2}{1/3(1+2k_{l+1}\lambda_l/s_l^\circ)}\right)  \to  \sum_{l=0}^L \exp\left(-\frac{3s_l^\circ}{2}  \right) \qquad \qquad \text{since} \hspace{1mm} \frac{k_{l+1}\lambda_l}{s_l^\circ} \to 0 \hspace{1mm} \text{by} \hspace{1mm}  \eqref{e:ass-lem-prior}\\
    & =  \sum_{l=0}^L \exp\left(-\frac{3n\epsilon_n^2}{4 \sum u_l} + \frac{3}{2} \right)  \leq  5(L+1) \exp\left(-\frac{n\epsilon_n^2}{2 \sum u_l}  \right)  \leq \exp\left(-\frac{n\epsilon_n^2}{4\sum u_l} \right)
\end{align*}
since $\sum u_l \log(5(L+1)) \sim \sum u_l \log L=o( n\epsilon_n^2)$ by \eqref{e:ass-lem-prior}.

\noindent {\it Part 2.}
\begin{align*}
    \sum_{l=0}^L \prob\left( \sup_{i=1,\cdots,k_{l+1}} ||\boldw_{li}||_1 > B_l^\circ\Big|\boldz\right)& \leq \sum_{l=0}^L \sum_{i=1}^{k_{l+1}}\prob\left( ||\boldw_{li}||_1 > B_l^\circ\Big|\boldz \right) \\
    & \leq \sum_{l=0}^L \sum_{i=1}^{k_{l+1}} \prob\left( ||\boldw_{li}||_\infty > \frac{B_l^\circ}{k_{l}+1}\Big|\boldz \right)\\
    &\leq \sum_{l=0}^L \sum_{i=1}^{k_{l+1}} \sum_{j=1}^{k_l+1}\prob\left(|w_{lij}|>\frac{B_l^\circ}{k_{l}+1}\Big|\boldz\right)\\
    &\leq 2\sum_{l=0}^L \sum_{i=1}^{k_{l+1}} \sum_{j=1}^{k_l+1} \exp\left(-\frac{{B_l^\circ}^2}{(k_l+1)^2}\right) \hspace{5mm} \text{By  concentration inequality} \\
    &= 2\sum_{l=0}^L \sum_{i=1}^{k_{l+1}} \sum_{j=1}^{k_l+1} \exp\Big(-\exp(\frac{2n\epsilon_n^2}{((L+1)\sum_{j'=0}^L (s_{j'}^\circ+1)}-2\log (k_l+1))\Big)\\
    &\leq \sum_{l=0}^L \sum_{i=1}^{k_{l+1}} \sum_{j=1}^{k_l+1} \frac{1}{(L+1) k_{l+1}(k_l+1)}\exp(-n\epsilon_n^2) = \exp(-n\epsilon_n^2)
    \end{align*}
where the third inequality holds since $|w_{lij}|$ given $\boldz$ is bound above by a $|N(0,\sigma_0^2)|$ random variable. The above proof holds as long as
$$\exp\left(\frac{2n\epsilon_n^2 }{(L+1) \sum_{j'=0}^L (s_{j'}^\circ+1)}-2\log (k_l+1)\right)\geq n \epsilon_n^2+\log (L+1)+\log k_{l+1}+\log (k_l+1)+\log 2$$
Taking log on both sides we get
$$\left(\frac{n\epsilon_n^2 }{ (L+1)\sum_{j'=0}^L (s_{j'}^\circ+1)}-\log (k_l+1)\right)\geq \frac{1}{2}\log (n \epsilon_n^2+\log (L+1)+\log k_{l+1}+\log (k_l+1)+\log 2)$$
This is true since $\sum_{j'=0}^L (s_{j'}^\circ+1)  = (L+1)n\epsilon_n^2/\sum u_l$ is bounded above by
\begin{align*} \frac{n\epsilon_n^2}{(L+1)(\log (k_{l}+1)+\frac{1}{2}\log (n \epsilon_n^2+\log (L+1)+\log k_{l+1}+\log (k_l+1)+\log 2)}
\end{align*}  \hfill \qedsymbol

\noindent {\bf Proof of Lemma \ref{lem:kl} part 1.}

\begin{align}
\label{e:ass-lem-kl}
{\it Assumption:} \hspace{5mm}-\log \lambda_l =O\{(k_l+1) \vartheta_l\},\: -\log(1-\lambda_l) = O\{ (s_l/k_{l+1})(k_l+1) \vartheta_l\}
\end{align}

\begin{align*}
    d_{\text{KL}}(P_0,P_{\btheta}) & = \int_{\boldx\in[0,1]^p} \int_{y \in R} \left(\log \frac{P_0(y,\boldx)}{P_{\btheta}(y,\boldx)} \right) P_0(y,\boldx) dy d\boldx \\
P_0(y,\boldx) & = \frac{1}{\sqrt{2\pi \sigma^2_e}}  \exp \left(- \frac{(y-\eta_0(\boldx))^2}{2\sigma^2_e} \right) \hspace{5mm}
P_{\btheta}(y,\boldx)  = \frac{1}{\sqrt{2\pi \sigma^2_e}}  \exp \left(- \frac{(y-\eta_{\btheta}(\boldx))^2}{2\sigma^2_e} \right)
\end{align*}
So we get,
\begin{align}
\label{e:kl-p0-pt}
\nonumber    d_{\text{KL}}(P_0,P_{\btheta}) & = \int_{\boldx\in[0,1]^p} \int_{y \in \mathbb{R}} \log \left( \exp \left[- \frac{(y-\eta_0(\boldx))^2}{2\sigma^2_e} + \frac{(y-\eta_{\btheta}(\boldx))^2}{2\sigma^2_e} \right] \right) P_0(y,\boldx) dy d\boldx \\
\nonumber    & = \int_{\boldx\in[0,1]^p} \int_{y \in \mathbb{R}} \frac{2y(\eta_0(\boldx)-\eta_{\btheta}(\boldx))- (\eta_0^2(\boldx)-\eta_{\btheta}^2(\boldx))}{2\sigma^2_e} P_0(y,\boldx) dy d\boldx \\
\nonumber    & = \int_{\boldx\in[0,1]^p} \frac{2\eta_0^2(\boldx)-2\eta_0(\boldx)\eta_{\btheta}(\boldx)- \eta_0^2(\boldx)+\eta_{\btheta}^2(\boldx)}{2\sigma^2_e} d\boldx \\
    & = \int_{\boldx\in[0,1]^p} \frac{(\eta_0(\boldx)-\eta_{\btheta}(\boldx))^2}{2} d\boldx=\frac{1}{2}||\eta_0-\eta_{\btheta}||_2^2
\end{align}
where, $\sigma^2_e = 1$ can be chosen w.l.o.g.
Next, let $\eta_{\btheta^*}(\boldx)$  be $\btheta^*$ satisfying $\arg \min_{\eta_{\btheta} \in \calF(L,\boldk,\bolds,\boldB)} \left\|\eta_{\btheta}-\eta_{0}\right\|_{\infty}^{2}$. Then, 
\begin{equation} \label{first_diff}
    ||\eta_{\btheta^*}-\eta_0||_1 \leq ||\eta_{\btheta^*}-\eta_0||_\infty = \sqrt{\xi}
\end{equation}
Here, we redefine $\overline{\bdelta}_l$ by considering the $L_1$ norms of the rows of $\overline{\boldD}_l = \overline{\boldW}_l - \overline{\boldW}^*_l$ as follows
$$\overline{\boldD}_l=(\overline{\bm{d}}_{l1}^\top, \cdots,
\overline{\bm{d}}_{lk_{l+1}}^\top)^{\top} \hspace{5mm} \overline{\bdelta}_l=(
||\overline{\bm{d}}_{l1}||_1 , \cdots,||\overline{\bm{d}}_{lk_{l+1}}||_1 )$$
Next we define a neighborhood $\calM_{\sqrt{\sum r_l} }$ as follows:
$$ \calM_{\sqrt{\sum r_l}} =\left\{\btheta: ||\overline{\bm{d}}_{li}||_1 \leq \frac{\sqrt{\sum r_l}B_l}{(L+1)(\prod_{j=0}^L B_j)}, i\in \mathcal{S}_l, ||\overline{\bm{d}}_{li}||_1=0, i \in \mathcal{S}_l^c, l = 0,\cdots,L  \right\}$$
where $\mathcal{S}_l^c$ is the set where $||\overline{\boldw}_{li}^*||_1=0$, $l=0,\cdots, L$.
Then, for every $\btheta \in \calM_{\sqrt{\sum r_l}}$ using \eqref{l1_bound_NN}, we have 
\begin{equation} \label{second_diff}
    || \eta_{\btheta} - \eta_{\btheta^*} ||_1 \leq \sqrt{\sum r_l} 
\end{equation}
Combining \eqref{first_diff} and \eqref{second_diff}, we get for $\btheta \in \calM_{\sqrt{\sum r_l}}, ||\eta_{\btheta}-\eta_0||_1 \leq \sqrt{\sum r_l}+ \sqrt{\xi}.$
So we get, 
$$d_{\text{KL}}(P_0,P_{\btheta}) \leq \frac{(\sqrt{\sum r_l}+ \sqrt{\xi})^2}{2}\leq \sum r_l+ \xi$$
Since $\btheta \in \mathcal{N}_{\sum r_l+\xi}$ for every $\btheta \in \calM_{\sqrt{\sum r_l}}$; therefore,
$$\int_{\btheta \in \mathcal{N}_{\sum r_l + \xi}} \widetilde{\pi}(\btheta)d\btheta \geq \int_{\btheta \in \calM_{\sqrt{\sum r_l}}} \widetilde{\pi}(\btheta) d\btheta$$

\noindent Let $\delta_n = (\sqrt{\sum r_l}B_l)/((L+1)(\prod_{j=0}^L B_j))$ and  $A = \{\overline{\boldw}_{li} :\enskip  ||\overline{\boldw}_{li} - \overline{\boldw}_{li}^*||_1 \leq \delta_n \}$
\begin{align}
    \nonumber \widetilde{\Pi}\left(\calM_{\sqrt{\sum r_l}}\right) & = \sum_{\boldz} \Pi\left(\calM_{\sqrt{\sum r_l}}\Big|\boldz\right)\pi(\boldz)\\
    \nonumber    &\geq   \sum_{\{\boldz: z_{li}=1, i \in \mathcal{S}_l, z_{li}=0, i \in \mathcal{S}_l^c, l=0,\cdots, L\}} \Pi\left(\calM_{\sqrt{\sum r_l}}\Big|\boldz\right)\pi(\boldz) \\
\nonumber    &= \prod_{l=0}^L  (1-\lambda_l)^{k_{l+1}-s_l} \lambda_l^{s_l}\prod_{i \in S_l} \E (\indicator_{\{\overline{\boldw}_{li} \in A \}}|z_{li}=1) \\
    \nonumber &\geq \prod_{l=0}^L (1-\lambda_l)^{k_{l+1}-s_l} \lambda_l^{s_l} \prod_{i \in \mathcal{S}_l} \int_{\overline{\boldw}_{li} \in A} \left(\frac{1}{2\pi}\right)^{\frac{k_l+1}{2}}\prod_{j=1}^{k_l+1} \exp \left(- \frac{\overline{w}_{lij}^2}{2} \right) d \overline{w}_{lij} \\
    \nonumber & \geq \prod_{l=0}^L (1-\lambda_l)^{k_{l+1}-s_l} \lambda_l^{s_l} \prod_{i \in \mathcal{S}_l}  \left(\frac{1}{2\pi}\right)^{\frac{k_l+1}{2}}\prod_{j=1}^{k_l+1} \int_{\overline{w}^*_{lij}-\frac{\delta_n}{k_l+1}}^{\overline{w}^*_{lij}+\frac{\delta_n}{k_l+1}} \exp \left(- \frac{\overline{w}_{lij}^2}{2} \right) d \overline{w}_{lij} \\
    \nonumber & = \prod_{l=0}^L (1-\lambda_l)^{k_{l+1}-s_l} \lambda_l^{s_l} \prod_{i \in \mathcal{S}_l}  \left(\frac{1}{2\pi}\right)^{\frac{k_l+1}{2}}\prod_{j=1}^{k_l+1} \frac{2\delta_n}{k_l+1}\exp \left(- \frac{\widehat{w}_{lij}^2}{2} \right)\end{align}
  where the third equality follows since  $\E (\indicator_{\{\overline{\boldw}_{li} \in A \}}|z_{li}=0)=1 $ since $||\overline{\boldw}_{li}^*||_1=0$, for $i \in \mathcal{S}_l^c$. The last equality is by mean value theorem, $\widehat{w}_{lij} \in [\overline{w}^*_{lij}-\delta_n/(k_l+1),\overline{w}^*_{lij}+\delta_n/(k_l+1)]$, thus
    \begin{align}
 \nonumber  \hspace{5mm}  & = \prod_{l=0}^L (1-\lambda_l)^{k_{l+1}-s_l} \lambda_l^{s_l} \prod_{i \in \mathcal{S}_l} \exp \Bigg( \frac{k_l+1}{2} \log \frac{1}{2\pi} + (k_l+1) \log \frac{2\delta_n}{k_l+1} -  \sum_{j=1}^{k_l+1} \frac{\widehat{w}_{lij}^2}{2} \Bigg) \\
    & \nonumber = \exp\Bigg[ -\sum_{l=0}^L\Big\{ s_l \log \left( \frac{1}{\lambda_l}\right)+(k_{l+1}-s_l)\log\left(\frac{1}{1-\lambda_l}\right) \\
    & \nonumber \qquad \qquad  \quad  + \sum_{i\in \mathcal{S}_l} \Bigg( - \frac{k_l+1}{2} \log \frac{1}{2\pi} - (k_l+1) \log \frac{2\delta_n}{k_l+1} +  \sum_{j=1}^{k_l+1} \frac{\widehat{w}_{lij}^2}{2} \Bigg) \Bigg\}  \Bigg] \\
    & \nonumber = \exp\Bigg[ -\sum_{l=0}^L\Bigg\{ s_l \log \left( \frac{1}{\lambda_l}\right)+(k_{l+1}-s_l)\log\left(\frac{1}{1-\lambda_l}\right) \\
    & \qquad \qquad \qquad  - \frac{s_l(k_l+1)}{2} \log \frac{1}{2\pi} - s_l(k_l+1) \log \frac{2\delta_n}{k_l+1} + \sum_{i\in \mathcal{S}_l}  \sum_{j=1}^{k_l+1} \frac{\widehat{w}_{lij}^2}{2} \Bigg\}  \Bigg] \label{main_expression}
\end{align}
 Now,
\begin{align}
\label{weight_bound}
    \nonumber \sum_{l=0}^L \sum_{i\in \mathcal{S}_l} \sum_{j=1}^{k_l+1} \frac{\widehat{w}_{lij}^2}{2}  & \leq \frac{1}{2}\sum_{l=0}^L \sum_{i\in \mathcal{S}_l} \sum_{j=1}^{k_l+1} \max ((\overline{w}^*_{lij}-\delta_n/(k_l+1))^2,(\overline{w}^*_{lij}+\delta_n/(k_l+1))^2) \\
    \nonumber & \leq  \sum_{l=0}^L \sum_{i\in \mathcal{S}_l} \sum_{j=1}^{k_l+1}(\overline{w}_{lij}^{*2} + \delta_n^2/(k_{l}+1)^2)  \leq \sum_{l=0}^L \sum_{i\in \mathcal{S}_l} ||\overline{\boldw}_{li}^*||_1^2 +\sum_{l=0}^L \sum_{i\in \mathcal{S}_l}\delta_n^2/(k_l+1) \\
& \leq \sum_{l=0}^L s_l(B_l^2 +1) \leq n \sum r_l \leq n \left( \sum r_l + \xi \right)   
\end{align}
where the above line uses $ \delta_n \to 0 $. Finally
\begin{align}
    \nonumber & \sum_{l=0}^L \left( s_l \log \left( \frac{1}{\lambda_l}\right)+(k_{l+1}-s_l)\log\left(\frac{1}{1-\lambda_l}\right) - \frac{s_l(k_l+1)}{2} \log \frac{1}{2\pi} - s_l(k_l+1) \log \frac{2\delta_n}{k_l+1} \right)\\
    \nonumber & \leq \sum_{l=0}^L \Bigg(C n r_l + \frac{s_l(k_l+1)}{2} \Bigg\{  2\log (k_l+1) + 2\log (L+1) + 2\sum_{m=0,m\neq l}^{L} \log B_m - \log \sum r_l  \Bigg\}\Bigg)\\
   & \leq Cn \sum r_l  \leq Cn \left( \sum r_l + \xi \right) \label{rest_bound}
\end{align}
where the first inequality follows from \eqref{e:ass-lem-kl} and expanding $\delta_n$. The   last inequality follows since $n\sum r_l \to \infty$ which implies $-\log \sum r_l=O(\log n)$. Combining \eqref{weight_bound} and \eqref{rest_bound} and replacing \eqref{main_expression}, the proof follows. \hfill \qedsymbol
    
    \vspace{1mm}
\noindent{\bf Proof of Lemma \ref{lem:kl} part 2.}
\label{lemma:KL_likelihood_bound_lemma}

\begin{align*}
\label{e:ass-lem-kl}
{\it Assumption:} \hspace{5mm}-\log \lambda_l =O\{(k_l+1) \vartheta_l\},\: -\log(1-\lambda_l) = O\{ (s_l/k_{l+1})(k_l+1) \vartheta_l\}
\end{align*}

\noindent Suppose there exists $q \in \mathcal{Q}^{\bf MF}$ such that
\begin{eqnarray}
\nonumber d_{\rm KL}(q,\pi) \leq C_1 n \sum r_l, \label{KL_ineq}\\
\sum_{\boldz}\int_{\bTheta} \left\|\eta_{\btheta}-\eta_{\btheta^*}\right\|_{2}^{2} q(\btheta,\boldz)d\btheta \leq \sum r_l. \label{inf_norm_ineq}
\end{eqnarray}
Recall $\btheta^* = \arg \min_{\btheta \in \btheta(L,p,\bolds,\boldB)} \left\|\eta_{\btheta}-\eta_{0}\right\|_{\infty}^{2}$. By 
 relation \eqref{e:kl-p0-pt},
\begin{align*}
\sum_{\boldz}\int nd_{\rm KL}(P_0,P_{\btheta})q(\btheta,\boldz)d\btheta&= \sum_{\boldz} \frac{n}{2}\int ||\eta_0-\eta_{\btheta}||_2^2 q(\btheta,\boldz)d\btheta \\
&\leq \frac{n}{2}\sum_{\boldz}\int ||\eta_{\btheta^*}-\eta_{\btheta}||_2^2q(\btheta,\boldz)d\btheta+\frac{n}{2}||\eta_{\btheta^*}-\eta_0||_\infty^2\\
&\leq C n (\sum r_l+\xi)
\end{align*}
where the above relation is due to \eqref{inf_norm_ineq} which will complete the proof.


We next construct $q\in\calQ^{\bf MF}$ as
$$\begin{aligned}
& \overline{w}_{lij}|z_{li} \sim z_{li} \calN (\overline{w}_{lij}^*,\sigma_l^2) + (1-z_{li}) \delta_0, \qquad z_{li} \sim \text{Bern} (\gamma_{li}^*) 
\qquad \gamma_{li}^* =  \indicator(||\boldw_{li}^*||_1 \neq 0)
\end{aligned}$$
where $\sigma_l^2 = \frac{s_l}{8n(L+1)}(4^{L-l}  (k_{l}+1)\log( k_{l+1} 2^{k_{l}+1}) \prod_{m=0,m\neq l}^{L} B_m^2)^{-1}.$ 

We next consider the relation \eqref{e:q-bound} in Lemma \ref{lem:psi-l-bound}.

\noindent We upper bound the expectation of the supremum of $L_1$ norm of multivariate Gaussian variables:
$$\int \widetilde{W}_l q(\btheta,\boldz) d\btheta \leq \int \sup_{i}||\overline{\boldw}_{li}-\overline{\boldw}_{li}^*||_1 q(\btheta|\boldz) d\btheta  \leq \int \sup_{i}||\overline{\boldw}_{li}-\overline{\boldw}_{li}^*||_1 q(\btheta|\boldz=\boldsymbol{1}) d\btheta$$
since $q(\boldz)\leq 1$. If $z_{li}=1$, then $||\overline{\boldw}_{li}-\overline{\boldw}_{li}^*||_1=0$, thus the above integral is maximized at $\boldz=\boldsymbol{1}$ where $\boldz=\boldsymbol{1}$ indicates all neurons are present in the network. In this case, all $w_{lij}$ are nothing but independent Gaussian random variables. In this direction we make use of concentration inequalities similar to the proof of theorem 2 in \cite{Cherief-Abdellatif-2020}.
Let, $Y = \sup_{i}||\overline{\boldw}_{li}-\overline{\boldw}_{li}^*||_1 $.
\begin{align*}
\exp(t\E Y) & \leq \E(\exp(tY)) =\E [ \sup_{i} \exp (t||\overline{\boldw}_{li}-\overline{\boldw}_{li}^*||_1)]\\
& \leq \sum_{i=1}^{k_{l+1}} \E [\exp(t \sum_{j=1}^{k_l+1} |\overline{w}_{lij}-\overline{w}^*_{lij}|)] = \sum_{i=1}^{k_{l+1}} \prod_{j=1}^{k_l+1} \E [\exp(t |\overline{w}_{lij}-\overline{w}^*_{lij}|)] \\
& = \sum_{i=1}^{k_{l+1}} \prod_{j=1}^{k_l+1} 2 \exp\left[\frac{\sigma^2_l t^2}{2}\right] \Phi (\sigma_lt)\leq k_{l+1} 2^{k_{l}+1}  \exp\left[(k_{l}+1) \frac{\sigma^2_l t^2}{2}\right]
\end{align*}
Thus, $
\E Y \leq (\log(k_{l+1} 2^{k_{l}+1}) + (k_{l}+1) \sigma^2_l t^2/2)/t$. Let $t=(1/\sigma_l) \sqrt{(2/(k_{l}+1)) \log(k_{l+1} 2^{k_{l}+1})}$,
\begin{align*}
\E Y & \leq \sigma_l \sqrt{\frac{k_{l}+1}{2}}\left[ \sqrt{\log(k_{l+1} 2^{k_{l}+1})} + \sqrt{\log(k_{l+1} 2^{k_{l}+1})} \right] \\
& = \sqrt{ 2 \sigma_l^2 (k_{l}+1)\log( k_{l+1} 2^{k_{l}+1})} \leq \sqrt{ 4 \sigma_l^2 (k_{l}+1)\log( k_{l+1} 2^{k_{l}+1})} 
\end{align*}
Similarly, 
$$\int \widetilde{W}_l^2 q(\btheta,\boldz) d\btheta = \int \sup_{i}(||\overline{\boldw}_{li}-\overline{\boldw}_{li}^*||_1)^2 q(\btheta,\boldz) d\btheta \leq \int \sup_{i}(||\overline{\boldw}_{li}-\overline{\boldw}_{li}^*||_1)^2 q(\btheta|\boldz=\boldsymbol{1})$$
Let, $Y' = \sup_{i}(||\overline{\boldw}_{li}-\overline{\boldw}_{li}^*||_1)^2 $.
\begin{align*}
\exp(t\E Y') & \leq \E(\exp(tY'))= \E [ \sup_{i} \exp (t(||\overline{\boldw}_{li}-\overline{\boldw}_{li}^*||_1)^2)]\\
& \leq \sum_{i=1}^{k_{l+1}} \E [\exp(t (\sum_{j=1}^{k_l+1} |\overline{w}_{lij}-\overline{w}^*_{lij}|)^2)]  \leq \sum_{i=1}^{k_{l+1}} \E [\exp(t (k_{l}+1)\sum_{j=1}^{k_{l}+1} (\overline{w}_{lij}-\overline{w}^*_{lij})^2)] \\
& = \sum_{i=1}^{k_{l+1}} \prod_{j=1}^{k_{l}+1} \E [\exp(t (k_{l}+1) (\overline{w}_{lij}-\overline{w}^*_{lij})^2)]  = \sum_{i=1}^{k_{l+1}} \prod_{j=1}^{k_{l}+1} \left(\frac{1}{1-2t(k_{l}+1)\sigma_l^2} \right)^{\frac{1}{2}}\\
& \leq k_{l+1} \left(\frac{1}{1-2t(k_{l}+1)\sigma_l^2} \right)^{\frac{k_{l}+1}{2}}
\end{align*}
Thus, $\E Y' \leq (\log k_{l+1} - ((k_{l}+1)/2)\log (1-2t(k_{l}+1)\sigma_l^2) )/t$. Let $t=1/(4\sigma_l^2(k_{l}+1))$,
\begin{align*}
\E Y' & \leq 4\sigma_l^2(k_{l}+1)\left[ \log k_{l+1} + \left(\frac{k_{l}+1}{2}\right)\log 2 \right] = 4 \sigma_l^2 (k_{l}+1)\log( k_{l+1} 2^\frac{k_{l}+1}{2}) \\
& \leq 4 \sigma_l^2 (k_{l}+1)\log( k_{l+1} 2^{k_{l}+1})
\end{align*}

Next we also get,
$$\int (\widetilde{W}_l+B_l) q(\btheta,\boldz) d\btheta = \int \widetilde{W}_l q(\btheta,\boldz) d\btheta + B_l \leq \sqrt{ 4 \sigma_l^2 (k_{l}+1)\log( k_{l+1} 2^{k_{l}+1})} + B_l \leq 2B_l $$
\begin{align*}
&\int (\widetilde{W}_l+B_l)^2 q(\btheta,\boldz) d\btheta  = \int \widetilde{W}_l^2 q(\btheta,\boldz) d\btheta + 2B_l \int \widetilde{W}_l q(\btheta,\boldz) d\btheta + B_l^2 \\
& \leq 4 \sigma_l^2 (k_{l}+1)\log( k_{l+1} 2^{k_{l}+1}) + 2B_l \sqrt{ 4 \sigma_l^2 (k_{l}+1)\log( k_{l+1} 2^{k_{l}+1})} + B_l^2  \leq 4B_l^2
\end{align*}
\begin{align*}
&\int \widetilde{W}_l (\widetilde{W}_l+B_l) q(\btheta,\boldz) d\btheta  = \int \widetilde{W}_l^2 q(\btheta,\boldz) d\btheta + B_l \int \widetilde{W}_l q(\btheta,\boldz) d\btheta \\
& \leq 4 \sigma_l^2 (k_{l}+1)\log( k_{l+1} 2^{k_{l}+1}) + B_l \sqrt{ 4 \sigma_l^2 (k_{l}+1)\log( k_{l+1} 2^{k_{l}+1})} \\
& \leq \sqrt{ 4 \sigma_l^2 (k_{l}+1)\log( k_{l+1} 2^{k_{l}+1})} \left(\sqrt{ 4 \sigma_l^2 (k_{l}+1)\log( k_{l+1} 2^{k_{l}+1})} + B_l \right)\\
& \leq 2B_l \sqrt{ 4 \sigma_l^2 (k_{l}+1)\log( k_{l+1} 2^{k_{l}+1})}
\end{align*}
since $\sqrt{ 4 \sigma_l^2 (k_{l}+1)\log( k_{l+1} 2^{k_{l}+1})}$ is bounded above by
\begin{align*}
& \sqrt{ \frac{4 s_l}{8n(L+1)}\Bigg(4^{L-l}  (k_{l}+1)\log( k_{l+1} 2^{k_{l}+1}) \prod_{m=0,m\neq l}^{L} B_m^2 \Bigg)^{-1} (k_{l}+1)\log( k_{l+1} 2^{k_{l}+1})} \\
& = B_l \sqrt{ \frac{s_l}{2n(L+1)}\left(4^{L-l}   \prod_{m=0}^{L} B_m^2 \right)^{-1} } \leq B_l, \text{ The quantity in square root $<$ 1 for large $n$.}
\end{align*}
Let $b_j= (k_{j}+1)\log( k_{j+1} 2^{k_{j}+1})$. From relation \eqref{e:q-bound}, we get
\begin{align*}
& \int ||\eta_{\btheta} - \eta_{\btheta^*}||^2_2 q(\btheta,\boldz)d\btheta  \leq \sum_{j=0}^L c_{j-1}^2 (4 \sigma_j^2 b_j) \Bigg( \prod_{m=j+1}^L 4B_m^2\Bigg) \\
&+ 2\sum_{j=0}^{L} \sum_{j'=0}^{j-1} c_{j-1}c_{j'-1} 
2B_j \sqrt{ 4 \sigma_j^2 b_j} 
\Bigg( \prod_{m=j+1}^L 4B_m^2 \Bigg)  \sqrt{4 \sigma_{j'}^2 b_{j'}} \Bigg( \prod_{m=j'+1}^{j-1} 2B_m \Bigg)\\
& = 4 \sum_{j=0}^L 4^{L-j} \sigma_j^2 b_j \Bigg(\prod_{m=0}^{j-1} B_m^2 \Bigg) \Bigg(\prod_{m=j+1}^L B_m^2 \Bigg) \\
& \enskip + 8\sum_{j=0}^{L} \sum_{j'=0}^{j-1} \Bigg(\prod_{m=0}^{j-1} B_m \Bigg)\Bigg(\prod_{m=0}^{j-1} B_m \Bigg) 2B_j \Bigg( \prod_{m=j+1}^L 4B_m^2 \Bigg) \Bigg( \prod_{m=j'+1}^{j-1} 2B_m \Bigg) \sqrt{\sigma_j^2 b_j} \sqrt{\sigma_{j'}^2b_{j'}} \\
& = 4 \sum_{j=0}^L 2^{2L-2j} \sigma_j^2 b_j \prod_{m=0,m\neq j}^{L} B_m^2 \\
& \enskip + 8\sum_{j=0}^{L} \sum_{j'=0}^{j-1} 4^{L-j} 2^{j-j'}\Bigg(\prod_{m=0}^{j-1} B_m \Bigg)\Bigg(\prod_{m=0}^{j-1} B_m \Bigg) \Bigg( \prod_{m=j+1}^L B_m \Bigg) \Bigg( \prod_{m=j'+1}^{L} B_m \Bigg) \sqrt{\sigma_j^2 b_j} \sqrt{\sigma_{j'}^2b_{j'}} \\
& = 4 \sum_{j=0}^L 2^{2L-2j} \sigma_j^2 b_j\Bigg( \prod_{m=0,m\neq j}^{L} B_m^2 \Bigg) \\
& \enskip + 8\sum_{j=0}^{L} \sum_{j'=0}^{j-1} 2^{L-j} 2^{L-j'}\Bigg( \prod_{m=0,m\neq j}^{L} B_m \Bigg) \Bigg( \prod_{m=0,m\neq j'}^{L} B_m \Bigg)\sqrt{\sigma_j^2 b_j} \sqrt{\sigma_{j'}^2 b_{j'}} \\
& = 4 \Bigg( \sum_{j=0}^L  2^{L-j} \sqrt{\sigma_j^2b_j} \Bigg( \prod_{m=0,m\neq j}^{L} B_m \Bigg)\Bigg)^2= 4 \Bigg( \sum_{j=0}^L \sqrt{ \frac{s_j}{8n(L+1)}} \Bigg)^2\\
&= \frac{1}{2n(L+1)} \Bigg( \sum_{j=0}^L \sqrt{s_j} \Bigg)^2  \leq \frac{\sum_{j=0}^L s_j}{2n} \leq \sum_{j=0}^L r_l
\end{align*}
This concludes the proof of \eqref{inf_norm_ineq}. Next, 
\begin{align}
 \nonumber &d_{\rm KL}(q,\pi) \leq \log\frac{1}{\pi(\boldz)} +  \indicator (\boldz=\boldsymbol{\gamma}^*) d_{\rm KL}\Bigg(\Big\{\prod_{l=0}^{L-1} \prod_{i=1}^{k_{l+1}} \prod_{j=1}^{k_{l}+1} \Big\{ \gamma_{li}^*\calN(\overline{w}_{lij}^*,\sigma_l^2) +(1-\gamma_{li}^*)\delta_0 \Big\}\\
 \nonumber & \qquad  \prod_{j=1}^{k_{L}+1} \calN (\overline{w}_{Lj}^*,\sigma_L^2) \Big\} ,\Big\{ \prod_{l=0}^{L-1} \prod_{i=1}^{k_{l+1}} \prod_{j=1}^{k_{l}+1} \Big\{ z_{li}\calN(0,\sigma^2_0) +(1-z_{li})\delta_0 \Big\} \prod_{j=1}^{k_{L}+1} \calN (0,\sigma^2_0) \Big\} \Bigg) \\
 \nonumber &= \log \frac{1}{ \prod_{l=0}^{L-1} \lambda_l^{s_l} (1-\lambda_l)^{k_{l+1}-s_l}} + \sum_{l=0}^{L-1} \sum_{i=1}^{k_{l+1}} \sum_{j=1}^{k_{l}+1} d_{\rm KL}\Big(\gamma_{li}^* \calN(\overline{w}_{lij}^*,\sigma_l^2) +(1-\gamma_{li}^*) \delta_0 , \\
 \nonumber & \qquad  \gamma_{li}^* \calN(0,\sigma^2_0) +(1-\gamma_{li}^*) \delta_0 \Big) + \sum_{j=1}^{k_{L}+1} d_{\rm KL} \Big( \calN(\overline{w}_{Lj}^*,\sigma_L^2) , \calN(0,\sigma^2_0) \Big)
\\
 \nonumber & = \sum_{l=0}^{L-1} \Bigg(s_l \log \frac{1}{\lambda_l} + (k_{l+1}-s_l) \log \frac{1}{1-\lambda_l} \Bigg) + \sum_{l=0}^{L-1} \sum_{i=1}^{k_{l+1}} \sum_{j=1}^{k_{l}+1} \gamma_{li}^* \Bigg\{\frac{1}{2}\log \frac{\sigma^2_0}{\sigma_l^2} + \frac{\sigma_l^2+{\overline{w}_{lij}^*}^2}{2\sigma^2_0} - \frac{1}{2} \Bigg\}\\
 \nonumber & \qquad + \sum_{j=1}^{k_{L}+1} \Bigg\{\frac{1}{2}\log \frac{\sigma^2_0}{\sigma_L^2} + \frac{\sigma_L^2+{\overline{w}_{Lj}^*}^2}{2\sigma^2_0} - \frac{1}{2} \Bigg\} \\
 \nonumber &\leq \sum_{l=0}^{L-1} C n r_l+ \sum_{l=0}^{L-1} \frac{s_l k_l+s_l}{2} \Bigg[\frac{\sigma_l^2}{\sigma^2_0} + \frac{B_l^2}{\sigma^2_0(k_l+1)}-1  + \log \frac{\sigma^2_0}{\sigma_l^2} \Bigg] + \frac{k_{L}+1}{2}\Bigg[\frac{\sigma_L^2}{\sigma^2_0} + \frac{B_L^2}{\sigma^2_0(k_L+1)}-1  + \log \frac{\sigma^2_0}{\sigma_L^2} \Bigg] 
\end{align}
where the first inequality follows from Lemma \ref{lem:mixture-density}. The inequality in the above line uses $\sum_{j=1}^{k_l+1}{\overline{w}_{lij}^*}^2\leq B_l^2$ and similar to the proof of Lemma 4.1 in \cite{Bai-Guang-2020} uses \eqref{e:ass-lem-kl}.
 
 Let $\sigma^2_0 = 1$ and it could be easily derived that $\sigma_l^2 \leq 1$.
\begin{align}
 \nonumber & d_{\rm KL}(q,\pi)  \leq \sum_{l=0}^{L-1} C n r_l + \sum_{l=0}^{L-1} \frac{s_l}{2} (k_l+1) \Bigg[\frac{B_l^2}{k_l+1} - \log \sigma_l^2 \Bigg] + \frac{(k_{L}+1)}{2} \Bigg[\frac{B_L^2}{k_L+1} - \log \sigma_L^2 \Bigg] \\
 \nonumber & = \sum_{l=0}^{L-1} C n r_l + \sum_{l=0}^{L-1} \frac{s_l}{2} (k_l+1) \Bigg[\frac{B_l^2}{k_l+1} - \log \Bigg(\frac{s_l}{8n(L+1)}\Bigg[4^{L-l}  b_l \prod_{m=0,m\neq l}^{L} B_m^2\Bigg]^{-1} \Bigg) \Bigg] \\
 \nonumber & \enskip +\frac{(k_{L}+1)}{2} \Bigg[\frac{B_L^2}{k_L+1} -  \log \Bigg(\frac{1}{8n(L+1)} \Bigg[b_L \prod_{m=0,m\neq L}^{L} B_m^2\Bigg]^{-1} \Bigg) \Bigg] \\
 \nonumber & = \sum_{l=0}^{L-1} C n r_l + \sum_{l=0}^{L} \frac{s_l}{2} (k_l+1) \Bigg[\frac{B_l^2}{k_l+1} - \log \Bigg(\frac{s_l}{8n(L+1)}\Bigg[4^{L-l}  b_l\prod_{m=0,m\neq l}^{L} B_m^2\Bigg]^{-1} \Bigg) \Bigg] \\
  \nonumber & = \sum_{l=0}^{L-1} C n r_l + \sum_{l=0}^{L} \frac{s_l}{2} B_l^2  + \sum_{l=0}^{L} \frac{s_l}{2} (k_l+1) \log \Bigg(\frac{8n(L+1)}{s_l}\Bigg) \\
  \nonumber & \enskip+ \sum_{l=0}^{L} s_l (k_l+1)(L-l)  \log 2 +  \sum_{l=0}^{L} \frac{s_l}{2} (k_l+1)\log(k_{l}+1) \\
  \nonumber & \enskip + \sum_{l=0}^{L} \frac{s_l}{2} (k_l+1) \log \Big(\log( k_{l+1} 2^{k_{l}+1})\Big) + \sum_{l=0}^{L} s_l (k_l+1) \Bigg( \sum_{m=0,m\neq l}^{L} \log B_m \Bigg) \\
  \nonumber & \leq \sum_{l=0}^{L-1} C n r_l + \sum_{l=0}^{L} \frac{s_l}{2} B_l^2 + \sum_{l=0}^{L} \frac{s_l}{2} (k_l+1) \log \Bigg(\frac{8n(L+1)}{s_l}\Bigg)+ L\sum_{l=0}^{L} s_l (k_l+1)\\
  \nonumber & \enskip +  \sum_{l=0}^{L} \frac{s_l}{2} (k_l+1)(\log(k_{l}+1) + \log ( k_{l+1} + k_l + 1)) + \sum_{l=0}^{L} s_l (k_l+1) \Bigg(\sum_{m=0,m\neq l}^{L} \log B_m \Bigg) \\
  \nonumber & \leq \sum_{l=0}^{L-1} C n r_l + \sum_{l=0}^{L} \frac{s_l}{2} B_l^2 + \sum_{l=0}^{L} \frac{s_l}{2} (k_l+1) \log \Bigg(\frac{8n(L+1)}{s_l}\Bigg) + L\sum_{l=0}^{L} s_l (k_l+1) \\
  \nonumber & \enskip + \sum_{l=0}^{L} s_l (k_l+1) \log ( k_{l+1} + k_l + 1) + \sum_{l=0}^{L} s_l (k_l+1) \Bigg( \sum_{m=0,m\neq l}^{L} \log B_m \Bigg) \\
  \nonumber & \leq \sum_{l=0}^{L-1} C n r_l + \sum_{l=0}^{L} s_l (k_l+1) \Bigg[ \frac{B_l^2}{2(k_l+1)} + \Bigg( \sum_{m=0,m\neq l}^{L} \log B_m \Bigg) + L + \log ( k_{l+1} + k_l + 1)\\
  \nonumber& \enskip + \frac{1}{2}  \log \Bigg(\frac{8n(L+1)}{s_l}\Bigg) \Bigg]\\
\nonumber  &\leq \sum_{l=0}^{L-1} (C+C') n r_l+ C' n r_L \\
  \nonumber & \enskip + \sum_{l=0}^{L} s_l (k_l+1) \Bigg[ \frac{B_l^2}{k_l+1} + \Bigg( \sum_{m=0,m\neq l}^{L} \log B_m \Bigg) + L + \log ( k_{l+1} + k_l + 1) + \log \Bigg(\frac{n}{s_l}\Bigg) \Bigg] \\
  \nonumber & \leq \sum_{l=0}^{L-1} (C+C') n r_l+ C' n r_L  + \sum_{l=0}^L s_l (k_l+1) \vartheta_l \leq C_1 n \sum_{l=0}^L r_l 
\end{align} 
This concludes the proof of \eqref{KL_ineq}. \hfill \qedsymbol

\vspace{1mm}
\noindent{\bf Proof of Corollary \ref{post_contraction_main_theorem}} 

\vspace{1mm}
\noindent The proof is a direct consequence of Theorem \ref{thm:var-post} in the main paper as long as assumptions of  Lemma \ref{lem:prior} and Lemma \ref{lem:kl} parts 1 and 2 hold when $\sigma_0^2=1$, $-\log \lambda_l=\log (k_{l+1})+C_l (k_l+1)\vartheta_l$ and $\epsilon_n=\sqrt{(\sum_{l=0}^L r_l+\xi)\sum_{l=0}^L u_l}$. This what we show next.


\noindent {\it Verifying assumption \eqref{e:ass-lem-prior} under Proof of Lemma \ref{lem:prior}:} Note, $\sum u_l=O(\epsilon_n^2)$, thus
$$ \sum u_l \log L =o(n \epsilon_n^2) \iff  \log L=o(n(\sum r_l+\xi))$$ 
which is indeed true since $\log L =o( L^2)$ and $L^2 \leq  n \sum r_l$. We will show that
$(k_{l+1}\lambda_l)/s_l^\circ \to 0$. With $\lambda_l=(1/k_{l+1})\exp(- C_l(k_l+1)\vartheta_l)$,
\begin{align*}
    \frac{k_{l+1}\lambda_l}{s_l^\circ}&\leq \frac{\sum u_l \exp(-C (k_l+1) \vartheta_l)}{n\epsilon_n^2}=\frac{\exp(-C (k_l+1)\vartheta_l+\log \sum u_l )}{n\epsilon_n^2}\\
    &\leq \frac{\exp(-C  (k_l+1) \vartheta_l +  \vartheta_l)}{n\epsilon_n^2} \to 0
\end{align*}
where the above relation holds since $\log \sum u_l\leq \vartheta_l$, $\vartheta_l \to \infty$, $k_l \to \infty$ and  $n\epsilon_n^2 \to \infty$.

\noindent {\it Verifying assumption \eqref{e:ass-lem-kl} under Proof of Lemma \ref{lem:kl} part 1. and part 2.} Note, $$-\log \lambda_l=\log (k_{l+1})+C_l(k_l+1)\vartheta_l \leq \vartheta_l+C_l(k_l+1)\vartheta_l=O\{(k_l+1)\vartheta_l\} $$
And then, 
$$1-\lambda_l=1-\exp(-C_l \vartheta_l (k_l+1))/k_{l+1}$$
$$-\log(1-\lambda_l)\sim \exp(-C_l \vartheta_l(k_l+1))/k_{l+1}=O\{ (k_{l}+1)s_l\vartheta_l/k_{l+1}\}$$
since $\exp(-C_l \vartheta_l(k_l+1)) \to 0$ and $(k_l+1)s_l \vartheta_l \to \infty$. \hfill \qedsymbol
\newpage

\section{Additional numerical experiments details} 

\label{AppendixB}
\subsection{\SJ{FLOPs Calculation}}
We only count multiply operation for floating point operations (FLOPs) similar to \cite{Zhao2019_VarCNNPrune}. In 2D convolution layer, we assume convolution is implemented as a sliding window and that the nonlinearity function is computed for free. Then, for a 2D convolutional layer (given bias is present) we get FLOPs as:
$$ {\rm FLOPs} = (C_{in,pruned} K_w K_h + 1) O_w O_h C_{out,pruned} $$
where, $C_{in,pruned}, C_{out,pruned}$ are the number of input channels and output channels after pruning. Channels are pruned if all the parameters associated with that channel in convolution mapping are zero. $K_w$ and $K_h$ are the kernel width and height respectively. Finally, $O_w, O_h$ are output width and height where $O_w = (I_w + 2\times P_w - D_w \times (K_w - 1) - 1)/S_w+1$ and $O_h = (I_h + 2\times P_h - D_h \times (K_h - 1) - 1)/S_h+1$. Here, $I_w, I_h$ are input, $P_w,P_h$ are padding, $D_w,D_h$ are dilation, $S_w, S_h$ are stride widths and heights respectively.
\vspace{2mm}

\noindent For fully connected (linear) layers (with bias) we get FLOPs as:
$$ {\rm FLOPs} = (I_{pruned}+1) O_{pruned} $$
where, $I_{pruned}$ is the number of pruned input neurons and $O_{pruned}$ is the number of pruned output neurons. 

\subsection{Variational parameters initialization} 
We initialize the $\gamma_{lj}$'s at a value close to 1 for all of our experiments. This ensures that at epoch 0, we have a fully connected deep neural network. This also warrants that most of the weights do not get pruned off at a very early stage of training which might lead to bad performance. The variational parameters $\mu_{ljj'}$ are initialized using $U(-0.6,0.6)$ for simulation and UCI regression examples whereas for classification Kaiming uniform initialization \citep{Kaiming-He-2015} is used. Moreover, $\sigma_{ljj'}$ are reparameterized using softplus function: $\sigma_{ljj'}=\log(1+\exp(\rho_{ljj'}))$ and $\rho_{ljj'}$ are initialized using a constant value of -6. This keeps initial values of $\sigma_{ljj'}$ close to 0 ensuring that the initial values of network weights stay close to Kaiming uniform initialization.

\subsection{Hyperparameters for training}
We keep MC sample size ($S$) to be 1 during training. We choose learning rate of $3\times 10^{-3}$, batch size of 400, and 10000 epochs in the 20 neurons case of simulation study-I. We use learning rate of $10^{-3}$, batch size of 400, and 20000 epochs in the 100 neurons case of simulation study-I. Next, we use learning rate of $5\times 10^{-3}$, full batch, and 10000 epochs for simulation study-II. In UCI regression datasets, we choose batch size = 128 and run 500 epochs for \textit{Concrete, Wine, Power Plant}, 800 epochs for \textit{Kin8nm}. For \textit{Protein} and \textit{Year} datasets, we choose batch size of 256 and run 100 epochs. For all the UCI regression datasets we keep learning rate of $10^{-3}$. The Adam algorithm is chosen for optimization of model parameters.

\SJ{In image classification datasets, for SS-IG model, we use $10^{-3}$ learning rate and minibatch size of 1024 in all experiments except in Lenet-Caffe on Fashion-MNIST experiment where we use $2\times10^{-3}$ learning rate and 1024 minibatch size. For SV-BNN model, we take $10^{-3}$ learning rate and 1024 minibatch size in all experiments after extensive hyperparameter search. For VBNN model, we take learning rate of $10^{-4}$ and minibatch size of 128 according to \cite{blundell2015weight}. We train each model for 1200 epochs using Adam optimizer in all the image classification experiments provided in main paper.}

\subsection{Fine tuning of constant in prior inclusion probability expression}
Recall the layer-wise prior inclusion probabilities: $\lambda_l=(1/k_{l+1})\exp(-C_l (k_{l}+1)\vartheta_l)$ from the Corollary \ref{post_contraction_main_theorem}. In our numerical experiments, we use this expression to choose an optimal value of $\lambda_l$ in each layer of a given network. The $\lambda_l$ varies as we vary our constant $C_l$ and we next describe how is $C_l$ chosen. The influence of $C_l$  is mainly due to the $k_{l}+1$ term and ${B_l}^2/(k_l+1)$ from $\vartheta_l$ term. We ensure that each incoming weight and bias onto the node from layer $l+1$ is bounded by 1 which leads us to choose $B_l$ to be $k_l+1$. So the leading term from $(k_{l}+1)\vartheta_l$ is $(k_{l}+1)$ and $C_l$ has to be chosen such that we avoid making exponential term from $\lambda_l$ expression close to 0. In our experiments we choose $C_l$ values in the negative order of 10 such that prior inclusion probabilities do not fall below $10^{-50}$. If we instead choose a $\lambda_l$ value very close to 0 then we might prune off all the nodes in each layer or might make the training unstable which is not ideal. Overall the aforementioned strategy of choosing $C_l$ constant values ensure reasonable values for the $\lambda_l$ in each layer.

\subsection{Simulation study I: extra details}
First we provide the network parameters used to generate the data for this simulation experiment. The edge weights in the underlying 2-2-1 network are as follows: $\boldW_0 = \{w_{011}=10, w_{012}=15, w_{021}=-15, w_{022}=10\}; \boldW_1 = \{w_{111}=-3, w_{121}=3\}$ and $\boldv_0 = \{ v_{01}=-5, v_{02}=5\}; \boldv_1 = \{ v_{11}=4\}$.

Below we provide additional results demonstrating the model selection ability of our SS-IG approach in a wider network consisting of 100 nodes in the single hidden layer structure considered in the simulation study-I from main paper. 
\begin{figure}[H]
\centering
\begin{subfigure}{.5\textwidth}
  \centering
  \includegraphics[width=\linewidth]{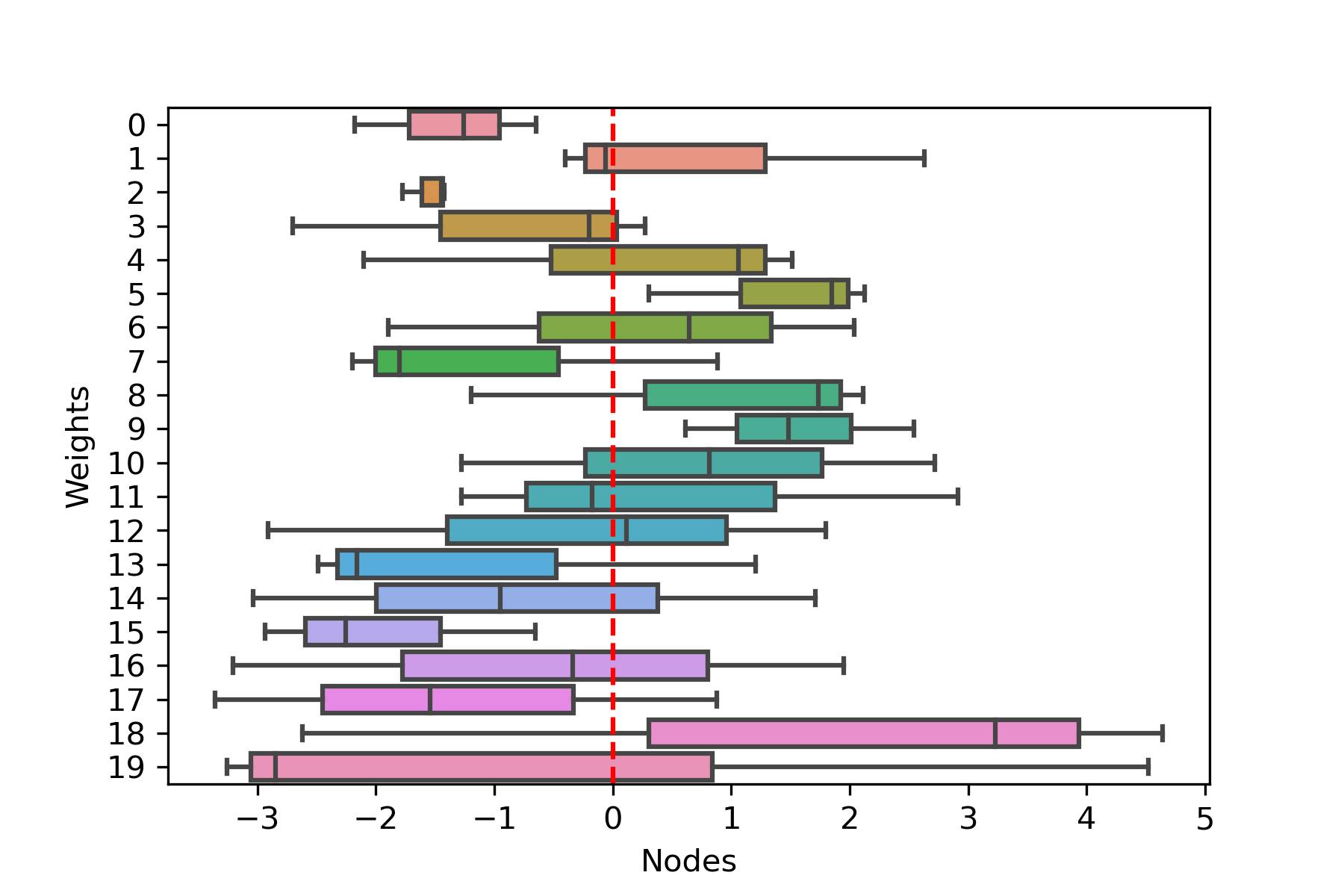}
  \caption{VBNN}
  \label{fig:accuracy2}
\end{subfigure}%
\begin{subfigure}{.5\textwidth}
  \centering
  \includegraphics[width=\linewidth]{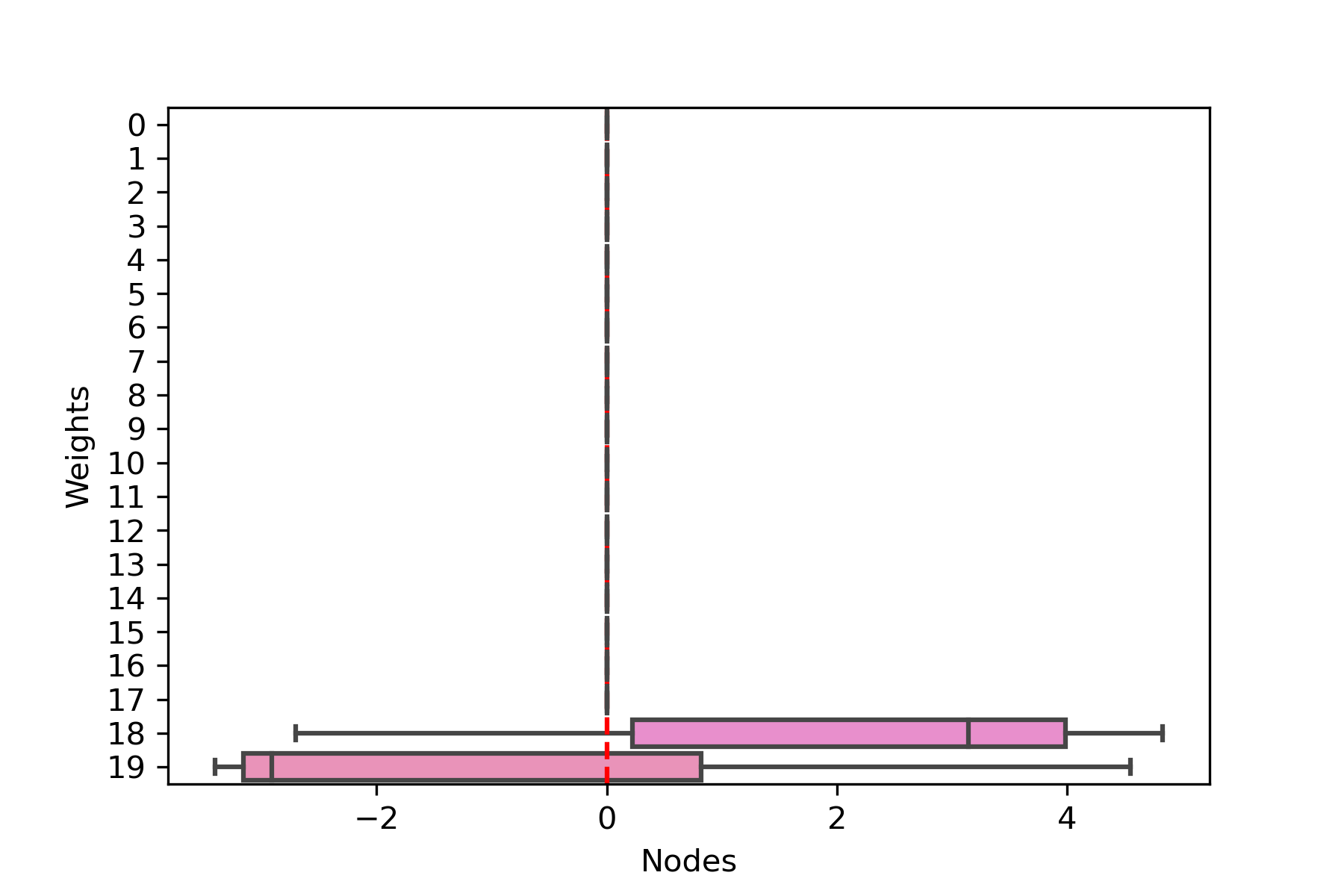}
  \caption{SS-IG}
  \label{fig:sparsity2}
\end{subfigure}
\caption{Node-wise weight magnitudes recovered by VBNN and proposed SS-IG model in the synthetic regression data generated using 2-2-1 network. The boxplots show the distribution of incoming weights into a given hidden layer node. Only the 20 nodes with the largest edge weights are displayed.}
\label{fig:simulation-1_100_nodes}
\end{figure}

\subsection{\SJ{Effect of Hidden Layer Widths}}
Here, we explore 2-hidden layer neural networks with varying widths. For our SS-IG model we use $10^{-3}$ learning rate and minibatch size of 1024 while for VBNN model, we take learning rate of $10^{-4}$ and minibatch size of 128 according to \cite{blundell2015weight}. We train both the models for 400 epochs using Adam optimizer. 

Figure~\ref{fig:mnist-results} summarizes the results. We have provided results for 3 different architectures which have 400, 800, and 1200 nodes each in their 2-hidden layers. In Figure~\ref{fig:MNIST-accuracy-arch}, we find that across the architectures both SS-IG and VBNN models have similar predictive performance. Further, our method is able to prune off more than 88\% of first hidden layer nodes and more than 92\% of second hidden layer nodes (Figure~\ref{fig:MNIST-sparsity-arch}) at the expense of 2\% accuracy loss due to sparsification compared to the densely connected VBNN. We also observe that as model capacity increases the sparsity percentage per layer decreases. This suggests that, each architecture is trying to reach a sparse network of comparable size. 

\begin{figure}[H]
\centering
\begin{subfigure}[b]{0.495\textwidth}
    \centering
    \includegraphics[width=0.9\textwidth]{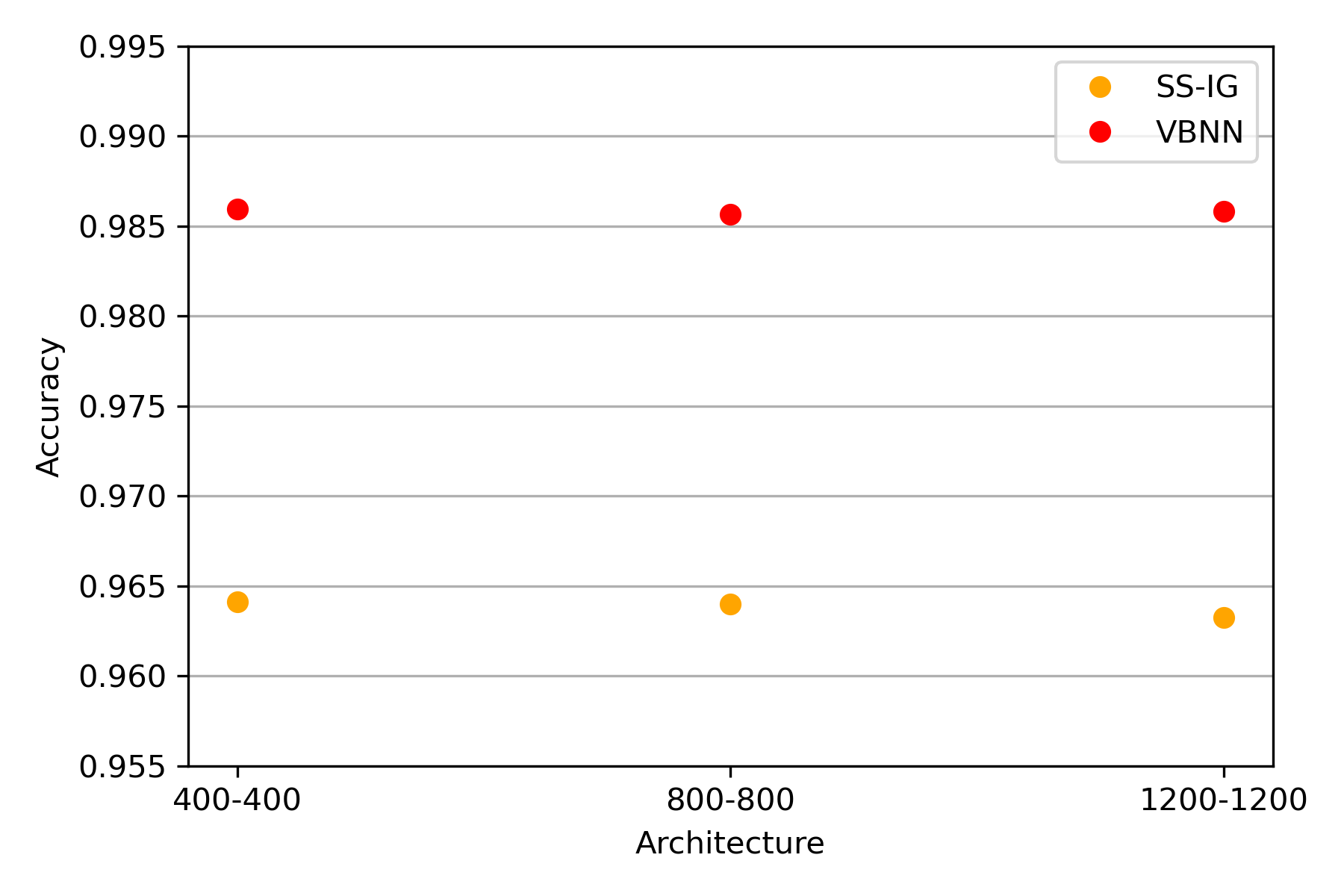} 
    \caption[]%
    {{\small Prediction accuracy per architecture}}    
    \label{fig:MNIST-accuracy-arch}
\end{subfigure}
\hfill
\begin{subfigure}[b]{0.495\textwidth}  
    \centering 
    \includegraphics[width=0.9\textwidth]{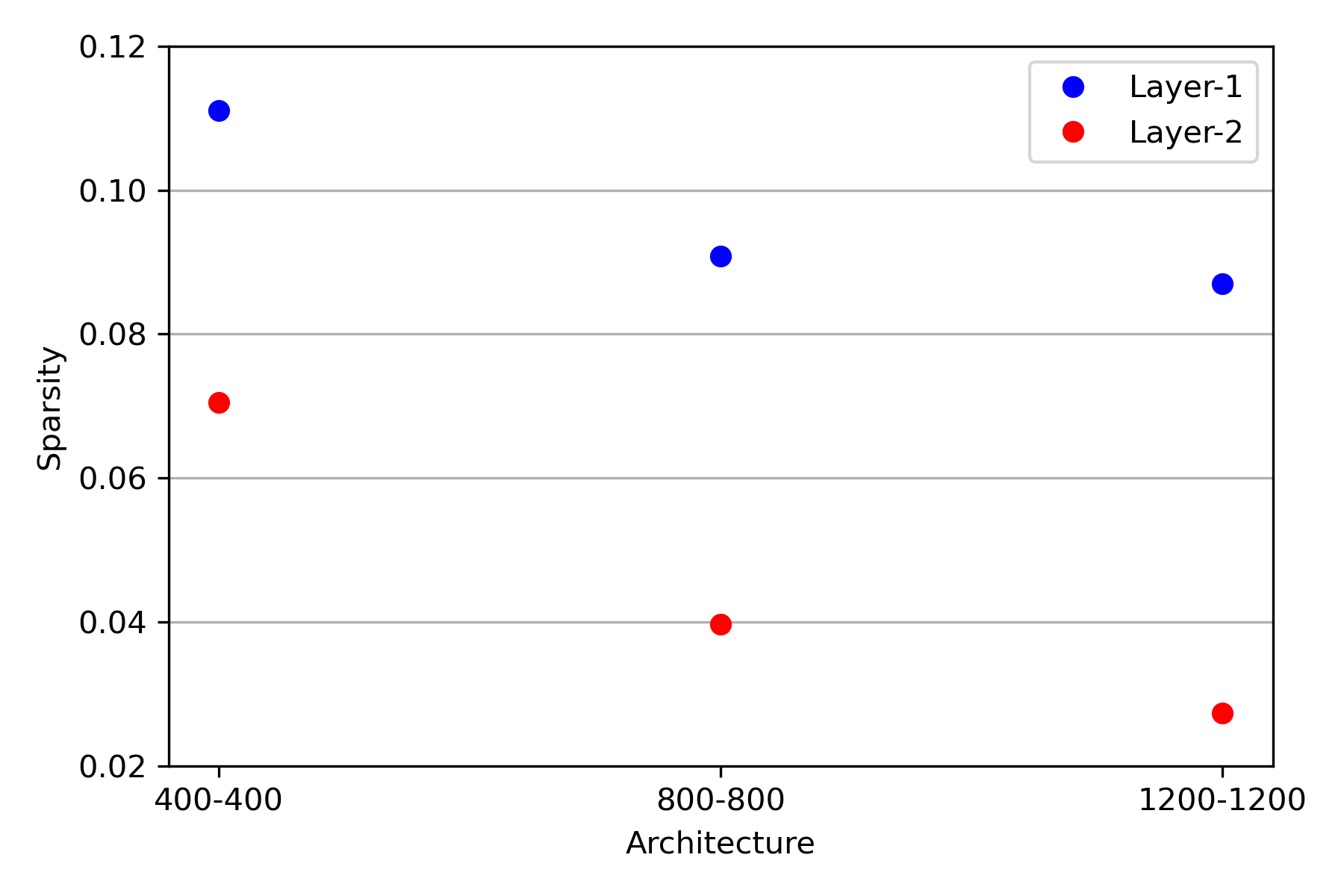} 
    \caption[]%
    {{\small Layer-wise sparsity per architecture}}    
    \label{fig:MNIST-sparsity-arch}
\end{subfigure}
\caption{MNIST experiment results for varying hidden layer widths}
\label{fig:mnist-results}
\end{figure}


\end{document}